\documentclass[twoside]{article}

\usepackage[accepted]{aistats2025}
\usepackage[round]{natbib}
\renewcommand{\bibname}{References}

\usepackage[dvipsnames]{xcolor} 
\definecolor{myblue}{RGB}{0,180,216} 
\definecolor{mypink}{RGB}{239,39,166} 
\definecolor{myorange}{RGB}{255,111,0} 
\definecolor{myyellow}{RGB}{255,190,11} 

\usepackage{multirow}
\usepackage[utf8]{inputenc}
\usepackage{amsmath,amssymb,array}
\usepackage{bbm}
\usepackage{graphicx}
\usepackage{xargs}
\usepackage{algorithmic}
\usepackage[font=small,skip=0pt]{caption}
\usepackage{wrapfig}
\usepackage{caption}

\usepackage{multicol}
\usepackage{booktabs}
\usepackage{todonotes}
\usepackage{amsthm}
\usepackage{titletoc} 

\renewcommand\citet{\cite}
\usepackage[hyphens]{url}
\usepackage[hidelinks,colorlinks=true,citecolor=blue!50!black,linkcolor=black,urlcolor=green!50!black]{hyperref}
\usepackage{colortbl}
\usepackage{csquotes} 
\usepackage{mathtools}
\usepackage{float} 
\usepackage{adjustbox}
\usepackage{dsfont}

\newtheorem{theorem}{Theorem}
\newtheorem{lemma}{Lemma}
\newtheorem{definition}{Definition}
\newtheorem{corollary}{Corollary}
\newtheorem{remark}{Remark}

\usepackage[capitalize,noabbrev]{cleveref}

\usepackage[acronym,nonumberlist,toc,nogroupskip]{glossaries}
\newglossaryentry{formula}{name=formula,
                           description={A mathematical expression}}

\newacronym{SI}{SI}{Shapley interaction}
\newacronym{SV}{SV}{Shapley value}
\newacronym{XAI}{XAI}{Explainable AI}
\newacronym{ALE}{ALE}{accumulated local effects}
\newacronym{PDP}{PDP}{partial dependence plot}
\newacronym{GV}{GV}{generalized value}
\newacronym{MT}{MT}{M\"obius transform}

\begin{document}

%

%
\runningauthor{Fabian Fumagalli, Maximilian Muschalik, Eyke H{\"u}llermeier, Barbara Hammer, Julia Herbinger}

\twocolumn[

\aistatstitle{Unifying Feature-Based Explanations with Functional ANOVA and Cooperative Game Theory}

\aistatsauthor{Fabian Fumagalli\\ Bielefeld University, CITEC\And Maximilian Muschalik \\ LMU Munich, MCML \And Eyke H{\"u}llermeier \\ LMU Munich, MCML \AND Barbara Hammer \\ Bielefeld University, CITEC \And Julia Herbinger \\ Leibniz Institute for Agricultural\\ Engineering and Bioeconomy}

\aistatsaddress{} 

]

\begin{abstract}
Feature-based explanations, using perturbations or gradients, are a prevalent tool to understand decisions of black box machine learning models.
Yet, differences between these methods still remain mostly unknown, which limits their applicability for practitioners.
In this work, we introduce a unified framework for local and global feature-based explanations using two well-established concepts: functional ANOVA (fANOVA) from statistics, and the notion of value and interaction from cooperative game theory.
We introduce three fANOVA decompositions that determine the influence of feature distributions, and use game-theoretic measures, such as the Shapley value and interactions, to specify the influence of higher-order interactions.
Our framework combines these two dimensions to uncover similarities and differences between a wide range of explanation techniques for features and groups of features.
We then empirically showcase the usefulness of our framework on synthetic and real-world datasets.
\end{abstract}

\section{INTRODUCTION}

\begin{figure*}[t]
    \centering
    \includegraphics[width=\textwidth]{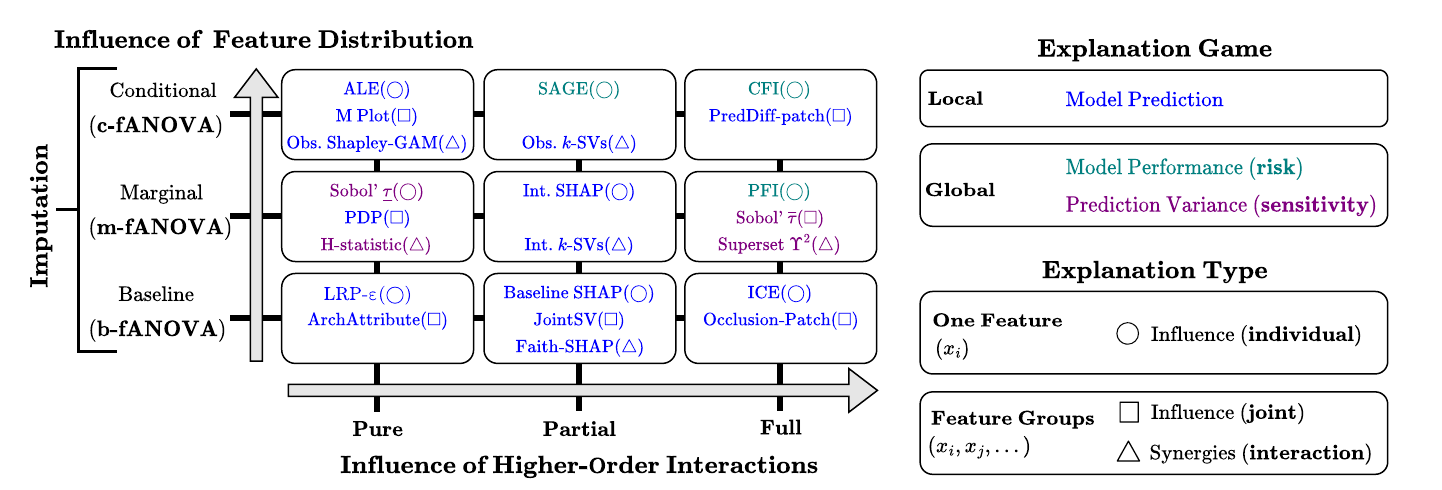}
    \caption{Categorization of selected feature-based explanations with our framework: {\color{blue}Local}, and global methods ({\color{teal}risk}, {\color{violet}sensitivity}) by color, and explanation types individual ($\bigcirc$) and joint ($\square$) influence, as well as interactions ($\triangle$) by symbol. Each imputation method (b/m/c) corresponds to a fANOVA decomposition with increasing influence of feature distributions, whereas pure, partial and full effects are increasingly influenced by higher-order interactions.}
    \label{fig_intro_illustration}
\end{figure*}

The growing application of black box machine learning models in high-stake decision domains has led to increased research in \gls*{XAI} to uncover aspects of their reasoning.
In this context, feature-based explanations are a prominent tool to address research questions, such as ``How does this individual feature influence the model's prediction or performance?''.
The trend of developing new feature-based explanations to address the limitations of existing methods has resulted in a large toolbox, making it unclear how they are related or which method users should choose.
Categorizations, e.g., into local and global, and guidance on pitfalls \citep{molnar2022pitfalls} or within sub-categories of explanation methods \citep{DBLP:conf/icml/SundararajanN20, Ewald2024featimp}, help to understand differences but do not provide a general understanding of relations between methods, nor answer how to choose an appropriate one.
Existing frameworks unify methods for individual features \citep{covert2021explaining,lundstrom2023unifying,Deng2024Unify}, or SHAP \citep{hiabu2022unifying,bordt2022shapley}, and serve as a building block here.
\\
In this paper, we present a \emph{unified framework} for feature-based explanations, which is driven from two different perspectives: \emph{functional analysis of variance (fANOVA)} \citep{stone1994fanova, hooker_discovering_2004} from statistics, as well as the \emph{\gls*{SV}} \citep{shapley_1951} and \textit{\gls*{MT}} \citep{grabisch2000equivalent} from cooperative game theory.
Our framework covers \emph{local explanations} for a specific data point and two global explanations in terms of \emph{sensitivity} and \emph{risk}.
Practitioners determine the \emph{influence of feature distribution} by a specific choice of fANOVA.
This choice corresponds to the \emph{marginalization} of features using either a \emph{baseline (b)}, the \emph{marginal (m)}, or the \emph{conditional (c)} distribution and yields an \emph{increasing influence of feature distributions}.
Based on the fANOVA decomposition, we introduce three types of feature influence measures: \emph{Individual} and \emph{joint} effects capture the contributions of single and groups of features, whereas \emph{interaction} effects quantify synergies and redundancies of groups of features.
Finally, using fundamental concepts from cooperative game theory, we quantify \emph{pure}, \emph{partial}, and \emph{full} effects that yield an increasing \emph{influence of higher-order interactions}.
\cref{fig_intro_illustration} illustrates our framework and categorizes existing methods.

\paragraph{Contributions.} Our main contributions include:

\textbf{(1)} A framework for feature-based explanations using fANOVA and cooperative game theory, including a local and two global explanation games.

\textbf{(2)} Interpretations for three types of explanations according to influence of a) feature distributions, and b) higher-order interactions, undermined by our theoretical results and examples.

\textbf{(3)} A summary of existing methods within our framework, including a guide for practitioners.

\section{BACKGROUND}
\label{sec:background}

Here, we provide the relevant background on fANOVA (\cref{sec_fanova}), gradient-based methods (\cref{sec_gradients}), and game theory (\cref{sec_game_theory}).
Each field quantifies feature-based explanations from a different viewpoint, which our unified framework (\cref{sec_unified_framework}) summarizes.

\paragraph{Notations.} Throughout the paper, we use the feature set $D := \{1,\dots,d\}$, lowercase letters for cardinalities ($s := \vert S \vert$), abbreviated sets ($i:=\{i\}$ and $ij := \{i,j\}$), and denote the complement as $-S := D \setminus S$.
\subsection{Functional ANOVA}\label{sec_fanova}
Functional analysis of variance (fANOVA) \citep{Hoeffding_1948,stone1994fanova,hooker_discovering_2004,Owen_2013} is a popular tool in statistics that \emph{globally} decomposes a black box model $F: \mathbb{R}^d \to \mathcal Y$ on a $d$-dimensional feature space into \emph{pure effects} 
\begin{equation*}
    F(x) = f_\emptyset + \sum_{i=1}^d f_i(x_i) + \sum_{i\neq j} f_{ij}(x_i,x_j) \dots = \sum_{S \subseteq D} f_S(x),
\end{equation*}
where $f_S$ is restricted to features in $S$.
These effects $f_S$ are referred to as \emph{baseline} ($S=\emptyset)$, \emph{main} ($s = 1$) and \emph{interaction} effects ($s > 1$).
A fANOVA always exists, but is not generally unique.
Given the feature distribution $P$, the components are defined as
\begin{align}\label{eq_fanova_effect}
    f_S(x) =\int F(x) \, dP(x_{-S}) - \sum_{T \subset S} f_T(x) ,
\end{align}

i.e. the effect $f_S$ is the expectation over features in $-S$ minus the sum of lower-order effects $f_T$, where $T$ is contained in $S$.
For independent feature distributions, all effects $f_S$ are zero-centered and orthogonal, which yields the \emph{unique variance decomposition} $\mathbb{V}[F(X)] = \sum_{S \subseteq D} \sigma_S^2$ with $\sigma_S^2 := \mathbb{V}[f_S(X)]$.
Based on this decomposition, \emph{sensitivity analysis} \citep{Owen_2013} then quantifies the global importance of a set of features by summarizing the variances of the effects, $\sigma_S^2$, via the Sobol' indices $\underline{\tau},\bar{\tau}$ \citep{Sobol_2001} or the superset measure  $\Upsilon^2$ \citep{hooker_discovering_2004} as
\begin{align*}
&{\underline{\tau}}_S := \sum_{T \subseteq S}\sigma^2_T, 
&&{\overline{\tau}}_S := \sum_{T \cap S \neq \emptyset} \sigma^2_T,
&& \Upsilon_S^2 := \sum_{T \supseteq S} \sigma^2_T.
\end{align*}

Additionally, \cite{Owen_2014} utilizes the \gls*{SV} \citep{shapley_1951} (introduced in \cref{sec_game_theory}) to summarize the variances.
For dependent random variables, the effects $f_T$ are not orthogonal, and thus the variance decomposition does not hold. 
With alternative definitions of $f_T$ \citep{hooker_generalized_2007} or assumptions on $P$ \citep{Chastaing_Gamboa_Prieur_2012} unique and quasi-orthogonal decompositions can be found.
Moreover, it was shown that the definition via Eq.~(\ref{eq_fanova_effect}) yields sensible results when combined with the \gls*{SV} \citep{Owen.2017}.

\subsection{Gradient-Based Attribution Methods}\label{sec_gradients}

Gradient-based attributions are a prominent model-specific approach to explain predictions of black box models, where gradient information is available.
They are widely applied in deep learning, and a variety of methods have been proposed.
They can be broadly categorized into methods that use the gradients of each layer and backpropagate attributions throughout the network \citep{bach2015pixel,Shrikumar_Greenside_Kundaje_2017,lundberg_2017_unified}, and methods that rely on the gradients computed with respect to the features directly \citep{DBLP:conf/icml/SundararajanTY17,DBLP:journals/corr/SmilkovTKVW17,DBLP:journals/natmi/ErionJSLL21}.
Most popular methods can be unified with Taylor interactions \citep{Deng2024Unify}.
Taylor interactions are obtained by a multivariate infinite Taylor expansion of $F$ around a baseline $b \in \mathbb{R}^d$ as

\begin{align*}
    F(x) &= F(b) + \sum_{i \in D} \underbrace{\sum_{\kappa \in \Omega_i}\phi(\kappa)}_{=: \psi(i)}+ \sum_{S \subseteq D, s > 1} \underbrace{\sum_{\kappa \in \Omega_S}I(\kappa)}_{=: J(S)},
\end{align*}
where $\psi,J$ are the generic independent, and interaction effect, respectively.
Here, $\kappa=[\kappa_1,\dots,\kappa_d] \in \mathbb{N}_0^d$ capture the order of partial derivatives for each feature in the Taylor expansion, and $\Omega_S := \{\kappa \mid  \forall i \in S: \kappa_i>0, \forall i \notin S: \kappa_i=0\}$ is the set of degree vectors, where partial derivatives are computed exclusively for features in $S$.
For more details regarding the definition of Taylor interactions, we refer to Appendix~C.

\subsection{Cooperative Game Theory}\label{sec_game_theory}

Feature-based explanations attribute \emph{value} to individual features or groups of features.
This attribution can be viewed from a game theoretical perspective by studying a cooperative \emph{game} $\nu: 2^D \to \mathbb{R}$, which describes a \emph{payoff} for any possible group of features.
This game is typically the prediction (local) or the performance (global) of the model restricted to the particular subset of features \citep{covert2021explaining}.
Given a set of features $T \subseteq D$ and a feature $i \in D$, we measure the \emph{marginal contribution} of this feature as $\Delta_i(T) := \nu(T \cup i) - \nu(T \setminus i)$, i.e., the impact of adding this feature to a group of features.
The \emph{joint marginal contribution} $\Delta_{[ij]}(T) := \nu(T \cup ij) -\nu(T \setminus ij)$ is an extension for the contribution of two features $i,j \in D$. 
The \emph{interaction} $\Delta_{ij}(T) := \Delta_{[ij]}(T) - \Delta_i(T) - \Delta_j(T)$ of $i,j$ in the presence of $T$ is the effect of adding both features jointly minus their individual contributions.
Subtracting the marginal contributions yields the \emph{synergy} or \emph{redundancy} of both features by accounting for their individual effects.
The generalization of this recursion is the \emph{discrete derivative} of a set $S \subseteq D$ in the presence of a set $T \subseteq D \setminus S$ defined as
\begin{align*}
    \Delta_S(T) := \sum_{L \subseteq S} (-1)^{s-\ell}\nu(T \cup L) \text{ for } T \subseteq -S.
\end{align*}
The discrete derivative is a natural measure of interaction for $S\subseteq D$ in the presence of $T$.
The \emph{absence} ($T=\emptyset$) and \emph{presence} ($T=-S$) of all remaining features yields the \gls*{MT} and Co-\gls*{MT} \citep{grabisch2000equivalent}, respectively, as
\begin{align*}
    m(S) := \Delta_S(\emptyset) \text{ and } \tilde m(S) := \Delta_S(-S).
\end{align*}

The \gls*{MT}, or Harsanyi dividend \citep{harsanyi1963simplified}, is the unique measure \citep{rota1964foundations} that satisfies $\nu(T) = \sum_{S \subseteq T} m(S) \text{ for } T \subseteq D$, i.e., $m(S)$ is the \emph{pure} additive contribution obtained by the interaction of $S$, which cannot be attributed to any subgroup of $S$.
Clearly, such contributions are difficult to interpret, since they involve $2^d$ distinct coefficients.
Instead, the \gls*{SV} \citep{shapley_1951} \emph{aggregates} marginal contributions for an individual feature $i \in D$ as

\begin{align*}
    \phi^{\text{SV}}(i) := \sum_{T \subseteq D \setminus i}  \frac{\Delta_i(T)}{d \cdot \binom{d-1}{t}} = \sum_{S \subseteq D: i \in S} \frac{m(S)}{s},
\end{align*}

where the second representation is a well-known result from cooperative game theory \citep{DBLP:journals/fss/Grabisch97}.
In this context, the \gls*{SV} summarizes the main effect and all pure interactions equally distributed among the involved features.
Moreover, the \gls*{SV} is the unique measure that yields a \emph{fair} attribution with four intuitive axioms: linearity, symmetry, dummy and efficiency \citep{shapley_1951}.
Beyond individual contributions, \glspl*{GV} \citep{DBLP:journals/dam/MarichalKF07} and JointSVs \citep{Harris.2022} summarize joint marginal contributions, whereas the Shapley interaction index \citep{Grabisch.1999} summarizes discrete derivatives for all groups of features.
Lastly, \glspl*{SI} \citep{lundberg_consistent_2019,sundararajan2020shapley,bordt2022shapley,tsai_faith-shap_2022} provide interactions for subsets up to size $k=1,\dots,d$, where $k=1$ is the \gls*{SV}, and $k=d$ the \gls*{MT} \citep{bordt2022shapley}.

\section{A UNIFIED FRAMEWORK}\label{sec_unified_framework}

We now present our unified framework for feature-based explanations.
We propose three fANOVA decompositions (\cref{sec_framework_fanova}), and define a \emph{local} explanation game based on the prediction, and two \emph{global} explanation games based on sensitivity and risk (\cref{sec_explanation_games}).
We then introduce three types of explanations: individual, joint and interaction effects (\cref{sec_framework_gametheory_types}), and show that feature-based explanations are essentially determined by two components: First, the choice of distribution $P$ in the fANOVA specifies the \emph{influence of feature distribution} in the explanation.
Second, \emph{pure, partial} or \emph{full} effects (\cref{sec_framework_gametheory_measure}) are increasingly \emph{influenced by higher-order interactions}.
All proofs are deferred to \cref{appx_sec_proofs}.

\subsection{Influence of Feature Distributions}\label{sec_framework_fanova}
The fANOVA is determined by $P$ in the
\begin{equation}\label{eq_F_S}
    \textbf{value function } F_S(x) := \int F(x) \, dP(x_{-S}),
\end{equation}
and iteratively constructed by Eq.~(\ref{eq_fanova_effect}).
We now present three decompositions (induced by choices of $P$) that are increasingly influenced by the feature distribution.

\begin{definition}
\emph{Baseline fANOVA} (\textbf{b-fANOVA}) with baseline $b \in \mathbb{R}^d$ is given by $P^{\text{(b)}}(x_{-S}) := \bigotimes_{i \in -S} \delta_{b_i}$ with Dirac functions $\delta_{b_i}$, and effects $f^{(b)}_S$ derived from
\begin{align*}
    F^{(b)}_S(x) := \int F(x) \, dP^{(b)}(x_{-S}) = F(x_S,b_{-S}).
\end{align*}
\end{definition}

The feature distribution influences b-fANOVA exclusively through the choice of the baseline $b$.
Moreover, b-fANOVA effects correspond to the Taylor interactions \citep{Deng2024Unify}, introduced in \cref{sec_game_theory}.

\begin{theorem}\label{thm_bfanova_taylor}
If $F$ is represented by its Taylor series expanded around $b$ for an instance $x_0$, then the effect $f_S^{(b)}$ is given by the generic effects \citep{Deng2024Unify}
\begin{align*}
    f_S^{(b)}(x_0) = \begin{cases}
        F(b), \text{ if } S = \emptyset &(\text{baseline})
        \\
        \psi(i), \text{ if } S = i &(\text{main effect})
        \\
        J(S),  \text{ else.} &(\text{interaction effect})
    \end{cases}
\end{align*}
\end{theorem}

In b-fANOVA, the feature distribution is highly compressed in $b$. 
Instead, we define a fANOVA that partially takes into account the feature distribution.

\begin{definition}
\emph{Marginal fANOVA} (\textbf{m-fANOVA}) is given by $P^{\text{(m)}}(x_{-S}) := p(x_{-S})$ with joint marginal distribution $p$, and effects $f^{(m)}_S$ derived from
\begin{align*}
    F^{\text{(m)}}_S(x) := \int F(x) \, d P^{(m)}(x_{-S}) = \mathbb{E}[F(x_S,X_{-S})].
\end{align*}
\end{definition}

Using the joint marginal distribution of the remaining features breaks the feature dependencies between features in $S$ and $-S$, which yields evaluations of $F$ in sparsely sampled or unseen regions of the input space, known as extrapolation \citep{hooker_generalized_2007}. 
Therefore, we propose a third fANOVA based on conditional distributions that fully accounts for the feature distribution.

\begin{definition}
\emph{Conditional fANOVA} (\textbf{c-fANOVA}) is given by $P^{\text{(c)}}(x_{-S}) := p(x_{-S} \mid x_S)$ with conditional distribution $p$, and effects $f^{(c)}_S$ derived from
\begin{align*}
    F^{\text{(c)}}_S(x) := \int F(x) \, d P^{(c)}(x_{-S}) = \mathbb{E}[F(X) \mid X_S = x_S].
\end{align*}
\end{definition}

The c-fANOVA, m-fANOVA and b-fANOVA are related through the following theorem.

\begin{theorem}\label{theorem_fanova_equivalence}
    If $F$ is linear, then m-fANOVA equals b-fANOVA with $b := \mathbb{E}[X]$.
    For independent features, c-fANOVA equals m-fANOVA, and if $F$ is additionally multilinear, then it equals b-fANOVA with $b := \mathbb{E}[X]$.
\end{theorem}

\begin{table*}[t]
\caption{Main and interaction fANOVA components of $F_{\text{lin}}$ and $F_{2\text{int}}$, where $\bar x_i := x_i - \mu_i$}\label{tab:fanova_example}
\begin{tabular}{@{}l|l|l|l|l@{}}
\toprule
& & \textbf{b-fANOVA} $f^{(b)}$ & \textbf{m-fANOVA} $f^{(m)}$ & \textbf{c-fANOVA} $f^{(c)}$
\\ \midrule
\multirow{2}{*}{\textbf{Main Effect} $f_i$} & $F_{\text{lin}}$ 
&$\beta_i \cdot (x_i - b_i)$ 
& $\beta_i \cdot \bar x_i$ 
& $f_{\text{lin}}^{(m)}+  \bar x_i \frac{\beta_{j}\sigma_{ij}}{\mathbb{V}[X_i]}$
\\
& $F_{2\text{int}}$ 
& $f^{(b)}_{\text{lin}} + \beta_{ij} b_j (x_i - b_i)$ 
& $f^{(m)}_{\text{lin}} + \beta_{ij} \mu_j \bar x_i - \beta_{ij}\sigma_{ij}$
&  $f_{\text{lin}}^{(c)}+ \beta_{ij}\bar x_i (\mu_j + \frac{\sigma_{ij} x_i}{\mathbb{V}[X_i]}) - \beta_{ij}\sigma_{ij}$
\\ \midrule
\multirow{2}{*}{\textbf{Interaction} $f_{ij}$} & $F_{\text{lin}}$ & 0 & 0 & $-\sum_{\ell \in {ij}} \bar x_\ell \frac{\beta_{-\ell}\sigma_{-\ell,\ell}}{\mathbb{V}[X_\ell]}$
\\
& $F_{\text{2int}}$ 
& $\beta_{ij} (x_i - b_i) (x_j - b_j) $ 
& $\beta_{ij} \bar x_i \bar x_j  + \beta_{ij}\sigma_{ij}$
& $f^{\text{(m)}}_{\text{2int}} + f_{\text{lin}}^{(c)} - \beta_{ij}\sigma_{ij}(\frac{\bar x_i x_i}{\mathbb{V}[X_i]} +  \frac{\bar x_j x_j}{\mathbb{V}[X_j]})$
\\ \bottomrule
\end{tabular}
\end{table*}

\paragraph{Illustration of FANOVA Decompositions.}
We now illustrate differences between the fANOVA decompositions based on features from a two-dimensional multivariate normal distribution $X \sim \mathcal{N}(\mu, \Sigma)$ with mean $\mu \in \mathbb{R}^2$ and covariance $\Sigma \in \mathbb{R}^{2 \times 2}$. 
We consider a linear function $F_{\text{lin}}(x) = \beta_1 x_1 +  \beta_2 x_2$, and one including a two-way interaction $F_{\text{2int}}(x) = F_{\text{lin}}(x) + \beta_{12} x_1 x_2$. 
The main and interaction effects of the fANOVA decompositions are summarized in \cref{tab:fanova_example}.
For b- and m-fANOVA of $F_{\text{lin}}$ the main effect is intuitively the deviation of $x_i$ from $b_i$ and the expectation of $\mathbb{E}[X_i] = \mu_i$, respectively, which is scaled by the linear coefficient $\beta_i$ (first row).
For c-fANOVA, this effect is distorted by the cross-correlation with the second feature. 
As stated in \cref{theorem_fanova_equivalence}, it is clear that c-fANOVA reduces to m-fANOVA for independent features and that m-fANOVA equals b-fANOVA for $b_i:=\mathbb{E}[X_i]$.
For $F_{2\text{int}}$, the first order effects of b- and m-fANOVA again only depend on feature $i$, but also include the interaction coefficient $\beta_{ij}$ (second row). 
In contrast, c-fANOVA additionally includes cross-correlation effects based on effects of features correlated with $x_i$. 
The two-way interaction of b- and m-fANOVA is solely defined by the centered interaction effect between the two features and thus results in zero values for $F_{lin}$. 
Yet, c-fANOVA accounts again for cross-correlation effects. 
For a more detailed discussion, we refer to \cref{appx_sec_fanova_examples}.

\subsection{Explanation Games via fANOVA}\label{sec_explanation_games}
In the previous section, we have shown that the choice of $P$ in $F_S$ from Eq.~(\ref{eq_F_S}) determines the influence of the feature distribution in the fANOVA decomposition.
In this section, we define three instances of an 
\begin{align*}
    \textbf{explanation game } \nu: 2^D \to \mathbb{R},
\end{align*}
which consider different properties of the value function $F_S$ given a feature set $S \subseteq D$.
We then show that the \gls*{MT} of $\nu$ directly relates to properties of the fANOVA effects $f_S$.
For local explanations, we naturally analyze $F_S$ evaluated at $x_0$.

\begin{definition}\label{def_local_explanation_game}
Given $F_S$ and an instance $x_0 \in \mathbb{R}^d$, we define the \textbf{local explanation game} as
\begin{align*}
     \nu_{x_0}^{(\text{loc})}(S) := F_S(x_0).
\end{align*}
\end{definition}

The local explanation game captures the predictions $F_S$ at $x_0$, i.e. the prediction $F$ at $x_0$ restricted to any subset of features $S$.
In this case, the \gls*{MT} directly corresponds to the fANOVA effect evaluated at $x_0$.

\begin{corollary}\label{cor_fanova_moebius}
The \gls*{MT} $m_{x_0}^{(\text{loc})}$ of the local explanation game $\nu^{(\text{loc})}_{x_0}$ is the fANOVA effect $f_S$ evaluated at $x_0$, 
\begin{align*}
     m_{x_0}^{(\text{loc})}(S) = \sum_{T \subseteq S} (-1)^{s-t}F_T(x_0) = f_S(x_0),
\end{align*}
i.e., the pure additive contribution of the features in $S$ in the fANOVA decomposition.
\end{corollary}

\begin{remark}
    Corollary~\ref{cor_fanova_moebius} follows directly from the definitions of the \gls*{MT} and the fANOVA components.
    In fact, both are special cases of the Möbius inversion theorem \citep[Proposition 2]{rota1964foundations} with the inclusion ordering.
    While the \gls*{MT} $m^{\text{(loc)}}_{x_0}$ is based on $F_S(x_0)$, the fANOVA effect is based on the functions $F_S$.
\end{remark}

The local explanation game depends on $x_0$, whereas \emph{global} explanation games capture $F_S$ across the random features $X$ or the data-generating process $(X,Y)$.

\begin{definition}\label{def_sensitivity_game}
Given $F_S$, the \textbf{global sensitivity game} is defined as the variance across the features
\begin{align*}
     \nu^{(\text{sens})}(S) := \mathbb{V}[F_S(X)].
\end{align*}
\end{definition}

\begin{definition}\label{def_global_risk_game}
Given $F_S$ and a loss $\ell: \mathcal Y \times \mathcal Y \to \mathbb{R}$, the \textbf{global risk game} is defined as the negative risk
\begin{align*}
    \nu^{(\text{risk})}(S) := -\mathbb{E}_{(x,y)\sim(X,Y)}[\ell(F_S(x),y)].
\end{align*}
\end{definition}

The sensitivity game quantifies the importance of features to the model's prediction whereas the risk game considers its performance.
In general, a large value of $\nu$ indicates a higher importance of $S$.
The following theorem describes the \gls*{MT} of the sensitivity game.

\begin{theorem}\label{theorem_sensitivity}
    The \gls*{MT} of the sensitivity game is 
        \begin{align*}
         m^{(\text{sens})}(S) = \mathbb{V}[f^{(c)}_S(X)]
    \end{align*}
    for independent features. For dependent features, it is distorted by $\sum_{L \cup L' =S,L \neq L'}\text{cov}(f^{(c)}_L(X),f^{(c)}_{L'}(X))$.
\end{theorem}

For independent features, the \gls*{MT} of the sensitivity game is the variance of the fANOVA effect $f_S$.
For dependent features, the covariances are exclusively distributed among the \gls*{MT}.
Lastly, for the risk game, the \gls*{MT} $m^{(\text{risk})}(S)$ is the additive \emph{improvement in risk} obtained by $S$.
Given an explanation game, we now specify different types of feature influence.

\subsection{Types of Feature Influence Measure}\label{sec_framework_gametheory_types}

Here we present three types of 
\begin{align*}
    \textbf{feature influences } \phi: 2^D \to \mathbb{R};
\end{align*}
Given an explanation game $\nu$, \textbf{individual feature influence} $\phi(i)$ quantifies the contribution of a single feature $i \in D$, which is naturally based on the marginal contribution $\Delta_i(T) := \nu(T \cup i) - \nu(T\setminus i)$ in the presence of features in $T$.
If individual features are less meaningful, e.g., in high-dimensional image data \citep{DBLP:conf/iclr/ZintgrafCAW17}, then practitioners are frequently interested in the \textbf{joint feature influence} $\phi(S)$ of a group of features $S$, which we base on the joint marginal contribution $\Delta_{[S]}(T) = \nu(T \cup S) - \nu(T \setminus S)$ in the presence of features in $T$.
Beyond joint contributions, to investigate \emph{synergies and redundancies} of groups of features, referred to as \emph{interactions}, we introduce \textbf{feature interactions} $\phi^{I}(S)$ based on discrete derivatives $\Delta_S(T)=\Delta_{[S]}(T)-\sum_{L \subset S} \Delta_L(T)$ in the presence of features in $T$, which account for lower-order effects.
Having established three types of feature influences, we now introduce measures to quantify these.

\subsection{Quantifying Feature Influence}\label{sec_framework_gametheory_measure}

Given an explanation game $\nu$ and explanation type, we propose feature influence in absence of all remaining features as the \textbf{pure effect}, i.e., $\phi^\emptyset(i) := \Delta_i(\emptyset) =\nu(i) - \nu(\emptyset)$ for individuals, $\phi^\emptyset(S) := \Delta_{[S]}(\emptyset) = \nu(S) - \nu(\emptyset)$ for joint influences, and $\phi^{I\emptyset}(S) := \Delta_S(\emptyset) = m(S)$ for interactions.
Second, we introduce feature influence in presence of all remaining features as the \textbf{full effect}, i.e. $\phi^+(i) := \Delta_i(-i)= \nu(D) - \nu(-i)$ for individuals, $\phi^+(S) := \Delta_{[S]}(-S)=\nu(D)-\nu(-S)$ for joint influences, and $\phi^{I+}(S) := \Delta_S(-S) = \tilde m(S)$ for interactions.
Notably, pure interactions $\phi^{I\emptyset}$ and full interactions $\phi^{I+}$ correspond to the \gls*{MT} and Co-\gls*{MT}, respectively.
Lastly, we introduce the \textbf{partial effect} as the \gls*{SV} $\phi^{\text{SV}}$ for individuals, GV $\phi^{\text{GV}}$ for joint influences, and \glspl*{SI} $\phi^{\text{SI}}$ for interactions.
With these definitions, we now analyze how feature influence measures capture the \gls*{MT}, and thus the corresponding properties of fANOVA effects.

\begin{theorem}\label{theorem_influence_summary}
    Feature influence measures summarize the \gls*{MT}, according to \cref{tab_influence_summary}.
    Consequently, local explanations summarize fANOVA effects evaluated at $x_0$ according to \cref{tab_influence_summary}.
\end{theorem}

With \cref{theorem_influence_summary}, we can distinguish all feature influence measures based on their summary of the \gls*{MT}.
For local explanations with Corollary~\ref{cor_fanova_moebius}, we can view the \gls*{MT} $m(S)$ as the fANOVA effect $f_S$ evaluated at $x_0$.
For individuals, the pure effect captures the main effect $m(i)$, whereas the full effect additionally includes all effects that involve feature $i$, while the partial effect (\gls*{SV}) distributes the higher-order effects equally among involved features.
For joint feature influences, the pure effect captures the effect $m(S)$ and all effects of subsets contained in $S$.
The full effect captures all effects that involve at least one feature in $S$, including the pure effect.
Moreover, the partial effect captures a share of each of these effects with an index-specific weight $w_t^s$ \citep{Shapiq.2024}.
For feature interactions, the pure interaction equals the isolated effect $m(S)$, whereas the full effect is the sum over all effects that involve all features in $S$.
Lastly, the partial effect captures a share of these higher-order effects with an index-specific weight $w_t^s$.
In summary, pure, partial, and full effects yield an increasing influence of higher-order interactions, which we illustrate with the following example.

\paragraph{Illustration of Feature Influence.}
We now illustrate the influence of main and higher-order effects for different feature influence methods. 
We consider three independent features $X_i \sim \mathcal{N}(0,1)$ and a functional relationship of $F_{3\text{int}}(x) = x_1 + x_2 + x_3 + x_1x_2 + x_1x_2x_3$. In this case, c-fANOVA is equivalent to m-fANOVA and b-fANOVA with $b = \mathbb{E}[X]$, and each term in $F$ corresponds to a fANOVA effect.
The weights of fANOVA effects for different feature influences of the local explanation are summarized in \cref{tab:influence_example}.
For individual contributions (rows 1-3), the pure effect captures the main effect $x_1$, whereas the partial effect includes the equally distributed interactions $x_1x_2$ and $x_1 x_2 x_3$.
Moreover, the full effect $\phi^+$ accounts for both interactions with weight $1$.
For joint influence (rows 4-6), the pure effect captures the interaction $x_1x_2$, as well as the individual effects $x_1,x_2$.
The partial and full effects additionally capture the higher-order interaction $x_1x_2x_3$.
Lastly, for interactions (rows 7-9), the pure effect captures the interaction $x_1x_2$, whereas the partial and full effect also account for the higher-order interaction $x_1x_2x_3$.

\begin{table*}[t]
\caption{Game-theoretic measures to quantify individual and joint feature influence, as well as feature interactions based on the \gls*{MT} $m$ of the game $\nu$. All methods capture the pure effect, and differently summarize higher-order interactions.}
\label{tab_influence_summary}
\begin{tabular}{@{}ll||ll@{}}
\toprule
& \textbf{Individual Effect}  & \textbf{Joint Effect}  & \textbf{Interaction Effect} 
\\ \midrule
\textbf{Pure} ($\phi^\emptyset$) & $m(i)$ &  $m(S) + \sum_{\emptyset \neq L\subset S} m(L)$ & $m(S)$
\\
\textbf{Partial} ($\phi^\text{SV/GV/SI}$) & $m(i) + \sum_{T \supset i} \frac{m(T)}{t}$ &  $m(S) + \sum_{T \cap S \notin \{S,\emptyset\}} w^{s}_{\vert T \setminus S\vert} m(T)$ & $m(S) + \sum_{T \supset S} w^{s}_{\vert T \setminus S\vert} m(T)$ 
\\
\textbf{Full} ($\phi^+$) & $m(i) + \sum_{T \supset i}m(T)$ & $m(S) + \sum_{T \cap S \notin \{S,\emptyset\}} m(T)$ & $m(S) + \sum_{T \supset S} m(T)$
\\ \bottomrule
\end{tabular}
\end{table*}

\begin{table}[ht]
\caption{Weights of fANOVA effects in local explanation $\phi_x$ with pure, partial, and full for Individual (Ind), Joint, and Interaction (Int) effects of feature $x_1$ ($i=\{1\}$) and interaction $x_1 \times x_2$ ($ij=\{1,2\}$).}\label{tab:influence_example}
\begin{tabular}{@{}l|l|l|l|l|l@{}}
\toprule
\multicolumn{2}{l|}{\textbf{$F_{3\text{int}}$ fANOVA effects}}  & $x_1$ & $x_2$ & $x_1x_2$ & $x_1x_2x_3$
\\ \midrule
\multirow{3}{*}{\textbf{Ind}} & \textbf{Pure} $\phi_x^\emptyset(i)$
&$1$ 
& $0$ 
& $0$
& $0$
\\
& \textbf{Partial} $\phi_x^{SV}(i)$ 
& $1$ 
& 0
& $\frac{1}{2}$ 
&  $\frac{1}{3}$ 
\\ 
& \textbf{Full} $\phi_x^+(i)$
& $1$
& $0$
& $1$ 
&  $1$ 
\\ \midrule
\multirow{3}{*}{\textbf{Joint}} & \textbf{Pure} $\phi_x^\emptyset(ij)$ 
&$1$ 
& $1$ 
& $1$ 
& $0$
\\
& \textbf{Partial} $\phi_x^{GV}(ij)$ 
&$1$
& $1$
& $1$ 
&  $w^2_1$ 
\\
& \textbf{Full} $\phi_x^+(ij)$ 
& $1$
& $1$
& $1$ 
&  $1$ 
\\ \midrule
\multirow{3}{*}{\textbf{Int}} & \textbf{Pure} $\phi_x^{I\emptyset}(ij)$ 
&$0$ 
& $0$
& $1$ 
& $0$
\\
& \textbf{Partial} $\phi_x^{SI}(ij)$ 
&$0$ 
& $0$
& $1$ 
&  $w^2_1$ 
\\
& \textbf{Full} $\phi_x^{I+}(ij)$ 
& $0$ 
& $0$
& $1$ 
&  $1$ 
\\ \bottomrule
\end{tabular}
\end{table}

\section{UNIFYING EXPLANATIONS}

We now use our previously established framework to categorize existing feature-based explanations, including a guide for practitioners.
A selection of algorithms is displayed in \cref{fig_intro_illustration}, and we refer to \cref{appx_section_table} for a full categorization.

\subsection{Local Feature-Based Explanations}
Local methods explain feature influence on predictions for an instance based on the local explanation game.

\paragraph{Influence via Perturbations.}
For individuals, baseline \citep{DBLP:conf/icml/SundararajanN20}, interventional (int.) \citep{DBLP:conf/aistats/JanzingMB20}, and observational (obs.) SHAP \citep{lundberg_2017_unified} rely on partial effects with b-, m- and c-fANOVA, respectively.
Occlusion-1 \citep{Zeiler_Fergus_2014} and individual conditional expectations (ICEs) \citep{goldstein_peeking_2015} use the full effect and b-fANOVA.
For groups of features, ArchAttribute \citep{tsang2020does} quantifies the pure joint effect with b-fANOVA.
Occlusion-patch and ICE compute the full joint effect with b-fANOVA,
whereas PredDiff-patch \citep{DBLP:conf/iclr/ZintgrafCAW17} uses c-fANOVA. 
Global feature effect methods, like \gls*{PDP} \citep{friedman_greedy_2001} and its conditional variant known as marginal (M) plot \citep{apley_visualizing_2020}, visualize the pure joint effect using m- and c-fANOVA across a feature's range.
\begin{corollary}
    In practice, \gls*{PDP} and M-Plot can be used to compute m- and c-fANOVA, respectively.
\end{corollary}
Inspired by game theory, joint \glspl*{SV} \citep{Harris.2022} and Faith-SHAP \citep{tsai_faith-shap_2022} capture partial joint and interaction effects with b-fANOVA.
The int. Shapley-GAM \citep{bordt2022shapley} and int. k-\glspl*{SV} \citep{lundberg_consistent_2019} use pure and partial effects of m-fANOVA, whereas obs. Shapley-GAM and k-\glspl*{SV} rely on pure and partial effects of c-fANOVA.

\paragraph{Influence via Gradients.}
For some model classes, most prominently in deep learning, gradients with respect to individual features can be calculated.
With \cref{thm_bfanova_taylor}, we linked Taylor interactions and b-fANOVA.
Gradient-based methods distribute the Taylor interactions, and partially account for b-fANOVA effects, as described by \cite{Deng2024Unify}.
An advantage of gradients is that effects without the feature of interest vanish $\frac{\partial F(x)}{\partial x_i} = \sum_{S \subseteq D: i \in S} \frac{\partial f_S(x)}{\partial x_i}$.
A clear disadvantage is that the derivatives $\frac{\partial f_S}{\partial x_i}$ differ from the actual effects $f_S$.
The gradient evaluated at $x_0$ is a linear approximation of $F$ at $x_0$, which yields a rough influence measure (Grad$\times$Input).
Other methods, like integrated gradients (IG) \citep{DBLP:conf/icml/SundararajanTY17} or ALE \citep{apley_visualizing_2020}, compute average gradients, where ALE relies on conditional expectations and IG on a path integral \citep{Friedman_2004}.
IG relates to $F$ via the Aumann-\gls*{SV} \citep{aumann_shapley1974}, but the choice of baseline has a strong impact \citep{sturmfels2020visualizing,Koenen_Wright_2024}.

\subsection{Global Feature-Based Explanations}
Global methods explain the feature influence on a global level, independent of specific instances.

\paragraph{Influence via Sensitivity Analysis.}
Sensitivity analysis studies the variance of $F_S$ \citep{Owen_2013}, captured by the global sensitivity game $\nu^{\text{sens}}$.
Sobol' indices, $\underline{\tau}_S,\overline{\tau}_S$ \citep{Sobol_2001} correspond to the pure and full joint effect, respectively.
The superset measure $ \Upsilon_S^2$ \citep{hooker_discovering_2004} is the full interaction effect.
To measure the strength of interaction, \cite{friedman_predictive_2008} proposed the H-statistic, which corresponds to the pure interaction effect.
Moreover, \cite{Owen_2014} proposed the partial effect for individuals.
Above measures require independent features and summarize m-fANOVA effects, which by \cref{theorem_sensitivity} are the \gls*{MT}.
For dependent features, \cite{Owen.2017} proposed the partial effect for c-fANOVA, and demonstrated intuitive results for a variety of distributions and models.

\paragraph{Influence via Performance (Risk).}
Feature-based explanations using the expected model loss include the data-generating process and provide global feature importance scores based on the risk game.
Permutation feature importance (PFI) \citep{Breiman2001, fisher2019all} and conditional feature importance (CFI) \citep{strobl2008conditional, molnar2023model} quantify the full individual effect based on m- and c-fANOVA.
In contrast, SAGE \citep{covert2020sage} computes the partial individual effect for c-fANOVA (theoretically) and m-fANOVA (practically).



\subsection{A Guide for Practitioners}
Our framework facilitates the process of selecting a suitable feature-based explanation for a given interpretation goal to answering the following four questions:

\textbf{(1) Explanation game:} Do we explain a single instance (local) or the general model behavior (global)? 

\textbf{(2) Explanation Type:} Do we report influence of \emph{individuals} or groups (\emph{joint}), or synergies (\emph{interaction})? 

\textbf{(3) Imputation method:} Should the feature distribution not (\emph{baseline}), partially (\emph{marginal}), or fully (\emph{conditional}) be captured by the explanation?

\textbf{(4) Higher-order interactions:} Should higher-order interactions not (\emph{pure}), \emph{partially}, or \emph{fully} influence the explanation? 


\paragraph{Explanation Cost.}
Computational complexity is typically assessed by the number of model calls to $F$.
In the local explanation game the primary factor is the value function $F_S$, for which a single model call suffices for b-fANOVA, whereas m- and c-fANOVA rely on approximations of expectations, requiring varying numbers of evaluations depending on the desired estimation quality \citep{baniecki2025efficient}.
For global explanations, local explanation game values must be computed across the feature distribution or the labeled data.
Once game values are available, pure and full individual and joint effects each require \emph{two} game evaluations, while pure and full interactions for a set of size $s$ demand $2^s$ evaluations.
Partial effects are the most computationally expensive, necessitating $2^d$ game evaluations or their approximations. Importantly, game evaluations can be reused across different explanation games, types, and effects.

\section{EXPERIMENTS}\label{sec_experiments}
We now showcase our framework on synthetic and real-world datasets. 
For m-fANOVA, we use a background dataset of $512$ samples, and for c-fANOVA the synthetic ground-truth.
Further details and experiments are described in \cref{appx_sec_experiment_setup} and \cref{appx_sec_further_exp}.
Our code is available at \url{https://github.com/FFmgll/unifying_feature_based_explanations}.

\begin{figure}[t]
    \includegraphics[width=\columnwidth]{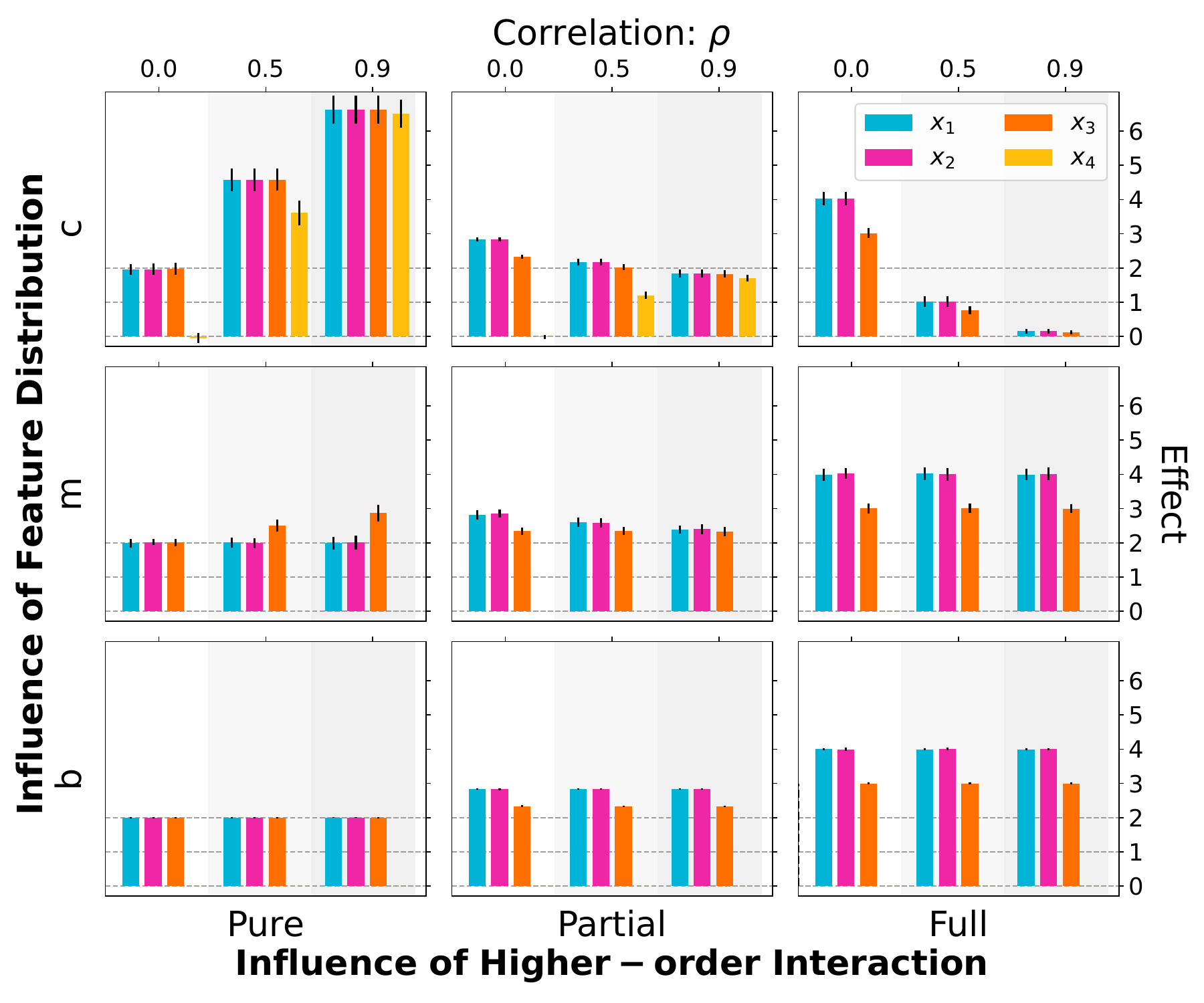}
    \caption{Local explanations for the instance $x = (1,1,1,1)$ averaged over $30$ repetitions of varying random seeds (fluctuation is shown by error bars). Note that the model is correctly specified as $F(x) = {\color{myblue}2x_1} + {\color{mypink}2x_2} + {\color{myorange}2x_3} + {\color{myblue}x_1} {\color{mypink}x_2} + {\color{myblue}x_1} {\color{mypink}x_2} {\color{myorange}x_3}$.}
    \label{fig_sim_results}
\end{figure}

\subsection{Synthetic Example}\label{sec_experiments_synth}
We consider a 4-dimensional normal distribution $X \sim \mathcal{N}(0, \Sigma)$ with $\sigma_{ii}^2 = 1$ and $\sigma_{ij} = \rho$ for $i \neq j$.
We then vary $\rho \in \{0, 0.5, 0.9\}$, i.e., from uncorrelated to highly correlated.
The target is given by $F(x) = 2 x_1 + 2 x_2 + 2 x_3 + x_1 x_2 + x_1 x_2 x_3$ and corrupted by noise $\epsilon \sim \mathcal{N}(0,0.01)$.
We then fit a correctly specified linear model (LM) to each setting.
\cref{fig_sim_results} shows the individual local explanations of our framework for one instance for varying feature dependencies.
As expected, explanations based on b-, m-, and c-fANOVA only differ when features are correlated. While b-fANOVA explanations are not affected by feature dependencies, m-fANOVA uses the joint marginal distribution, which particularly influences pure feature effects. C-fANOVA is highly influenced by feature dependencies due to cross-correlation effects. Thus, for $\rho = 0.9$ all features including the noise feature result in equally high pure effects. In contrast, full effects are small, since the remaining three features capture the effects of the feature of interest through cross-correlation effects. For b- and m-fANOVA, the comparison between pure, partial, and full clearly shows the presence of feature interactions between $x_1$, $x_2$, and $x_3$ with stronger effects between the first two features.

\subsection{Real-World Application}\label{sec_experiments_real_world}
We now demonstrate differences between the explanation methods in two real-world settings, and uncover interactions in local and global explanations.

\paragraph{California Housing.} We consider the global sensitivity game on the \emph{California housing} dataset \citep{Kelley.1997}.
We display pure (\gls*{MT}) and full (superset measures $\Upsilon^2$) interaction effects for the first and second order based on m-fANOVA in \cref{fig_california}.
Pure effects (a) show that features \emph{Latitude (La)} and \emph{Longitude (Lo)} have a strong individual importance with a negative pairwise interaction indicating a redundancy between these two features.
On the contrary, different full effects (b) indicate the presence of higher-order interactions.
While \emph{MedianIncome (MI)} shows almost no main (individual pure) effect, it has a strong full individual effect due to higher-order interactions particularly including \emph{La} and \emph{Lo} highlighted by the full pairwise interaction effects.

\begin{figure}[t]
    \centering
    \begin{minipage}[c]{0.04\columnwidth}
        \textbf{a)}
    \end{minipage}
    \hfill
    \begin{minipage}[c]{0.39\columnwidth}
        \includegraphics[width=\textwidth]{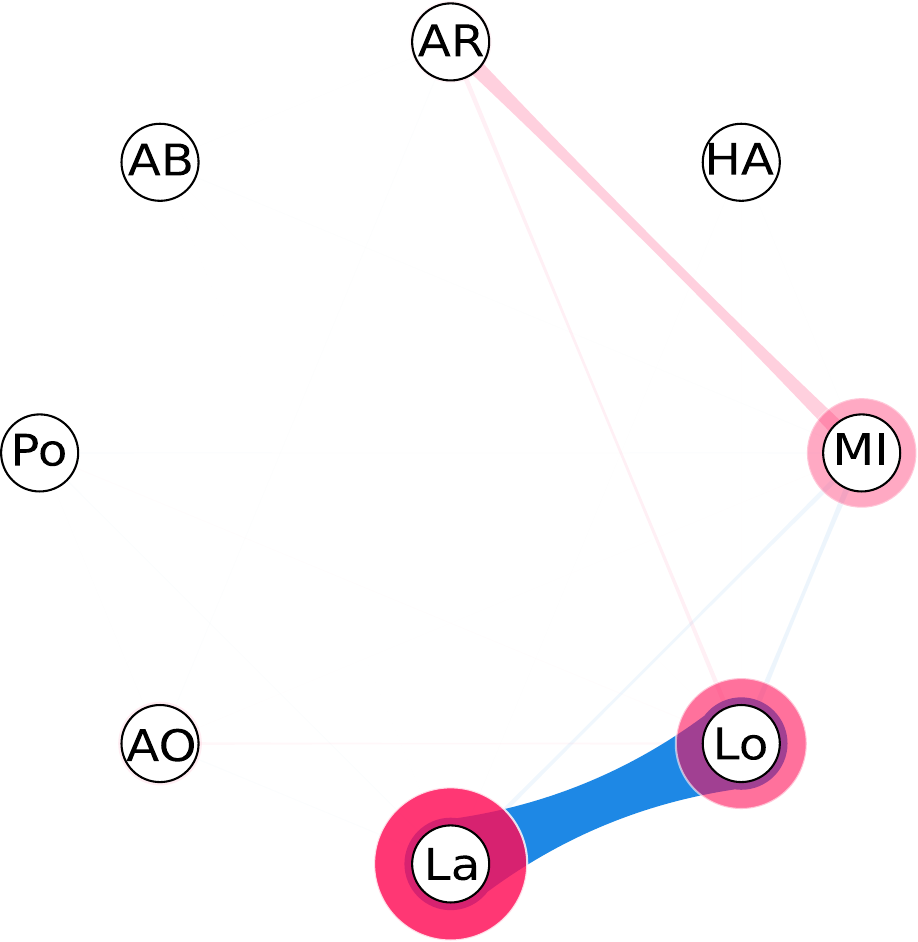}
    \end{minipage}
    \hfill
    \begin{minipage}[c]{0.04\columnwidth}
        \textbf{b)}
    \end{minipage}
    \hfill
    \begin{minipage}[c]{0.39\columnwidth}
        \includegraphics[width=\textwidth]{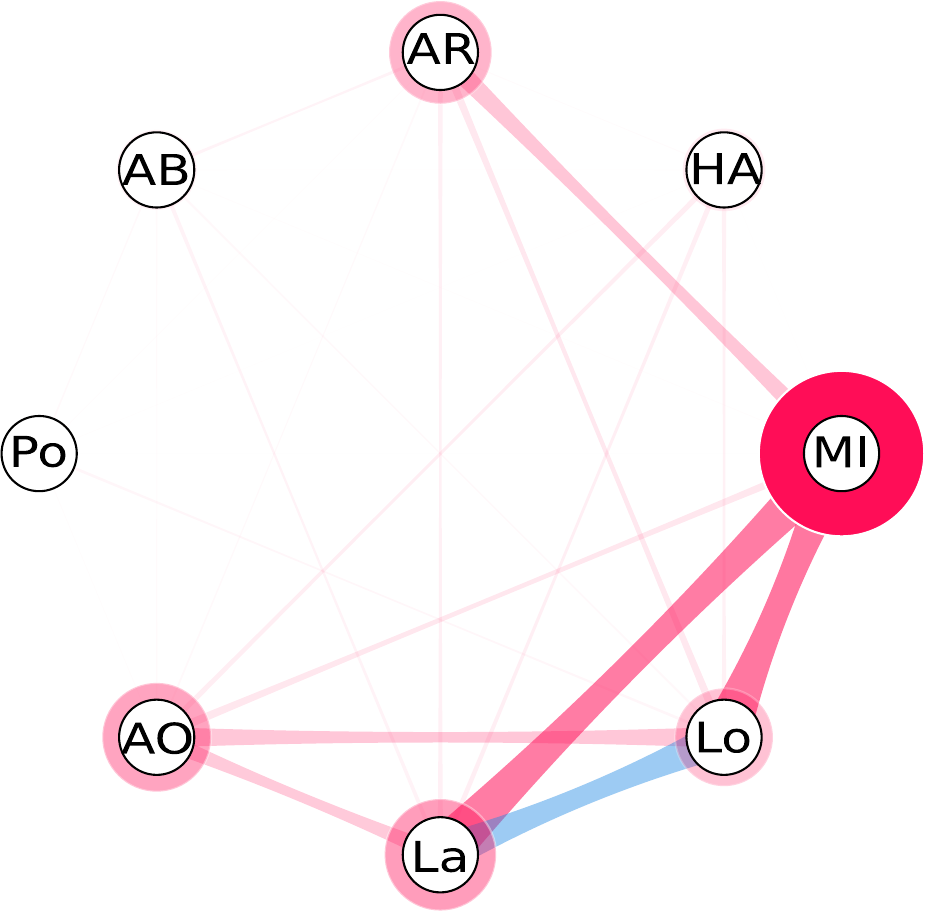}
    \end{minipage}
    \vspace{0.5em}
    \caption{Pure \textbf{(a)} and full \textbf{(b)} individual effects and two-way interactions for the \emph{global sensitivity game} of an XGBoost model trained on \emph{California housing}. Blue and red colors denote a reduction and increase of variance.}
    \label{fig_california}
\end{figure}

\paragraph{Sentiment Language Model.} In \cref{fig_language}, we further investigate local explanations on a sentiment analysis language model fine-tuned on the \emph{IMDB} \citep{Maas.2011} dataset.
The sentiment scores range between $\left[-1,1\right]$, where values close to $-1$ denote
a \emph{negative} sentiment and values close to $1$ correspond to a \emph{positive} sentiment.
We provide the model the sentence \emph{``The acting was bad, but the movie enjoyable''}.
The sentiment for this sentence is $\approx 0.8117$, whereas the baseline sentiment is $\approx 0.5361$, indicating a bias towards positive reviews.
Individual effects (left) strongly vary between pure, partial and full effects, indicating the presence of higher-order interactions.
Specifically, pure individual effects highlight the positive influence of ``enjoyable'', whereas ``bad'' and ``but'' have a negative influence. Interestingly, ``acting'' also shows a negative influence, which possibly indicates a bias in the prediction. The partial individual effect (\gls*{SV}) shows a stronger positive influence for ``enjoyable'', and a weaker negative influence for ``acting'' and ``bad''. In contrast, ``but'' now receives a positive partial influence, which indicates its involvement in positive higher-order interactions.
In fact, the partial interaction effects up to order $2$ (2-\glspl*{SV}, right) show that (``bad'',``but''), and  (``but'',``enjoyable'') have a strong positive partial interaction effect.
This observation is intuitive, since (``bad'',``but'') suggests that a negative sentiment can be reversed in meaning.
Further experiments on other games, models, and effects are provided in Appendix~G.

\begin{figure}[t]
    \hfill
    \begin{minipage}[c]{0.52\columnwidth}
        \tiny{\textbf{Pure, Individual}}\\[0.3em]
        \includegraphics[width=\textwidth]{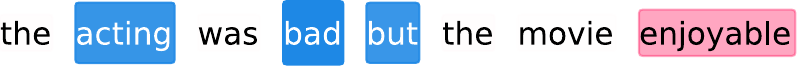}\\[1em]
        \tiny{\textbf{Partial, Individual}}\\[0.3em]
        \includegraphics[width=\textwidth]{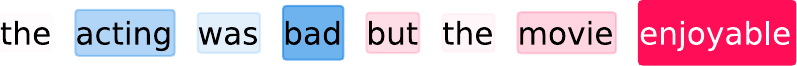}\\[1em]
        \tiny{\textbf{Full, Individual}}\\[0.3em]
        \includegraphics[width=\textwidth]{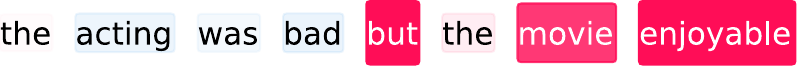}
        \vspace{0.3em}
    \end{minipage}
    \hfill
    \begin{minipage}[c]{0.45\columnwidth}
        \includegraphics[width=\textwidth]{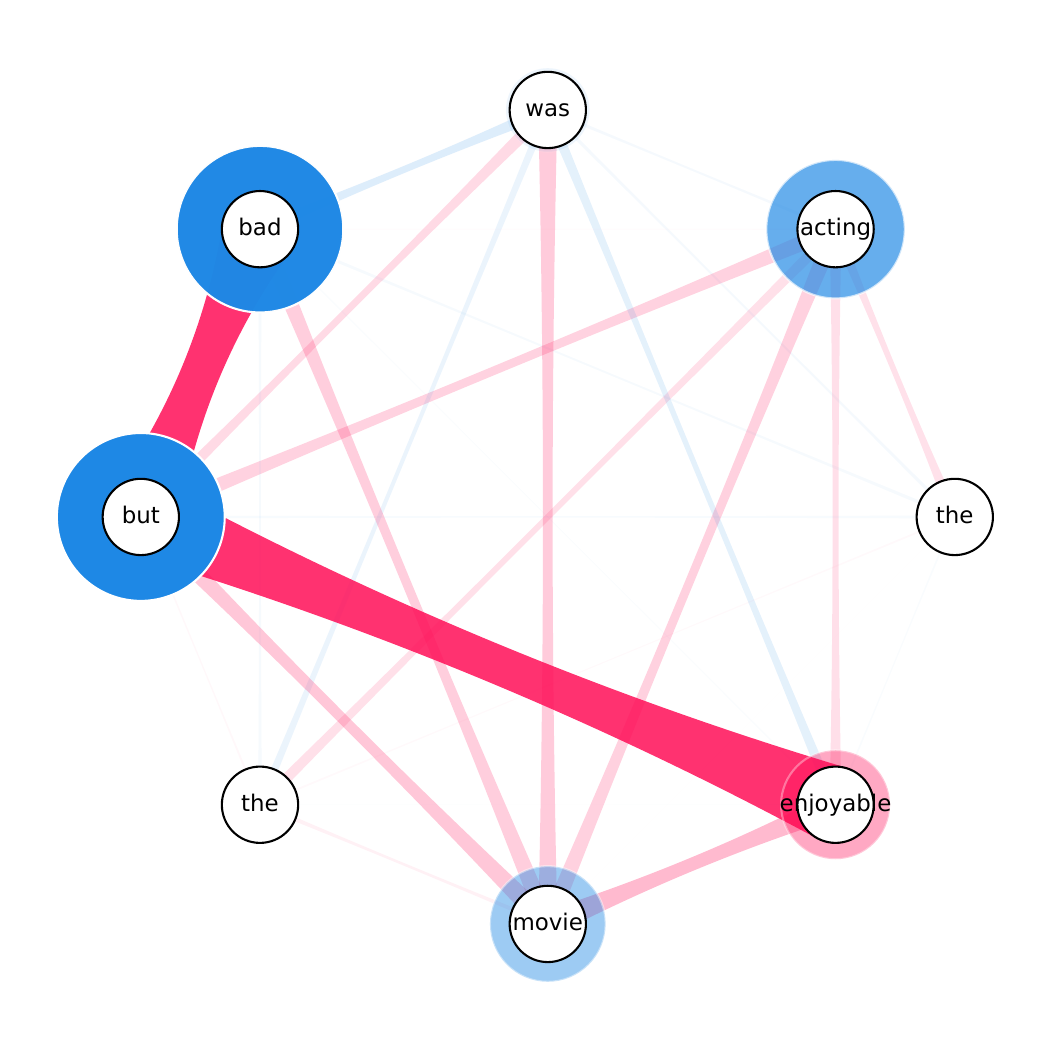}
    \end{minipage}
    \vspace{0.5em}
    \caption{Individual \textbf{(left)} and second order partial interaction \textbf{(2-\glspl*{SV}, right)} effects for the \emph{local explanation game} of a sentiment analysis language model. While ``but'' has a negative pure effect, the partial and full effects are positive, indicating positive higher-order interactions. In fact, (``bad'',``but''), and  (``but'',``enjoyable'') have a strong positive partial interaction effect.}
    \label{fig_language}
\end{figure}

\section{RELATED WORK}

Several unified frameworks for individual feature influence measures were introduced: \cite{lundberg_2017_unified} propose SHAP, which applies the \gls*{SV} across different methods.
\cite{DBLP:conf/icml/SundararajanN20} analyze the SV with baseline and marginal imputations, i.e. m- and b-fANOVA.
\cite{DBLP:conf/nips/HanSL22} characterize local explanations using a function approximation.
\cite{lundstrom2023unifying} introduce synergy functions to interpret gradient-based explanations, which correspond to b-fANOVA effects, whereas \cite{Deng2024Unify} unified gradient-based explanations with Taylor interactions, a fine-grained component of b-fANOVA effects.
\cite{covert2021explaining} propose a framework for local and global explanations, which summarizes perturbation-based methods.
\cite{tan2023considerations} unify global explanations by an additive decomposition and study how non-additive components influence the explanation.
Beyond individual influence measures,
\cite{hiabu2022unifying} unify the SV with PDP using a functional decomposition with marginal distributions, corresponding to m-fANOVA.
\cite{bordt2022shapley} introduced the Shapley-GAM and unified \glspl*{SI} for marginal and conditional distributions, which relates to m-fANOVA and c-fANOVA.

\section{LIMITATIONS}
Our unified framework covers a variety of local and global feature-based explanations, where \cref{appx_sec_extensions_framework} highlights a broader perspective.
However, instance-based explanations, such as counterfactuals \citep{DBLP:journals/corr/abs-1711-00399}, are not included.
Moreover, feature-based explanations that rely on distance measures, such as LIME \citep{DBLP:conf/kdd/Ribeiro0G16} are related to SHAP \citep{lundberg_2017_unified}, but not fully covered.
Similarly, some gradient-based methods summarize derivatives on a more fine-grained level than the related b-fANOVA effects \citep{Deng2024Unify}.
For practical considerations, the m- and c-fANOVA requires the \emph{approximation} of (conditional) expectations, where a variety of methods exist \citep{olsen2024comparative}, which we did not address here.
Moreover, the \gls*{SV} and \glspl*{SI} require an exponential amount of game evaluations, which also necessitates approximation methods \citep{DBLP:journals/jmlr/MitchellCFH22,Chen2023Overview_ExplainabilityWithShapley,Fumagalli.2023,fumagalli2024kernelshapiq,Kolpaczki.2024}.
For tree-based models with distributional assumptions, it was shown that exact calculation is feasible \citep{lundberg_consistent_2019,DBLP:conf/aaai/MuschalikFHH24}.

\section{CONCLUSION}

We combined fANOVA, a popular tool in statistics, and cooperative game theory, to present a unified framework for feature-based explanations. 
Our framework includes three \emph{explanation games}, one for local, and two for global explanations in terms of sensitivity and performance.
Within each explanation game, we then presented three types of explanations: \emph{Individual} and \emph{joint} feature effects quantify the \emph{features' influence} for individuals and groups of individuals, whereas \emph{interaction} effects measure \emph{synergies and redundancies} of groups of features.
By choosing b-, m- or c-fANOVA, practitioners are able to include \emph{increasing influence of the features' distribution} in their explanations, which we showcase on synthetic examples.
We further proposed three game-theoretic concepts, including the \gls*{SV} and \glspl*{SI}, to quantify pure, partial and full effects that yield an \emph{increasing influence of higher-order interactions}, which we again showcase on synthetic examples.
Finally, we showed that our framework unifies a variety of existing feature-based explanations, offering clear guidance for their interpretation and application.
We demonstrated its usefulness on synthetic and real-world datasets.

\clearpage
\section*{ACKNOWLEDGEMENTS}
We gratefully thank the anonymous reviewers for their valuable feedback for improving this work.
Fabian Fumagalli and Maximilian Muschalik gratefully acknowledge funding by the Deutsche Forschungsgemeinschaft (DFG, German Research Foundation): TRR 318/1 2021 – 438445824. Julia Herbinger gratefully acknowledges funding by the German Ministry for Education and Research through project Explaining 4.0 (ref. 01IS200551), the European Union’s Horizon Europe Innovation program through project SPIN-FERT (ref. 101157265), and the German Federal Ministry of Education and Research (BMBF) through project DCropS4OneHealth-2 (ref. 16LW0528K).
The authors of this work take full responsibility for its content.

\renewcommand{\bibname}{REFERENCES}
\bibliography{bib}

\newpage
\appendix
\onecolumn 

\startcontents[sections]
\printcontents[sections]{l}{1}{\setcounter{tocdepth}{3}}

\clearpage
\section{PROOFS}\label{appx_sec_proofs}
In this section, we provide proofs for the claims made in the main paper.
We will frequently use the following lemma.

\begin{lemma}[Inclusion-Exclusion Principle]\label{appx_lem_inexclusion}
    For a set $S$ and a subset $L \subseteq S$, the following holds
    \begin{align*}
        \sum_{T\subseteq S: T \supseteq L} (-1)^{s-t} = \begin{cases}
            1, \text{ if } L = S, \\
            0, \text{ else.}
        \end{cases}  
    \end{align*}
\end{lemma}

\begin{proof}
    It is clear that for $L=S$, we have $T=S$, and thus the sum equals 1.
    We therefore assume now $L \subset S$, and proceed to show
    \begin{align*}
         \sum_{T\subseteq S: T \supseteq L} (-1)^{s-t} 
         &= \sum_{t=\ell}^{s} (-1)^{s-t} \sum_{T \subseteq S: T \supseteq L, \vert T \vert = t} 
         =\sum_{t=\ell}^{s} (-1)^{s-t} \binom{s-\ell}{t-\ell}
         \\
         &= (-1)^{s+\ell}\sum_{t=0}^{s-\ell} (-1)^t \binom{s-\ell}{t} =(-1)^{s+\ell} (-1+1)^{s-\ell} = 0,
    \end{align*}
    where we have used the binomial theorem in the last row, and $s-\ell > 0$.
\end{proof}

\subsection{Proof of Theorem 1}
\begin{proof}
To prove the result, we compute the b-fANOVA effect $f_S^{(b)}$, which is given by definition as
\begin{align*}
    f_S^{(b)}(x) := F_S^{(b)}(x) - \sum_{T \subset S} f_T^{(b)}(x) = \sum_{T \subseteq S}(-1)^{s-t} F_T^{(b)}(x) = \sum_{T \subseteq S}(-1)^{s-t} F(x_T,b_{-T}),
\end{align*}
where the explicit form of $f^{(b)}_S$ is standard in the fANOVA literature, cf. \cref{appx_sec_proof_cor1} for a derivation.
Clearly, for $S= \emptyset$, we have $f_\emptyset^{(b)} = F(b)$.
For $S$ with $s>0$, by assumption, $F$ can be represented by its infinite Taylor expansion evaluated at $[x_S,b_{-S}]$, cf. Eq.~(\ref{appx_eq_taylor_exp}), as

\begin{align*}
    f_S^{(b)}(x) = \sum_{T \subseteq S}(-1)^{s-t} F(x_T,b_{-T}) = \sum_{T \subseteq S} (-1)^{s-t} \sum_{\kappa \in \mathbb{N}^d_0} \left(C(\kappa) \cdot \nabla F(\kappa) \cdot \prod_{i \in T} (x_i - b_i)^{\kappa_i} \prod_{i \in -T} (b_i - b_i)^{\kappa_i}\right).
\end{align*}

In the sum over $\kappa$ terms vanish, if $\kappa_i > 0$ for a $i \in -T$, since then $(b_i-b_i)^{\kappa_i}$ is zero.
Therefore, only terms remain, where $\kappa \in \Omega_{L}$ with $L \subseteq T$.
Hence, by definition of $I$, cf. \cref{appx_sec_background_grad}, and with exchanging sums
\begin{align*}
    f_S^{(b)}(x) &= \sum_{T \subseteq S} (-1)^{s-t} \sum_{L \subseteq T} \sum_{\kappa \in \Omega_L} \left(C(\kappa) \cdot \nabla F(\kappa) \cdot \prod_{i \in L} (x_i - b_i)^{\kappa_i}\right) 
    \\
    &= \sum_{T \subseteq S} (-1)^{s-t} \sum_{L \subseteq T} \sum_{\kappa \in \Omega_L} I(\kappa)
    \\
    &=\sum_{L \subseteq S} \sum_{T\subseteq S: T \supseteq L} (-1)^{s-t}\sum_{\kappa \in \Omega_L}  I(\kappa) 
    \\
    &=\sum_{L \subseteq S} \sum_{\kappa \in \Omega_L}  I(\kappa) \sum_{T \subseteq S: T \supseteq L} (-1)^{s-t}.
\end{align*}
By \cref{appx_lem_inexclusion}, this sum reduces to
\begin{align*}
    f_S^{(b)}(x) = \sum_{\kappa \in \Omega_S} I(\kappa) = 
    \begin{cases}
        \psi(i), \text{ if } S=i,
        \\
        J(S), \text{ if } s > 1,
    \end{cases}
\end{align*}
which concludes the proof.
\end{proof}

\subsection{Proof of Theorem 2}
\begin{proof}
    The fANOVA effects $f_S$ are determined via the value function $F_S$.
    Hence it suffices to show that $F_S$ are equivalent under these conditions, which was already partially stated by \cite{lundberg_2017_unified}.
    For a linear model $F(x) = \sum_{i \in D} \beta_i x_i$, we have due to linearity of expectations
    \begin{equation*}
        F^{(m)}_S(x) =  \mathbb{E}[ F(x_S,X_{-S})] = \sum_{i \in S} \beta_i x_i + \sum_{i \in -S} \beta_i \mathbb{E}[X_i] = F^{(b)}_S(x),
    \end{equation*}
    which concludes equality of m- and b-fANOVA for a linear model.
    For independent features,
    \begin{align*}
        F_S^{(c)}(x) = \mathbb{E}[ F(X) \mid X_S = x_S] = \mathbb{E}[ F(x_S,X_{-S}) \mid X_S = x_S] = \mathbb{E}[ F(x_S,X_{-S})] = F^{(m)}_S(x),
    \end{align*}
    since $X_S \perp X_{-S}$,
    which shows equality of m- and c-fANOVA.
    Second, assume $F$ is additionally multilinear (linear in each feature), i.e. for a feature $i \in D$ and $\tilde x := [x_1,\dots,\alpha x_i + \beta, \dots, x_d]$ it holds $F(\tilde x) = \alpha F(x) + \beta$.
    This is known as a multilinear map, and thus $F$ can be represented by
    \begin{align*}
        F(x) = \sum_{L\subseteq D} c_L \prod_{i \in L} x_i,
    \end{align*}
    where $c_L \in \mathbb{R}$ are some constants.
    Then, we have for the value function of m-fANOVA and $b_i := \mathbb{E}[X_i]$ due to independence and linearity of expectations
    \begin{align*}
        F_S^{(m)}(x) &= \mathbb{E}[F(x_S,X_{-S})] = \sum_{L \subseteq D} c_L \mathbb{E}[\prod_{i \in L\cap S}x_i \prod_{i \in L \cap -S} X_i] =   \sum_{L \subseteq D} c_L \prod_{i \in L\cap S}x_i \prod_{i \in L \cap -S}\mathbb{E}[X_i] 
        \\
        &=\sum_{L \subseteq D} c_L \prod_{i \in L\cap S}x_i \prod_{i \in L \cap -S}b_i 
        = F(x_S,b_{-S}) = F^{(b)}(x_S,b_{-S}),
    \end{align*}
    which concludes the proof.
\end{proof}

\begin{remark}
    Results derived from \cref{theorem_fanova_equivalence} for multilinear functions can be extended to linear combinations of multilinear functions, since expectations are linear. In summary, every model $F$ that can be described as the sum of several weighted multilinear functions will have similar b- and m-fANOVA effects for independent features, if $b := \mathbb{E}[X]$.
\end{remark}

\subsection{Proof of Theorem 3}
\begin{proof}
    In the following, our goal is to show that
    \begin{align*}
        m^{(\text{sens})}(S) = \mathbb{V}[f^{(c)}_S(X)] + \sum_{L \cup L' = S, L \neq L'} \text{cov}(f^{(c)}_L(X),f^{(c)}_{L'}(X)).
    \end{align*}
    We begin by computing the global sensitivity game $\nu^{(\text{sens})}$ as the variance of the value function $F^{(c)}_T = \sum_{S \subseteq T} f^{(c)}_S$ with
    \begin{align*}
        \nu^{(\text{sens})}(T) := \mathbb{V}[F^{(c)}_T(X)] 
        = \mathbb{V}[\sum_{L \subseteq T} f^{(c)}_L(X)]
        = \sum_{L \subseteq T} \mathbb{V}[f_L^{(c)}(X)] + \sum_{L,L' \subseteq T: L \neq L'} \text{cov}(f^{(c)}_{L}(X),f^{(c)}_{L'}(X)).
    \end{align*}
    The \gls{MT} is then given by definition as
    \begin{align*}
        m^{(\text{sens})}(S) := \sum_{T \subseteq S}(-1)^{s-t} \nu^{(\text{sens})}(T) = \sum_{T \subseteq S} (-1)^{s-t}\left(  \sum_{L \subseteq T} \mathbb{V}[f_L^{(c)}(X)] + \sum_{L,L' \subseteq T: L \neq L'} \text{cov}(f^{(c)}_{L}(X),f^{(c)}_{L'}(X))\right)
    \end{align*}
    We now proceed with each part separately:
    By exchanging sums and using \cref{appx_lem_inexclusion}, the first part of the sum reduces to
    \begin{align*}
        \sum_{T \subseteq S} (-1)^{s-t}  \sum_{L \subseteq T} \mathbb{V}[f_L^{(c)}(X)] &= \sum_{L \subseteq S} \mathbb{V}[f_L^{(c)}(X)] \sum_{T \subseteq S: T \supseteq S} (-1)^{s-t} =  \sum_{L \subseteq S} \mathbb{V}[f_L^{(c)}(X)] \mathbf{1}_{L=S} = \mathbb{V}[f_S^{(c)}(X)].
    \end{align*}
    For the second part, a well-known result from sensitivity analysis \citep{Owen_2013} is that for independent features, $f_L$ and $f_{L'}$ are independent as well for $L \neq L'$, and thus their covariance is zero.
    Hence, for independent features, the \gls{MT} is equal to the first part.
    For dependent features, by exchanging sums, the second part becomes
    \begin{align*}
        \sum_{T \subseteq S} (-1)^{s-t} \sum_{L,L' \subseteq T: L \neq L'} \text{cov}(f^{(c)}_{L}(X),f^{(c)}_{L'}(X)) = \sum_{L,L' \subseteq S: L \neq L'} \text{cov}(f^{(c)}_{L}(X),f^{(c)}_{L'}(X)) \sum_{T\subseteq S: T \supseteq L, T\supseteq L'} (-1)^{s-t}.
    \end{align*}
    Moreover by applying \cref{appx_lem_inexclusion}, we obtain for the last sum
    \begin{align*}
       \sum_{T\subseteq S: T \supseteq L, T\supseteq L'} (-1)^{s-t} = \sum_{T\subseteq S: T \supseteq L \cup L'} (-1)^{s-t} = \mathbf{1}_{L \cup L' = S},
    \end{align*}
    and thus for the previous sum 
    \begin{align*}
        \sum_{L,L' \subseteq S: L \neq L'} \text{cov}(f^{(c)}_{L}(X),f^{(c)}_{L'}(X)) \sum_{T\subseteq S: T \supseteq L, T\supseteq L'} (-1)^{s-t} 
        &= \sum_{L,L' \subseteq S: L \neq L'} \text{cov}(f^{(c)}_{L}(X),f^{(c)}_{L'}(X)) \mathbf{1}_{L\cup L' =S} 
        \\
        &= \sum_{L \cup L' = S, L \neq L'} \text{cov}(f^{(c)}_{L}(X),f^{(c)}_{L'}(X)).
    \end{align*}
    Together with the first part, we obtain
        \begin{align*}
        m^{(\text{sens})}(S) = \mathbb{V}[f^{(c)}_S(X)] + \sum_{L \cup L' = S, L \neq L'} \text{cov}(f^{(c)}_L(X),f^{(c)}_{L'}(X)),
    \end{align*}
    which concludes the proof.
\end{proof}

\subsection{Proof of Theorem 4}
\begin{proof}
Our goal is to show that the \gls{MT} is distributed according to \cref{tab_influence_summary}.
Note that joint effects and interaction effects, as well as the formulas in \cref{tab_influence_summary}, reduce to individual effects for $S=i$, and thus it suffices to prove the results for general $S \subseteq D$ and joint and interaction effects.
In the following, we heavily rely on the recovery property of the \gls{MT} \citep{rota1964foundations}
\begin{align*}
    \nu(T) = \sum_{S\subseteq T} m(S).
\end{align*}

\paragraph{Joint effects.}
The pure joint effect is given by 
\begin{align*}
    \phi^\emptyset(S)= \nu(S)-\nu(\emptyset) = \sum_{L \subseteq S} m(L) - m(\emptyset) = \sum_{\emptyset \neq L\subseteq S} m(L) = m(S) + \sum_{\emptyset \neq L\subset S} m(L),
\end{align*}
    as shown in \cref{tab_influence_summary}.
    The full joint effect is given by
\begin{align*}
    \phi^+(S)&= \nu(D)-\nu(-S) = \sum_{L \subseteq D} m(L) - \sum_{L \subseteq -S}m(L) = \sum_{L \subseteq D: L \not\subseteq -S} m(L) = \sum_{L \subseteq D: L \cap S \neq \emptyset} m(L)
    \\
    &=m(S) + \sum_{L \subseteq D: L \cap S \notin \{\emptyset,S\}} m(L),
\end{align*}   
as shown in \cref{tab_influence_summary}.
Lastly, the partial joint effect is a well-known result of \glspl{GV}, cf. \citep[Proposition 4.3]{DBLP:journals/dam/MarichalKF07}.

\paragraph{Interaction effects.}
For interaction effects, the pure interaction effect is by definition the \gls{MT} $m(S)$, as shown in \cref{tab_influence_summary}.
Moreover, the full interaction effect is by definition the Co-\gls{MT}, where the representation in \cref{tab_influence_summary} is a well-known result from cooperative game theory, cf. \citep[Table 3, third row, second column]{grabisch2000equivalent}.
Lastly, the partial effect is a well-known result for interaction indices, cf. \citep[Proposition 4.1]{Fujimoto.2006}.
Additionally, for \glspl{SI}, the representations are given in the corresponding paper and are summarized in \citep[Theorem 8]{bordt2022shapley}.

Since joint and interaction effects reduce to individual effects for $S=i$, this concludes the proof.
\end{proof}

\subsection{Proof of Corollary 1}\label{appx_sec_proof_cor1}

\begin{proof}
    This result follows directly by induction over the size of subsets $s$, where we omit the dependence of $x_0$ in $\nu_{x_0}$ and $m_{x_0}$ for readability.
    For $s=0$, we have $f_{\emptyset} = F_\emptyset = \nu(\emptyset) = m(\emptyset)$.
    According to the definition of fANOVA effects for individuals $(s=1)$, we have
    \begin{align*}
        f_i = F_i - \sum_{T \subset i} f_T = F_i - F_\emptyset = \sum_{T \subseteq i} (-1)^{1-t} F_T = m(i).
    \end{align*}
    Moreover, assume the statement holds for $1,\dots,s-1$ features, then  
    \begin{align*}
        f_S &= F_S - \sum_{T \subset S} f_T = F_S - \sum_{T\subset S} \sum_{L \subseteq T} (-1)^{t-\ell} F_L = F_S - \sum_{L \subseteq S} F_L \sum_{T\subset S: T\supseteq L} (-1)^{t-\ell} 
        \\
        &= F_S - \sum_{L \subseteq S} F_L \left(\sum_{T\subseteq S: T\supseteq L} (-1)^{t-\ell} - (-1)^{s-\ell}\right) 
        = F_S -\sum_{L \subseteq S} F_L (-1)^{s-\ell}\sum_{T\subseteq S: T\supseteq L} (-1)^{s-t} + \sum_{L \subseteq S} F_L (-1)^{s-\ell}
        \\
        &= F_S  - F_S  + \sum_{L \subseteq S} F_L (-1)^{s-\ell}= m(S),
    \end{align*}
    where we have used \cref{appx_lem_inexclusion} in the last row.
\end{proof}
\subsection{Proof of Corollary 2}
The PD function $\phi^{\text{PD}}$ for a set of features $S \subseteq D$ at $x$ is defined as
\begin{align*}
    \phi^{\text{PD}}_x(S) = \mathbb{E}[F(x_S,X_{-S})] = F^{(m)}_S(x).
\end{align*}
On the other hand the M plot function is given by
\begin{align*}
    \phi_x^{\text{M}}(S) = \mathbb{E}[F(X) \mid X_S = x_S] = F^{(c)}_S(x).
\end{align*}

In summary, the underlying functions of \gls{PDP} and M plot correspond to the value functions of $m$- and $c$-fANOVA.
In practice, Monte Carlo sampling is used to approximate these functions \citep{friedman_greedy_2001,apley_visualizing_2020}.
Using $F_S(x)$, we can compute m- and c-fANOVA, respectively.

\clearpage
\section{EXTENSIONS BEYOND OUR FRAMEWORK}\label{appx_sec_extensions_framework}
In this section, we additionally provide some details on a broader perspective of the framework.

\subsection{Beyond b-, m- and c-fANOVA}
In our framework, we introduced three types of fANOVA decompositions, which correspond to the most prominent examples used in \gls{XAI} literature.
The decomposition is specified by the choice of imputation method in the value function $F_S$.
Less prominent examples include re-training of models \citep{DBLP:journals/dke/StrumbeljKR09}, where for a subset of features $S \subseteq D$ the model is fit to the data using only the features in $S$.
While b-, m-, and c-fANOVA compress the feature distribution information in the explanation, model re-training additionally includes the model fitting procedure.
In future work, it would be interesting to extend our theoretical analysis to such extensions for simple examples, like the linear regression model.

\subsection{From Decompositions to Influence Measures and Back}

Using a suitable type of imputation method, and thereby definition of $F_S$, we obtain a fANOVA decomposition.
This is ensured by the recursive definition of fANOVA effects from Eq.~(\ref{eq_fanova_effect}).
In fact, for the last effect, we have
\begin{align*}
    f_D(x) = F(x) - \sum_{S \subset D} f_S(x),
\end{align*}
which yields a decomposition, as described for instance by \cite{apley_visualizing_2020,iwasawa2024interactiondecompositionpredictionfunction}.
Moreover, this recursive definition related to the \gls{MT} also ensures the orthogonality-like property \citep{apley_visualizing_2020,iwasawa2024interactiondecompositionpredictionfunction}.
Notably, partial individual, joint and interaction effects \emph{summarize} this fANOVA again into a functional decomposition \citep{bordt2022shapley} with components up to order $k$, i.e. 

\begin{align*}
    F(x) = \sum_{S\subseteq D: s \leq k} \phi_k(x,S).
\end{align*}

In this context, the explanation $\phi_k$ up to order $k$ can be viewed as a restricted fANOVA with components up to order $k$.
In future work, it would be interesting to understand how well such a low-complexity decomposition describes the model.

\clearpage
\section{ADDITIONAL BACKGROUND}\label{appx_sec_background}

In this section, we provide additional theoretical background.

\subsection{Background on Gradient-Based Attribution Methods}\label{appx_sec_background_grad}
Taylor interactions \citep{Deng2024Unify} are the components of the infinite Taylor expansion of $F$ at a baseline $b$
\begin{align}\label{appx_eq_taylor_exp}
    F(x) &= \sum_{\kappa \in \mathbb{N}^d_0}  \frac{1}{\kappa_1!\dots\kappa_d!} \cdot \frac{\partial^{\kappa_1+\dots+\kappa_d} F(b)}{\partial^{\kappa_1} x_1 \dots \partial{\kappa^d x_d}} \cdot (x_1-b_1)^{\kappa_1} \dots (x_d-b_d)^{\kappa_d}  = \sum_{\kappa \in \mathbb{N}^d_0} C(\kappa) \cdot \nabla F(\kappa) \cdot \pi(\kappa) 
\end{align}
where $\frac{\partial^k F}{\partial^k}$ are the $k$-th order partial derivative, and the Taylor interactions are defined as 
\begin{align*}
    I(\kappa) &:= C(\kappa) \cdot \nabla F(\kappa) \cdot \pi(\kappa) \text{ with}
    \\
    C(\kappa) &:= \frac{1}{\kappa_1!\dots\kappa_d!}
    \\
    \nabla F(\kappa) &:= \frac{\partial^{\kappa_1+\dots+\kappa_d} F(b)}{\partial^{\kappa_1} x_1 \dots \partial{\kappa^d x_d}}
    \\
    \pi(\kappa) &:= (x_1-b_1)^{\kappa_1} \dots (x_d-b_d)^{\kappa_d}.
\end{align*}
The sum ranges over all possible orders of partial derivatives per feature $\kappa=[\kappa_1,\dots,\kappa_d] \in \mathbb{N}_0^d$.
This set can be partitioned into the following sets
\begin{equation*}
    \{\kappa \in \mathbb{N}^d_0\} = \{[0,\dots,0]\} \bigsqcup \sqcup_{i \in D} \Omega_i \bigsqcup \sqcup_{S \subseteq D, s > 1} \Omega_S,
\end{equation*}
where $\Omega_S := \{\kappa \mid  \forall i \in S: \kappa_i>0, \forall i \notin S: \kappa_i=0\}$ is the set of degree vectors, where partial derivatives are computed exclusively for features in $S$.
Notably, for $\kappa=[0,\dots,0]$, we have $I(\kappa) = F(b)$.
Therefore, the Taylor interactions can be summarized by three sums
\begin{align*}
    F(x) &= \sum_{\kappa \in \mathbb{N}^d_0} I(\kappa) = F(b) + \sum_{i \in D} \sum_{\kappa \in \Omega_i}I(\kappa) + \sum_{S \subseteq D, s > 1} {\sum_{\kappa \in \Omega_S}I(\kappa)}
    = F(b) + \sum_{i \in D} \sum_{\kappa \in \Omega_i}\phi(\kappa) + \sum_{S \subseteq D, s > 1} {\sum_{\kappa \in \Omega_S}I(\kappa)},
\end{align*}
where we defined $\phi(\kappa) := I(\kappa)$ for $\kappa \in \Omega_i$.
$\phi$ is called the Taylor independent effect, whereas $I(\kappa)$ for $\kappa \in \Omega_S$ with $s>1$ is called the Taylor interaction effect.
The Taylor independent and interaction effects can be further summarized according to the features $i \in D$ and subsets $S \subseteq D$ with $s>1$ as
\begin{align*}
    F(x) &= F(b) + \sum_{i \in D} \sum_{\kappa \in \Omega_i}\phi(\kappa) + \sum_{S \subseteq D, s > 1} {\sum_{\kappa \in \Omega_S}I(\kappa)} = F(b) + \sum_{i \in D} \psi(i) + \sum_{S\subseteq D, s>1} J(S),
\end{align*}

where $\psi$ is the generic independent effect, and $J$ the generic interaction effect, defined by
\begin{align}\label{appx_eq_taylor_interactions}
    \psi(i) := \sum_{\kappa \in \Omega_i}\phi(\kappa) \text{ and } J(S) := {\sum_{\kappa \in \Omega_S}I(\kappa)}.
\end{align}

\subsection{Background on Cooperative Game Theory}

In this section, we briefly cover the definitions of \glspl{GV} \citep{DBLP:journals/dam/MarichalKF07} and Interaction Indices (II) \citep{Fujimoto.2006}. Both were introduced based on axioms as a weighted sum of joint marginal contributions and discrete derivatives as
\begin{align*}
    \phi^{\text{GV}}(S) := \sum_{T \subseteq D \setminus S} p_t^s \Delta_{[S]}(T) \text{ and } \phi^{\text{II}}(S) := \sum_{T \subseteq D \setminus S} p_t^s \Delta_{S}(T).
\end{align*}
The weights $p_t^s$ correspond to a discrete probability distribution and depend only on the subset sizes of $T$ and $S$.
Particular instances of IIs and \glspl{GV} are the Shapley II \citep{Grabisch.1999} and the Shapley GV \citep{DBLP:journals/dam/MarichalKF07}, which are obtained by $p_t^s = \frac{1}{d-s+1} \binom{d-s}{t}^{-1}$.
It was shown by \cite{Fujimoto.2006,DBLP:journals/dam/MarichalKF07} that \glspl{GV} and IIs can be written via the \gls{MT} as
\begin{align*}
     \phi^{\text{GV}}(S) = \sum_{T \subseteq D : T \cap S \neq \emptyset} w_{\vert T \setminus S\vert}^s m(T) \text{ and } \phi^{\text{II}}(S) = \sum_{T \subseteq D: T \supseteq S} w_{\vert T \setminus S\vert}^s m(T).
\end{align*}

The corresponding conversion formulas from $p$ to $w$ can be found in \cite{Fujimoto.2006,DBLP:journals/dam/MarichalKF07}.
For example, the Shapley GV and Shapley II are obtained with $w^s_\ell = \frac{1}{\ell+1}$.

For machine learning applications, particularly for local explanations, these indices are not optimal, since they do not sum to the overall payout (model prediction), referred to as \emph{efficiency} \citep{tsai_faith-shap_2022}.
As a remedy, \glspl{SI} $\phi^{\text{SI}}_k$ are a sub-family of IIs, which are defined for subsets up to an explanation order $k$, i.e. for subsets $S \subseteq D$ with $1 \leq s \leq k$.
It is required that \glspl{SI} are efficient, i.e. the sum of all values yields the overall payout, which is particularly interesting for local explanations, as it requires that the sum of all explanations yields the model's prediction.
Given an explanation order $k$, we define
\begin{align*}
    \textbf{generalized efficiency: } \nu(D) = \sum_{S\subseteq D: s\leq k} \phi_k^{\text{SI}}(S),
\end{align*}
where typically $\phi_k^{\text{SI}}(\emptyset) = \nu(\emptyset)$.
\Glspl{SI} were built on the SII \citep{Grabisch.1999}, known as $k$-\glspl{SV} ($k$-SII) \citep{bordt2022shapley}.
Moreover, the Shapley Taylor interaction index \citep{sundararajan2020shapley} and the faithful Shapley interaction index \citep{tsai_faith-shap_2022} have been proposed.
These \glspl{SI} differ by their choice of axioms, and it was shown that they differently summarize higher-order interactions \citep{bordt2022shapley}.
Moreover, the \glspl{SI} for $k=1$ yield the \glspl{SV} and for $k=d$ the \gls{MT}.
The \glspl{SI} thus have a complexity-accuracy trade-off with $k$.
For the exact definition of weights, we refer to each paper.

The Joint\glspl{SV} \citep{Harris.2022} are similar to \glspl{SI}, as they are also based on an explanation order $k$ and yield an efficient index.
However, they are based on \glspl{GV}, and we refer to the definition of weights to the corresponding paper.

\clearpage
\section{ADDITIONAL DETAILS ON ILLUSTRATIVE EXAMPLES}

In this section, we provide more details and derivations of the illustrative examples in Section 3.

\subsection{Details on Illustration of fANOVA Decompositions}\label{appx_sec_fanova_examples}

In Section 3.1, we demonstrate the difference of b-, m- and c-fANOVA on a purely linear function $F_{\text{lin}}= \beta_1 x_1 +  \beta_2 x_2$ containing only the main effects of the two multivariate normally distributed features $X \sim \mathcal{N}(\mu, \Sigma)$ with mean vector $\mu \in \mathbb{R}^2$ and covariance matrix $\Sigma \in \mathbb{R}^{2 \times 2}$, and a second function additionally containing the linear interaction between the two features: $F_{2\text{int}} = F_{\text{lin}} + \beta_{12} x_1x_2$. 
We now derive the main and two-way interaction fANOVA effects for each of the two functional relationships and each of the three distributional assumptions of the underlying fANOVA decomposition, which are summarized in Table 1 in the main paper.

\subsubsection{Baseline fANOVA (b-fANOVA)}

Here, we define the functional components for the two functional relationships $F_{\text{lin}}$ and $F_{2\text{int}}$ based on the b-fANOVA.

\paragraph{Linear Main Effect Model.}

\begin{align*}
    f^{(b)}_\emptyset &= F_{\text{lin}}(b)\\ &= \beta_1 b_1 + \beta_2 b_2 \\
    f^{(b)}_i (x) &= F_{\text{lin}}(x_i, b_j) - f^{(b)}_\emptyset\\ &= \beta_i x_i + \beta_j b_j - (\beta_i b_i + \beta_j b_j )\\  &= \beta_i (x_i - b_i) \\
    f^{(b)}_{12} (x) &= F_{\text{lin}}(x_1, x_2) - f^{(b)}_1(x) - f^{(b)}_2(x) - f^{(b)}_\emptyset\\ &= \beta_1 x_1 + \beta_2 x_2 - \beta_1 (x_1 - b_1)  - \beta_2 (x_2 - b_2) - (\beta_1 b_1 + \beta_2 b_2 ) \\ &= 0 
\end{align*}
It follows that $F_{\text{lin}}(x) = f^{(b)}_\emptyset + f^{(b)}_1 (x) + f^{(b)}_2 (x)$, hence, the function can be decomposed into a constant plus the centered first-order (main) effects of the two features. The centering is based on the baseline value $b$ for b-fANOVA.

\paragraph{Linear Model with a Two-Way Interaction.}

\begin{align*}
    f^{(b)}_\emptyset &= F_{2\text{int}}(b)\\ &= \beta_1 b_1 + \beta_2 b_2 + \beta_{12} b_1 b_2\\
    f^{(b)}_i (x) &= F_{2\text{int}}(x_i, b_j) - f^{(b)}_\emptyset\\ 
    &= \beta_i x_i + \beta_j b_j + \beta_{ij} x_i b_j - (\beta_i b_i + \beta_j b_j +\beta_{ij} b_i b_j)\\  
    &= \beta_i (x_i - b_i) + \beta_{ij} b_j (x_i - b_i)\\
    f^{(b)}_{12} (x) 
    &= F_{2\text{int}}(x_1, x_2) - f^{(b)}_1(x) - f^{(b)}_2(x) - f^{(b)}_\emptyset\\ 
    &= \beta_1 x_1 + \beta_2 x_2  + \beta_{12} x_1 x_2 - (\beta_1 (x_1 - b_1) + \beta_{12} b_2 (x_1 - b_1))\\ 
    &- (\beta_2 (x_2 - b_2)  + \beta_{12} b_1 (x_2 - b_2)) - f^{(b)}_\emptyset\\ 
    &= \beta_{12}(x_1 - b_1) (x_2 - b_2) 
\end{align*}
It follows that $F_{\text{int}}(x) = f^{(b)}_\emptyset + f^{(b)}_1 (x) + f^{(b)}_2 (x) + f^{(b)}_{12} (x)$, hence, the function can be decomposed into a constant, the centered first-order (main) effects of the two features, and their joint centered interaction.

\subsubsection{Marginal fANOVA (m-fANOVA)}

Here, we define the functional components for the two functional relationships $F_{\text{lin}}$ and $F_{2\text{int}}$ based on the m-fANOVA.

\paragraph{Linear Main Effect Model.}

\begin{align*}
    f^{(m)}_\emptyset &= \mathbb{E}[F_{\text{lin}}(X)]\\ &= \beta_1 \mu_1 + \beta_2 \mu_2 \\
    f^{(m)}_i (x) &= \mathbb{E}[F_{\text{lin}}(x_i, X_j)] - f^{(m)}_\emptyset\\ &= \beta_i x_i + \beta_j \mu_j - (\beta_i \mu_i + \beta_j \mu_j )\\  &= \beta_i (x_i - \mu_i) \\
    f^{(m)}_{12} (x) &= F_{\text{lin}}(x_1, x_2) - f^{(m)}_1(x) - f^{(m)}_2(x) - f^{(m)}_\emptyset\\ &= \beta_1 x_1 + \beta_2 x_2 - \beta_1 (x_1 - \mu_1)  - \beta_2 (x_2 - \mu_2) - (\beta_1 \mu_1 + \beta_2 \mu_2 ) )\\ &= 0
\end{align*}
It follows that $F_{\text{lin}}(x) = f^{(m)}_\emptyset + f^{(m)}_1 (x) + f^{(m)}_2 (x)$, hence, the function can be decomposed into a constant plus the mean-centered first-order (main) effects of the two features. As shown in Theorem 2, b-fANOVA corresponds to m-fANOVA for this linear functional relationship if $b:= \mu$. 

\paragraph{Linear Model with a Two-Way Interaction.}

\begin{align*}
    f^{(m)}_\emptyset &= \mathbb{E}[F_{2\text{int}}(X)]\\ 
    &= \beta_1 \mu_1 + \beta_2 \mu_2 + \beta_{12} (\sigma_{12} + \mu_1 \mu_2)\\
    f^{(m)}_i (x) 
    &= \mathbb{E}[F_{2\text{int}}(x_i, X_j)] - f^{(m)}_\emptyset\\ 
    &= \beta_i x_i + \beta_j \mu_j + \beta_{ij} x_i \mu_j - (\beta_i \mu_i + \beta_j \mu_j +\beta_{ij} (\sigma_{ij} + \mu_i \mu_j))\\  
    &= \beta_i (x_i - \mu_i) + \beta_{ij} \mu_j (x_i - \mu_i) - \beta_{ij} \sigma_{ij}\\
    f^{(m)}_{12} (x) &= F_{2\text{int}}(x_1, x_2) - f^{(m)}_1(x) - f^{(m)}_2(x) - f^{(m)}_\emptyset\\ 
    &= \beta_1 x_1 + \beta_2 x_2  + \beta_{12} x_1 x_2 - (\beta_1 (x_1 - \mu_1) + \beta_{12} \mu_2 (x_1 - \mu_1) - \beta_{12} \sigma_{12})\\ 
    &- (\beta_2 (x_2 - \mu_2)  + \beta_{12} \mu_1 (x_2 - \mu_2) - \beta_{12} \sigma_{12}) - f^{(m)}_\emptyset\\ 
    &= \beta_{12}(x_1 - \mu_1) (x_2 - \mu_2) + \beta_{12}\sigma_{12}
\end{align*}
with $\mathbb{E}[X_1 X_2] = cov(X_1,X_2) + \mathbb{E}[X_1]\mathbb{E}[X_2] = \sigma_{12} + \mu_1 \mu_2$.
The functional relationship can be decomposed by $F_{\text{int}}(x) = f^{(m)}_\emptyset + f^{(m)}_1 (x) + f^{(m)}_2 (x) + f^{(m)}_{12} (x)$ using m-fANOVA. Hence, the function can be decomposed into a constant, the mean-centered first-order (main) effects of the two features, and their joint mean-centered interaction.
Since m-fANOVA integrates over the joint marginal distribution, correlations of features in $-S$ can influence the functional components by an additive shift (here: $\beta_{12} \sigma_{12}$). However, dependencies between $S$ and $-S$ in m-fANOVA are broken, and thus are not influencing the definition of the functional components.

\subsubsection{Conditional fANOVA (c-fANOVA)}

Here, we define the functional components for the two functional relationships $F_{\text{lin}}$ and $F_{2\text{int}}$ based on the c-fANOVA. 

\paragraph{Linear Main Effect Model.}

\begin{align*}
    f^{(c)}_\emptyset &= \mathbb{E}[F_{\text{lin}}(X)]\\ &= \beta_1 \mu_1 + \beta_2 \mu_2 \\
    f^{(c)}_i (x) &= \mathbb{E} [F_{\text{lin}}(X) \vert X_i = x_i] - f^{(c)}_\emptyset\\ &= \beta_i x_i + \beta_j \left(\mu_j + \frac{\sigma_{ij}}{\mathbb{V}[X_i]} (x_i - \mu_i)\right) - (\beta_i \mu_i + \beta_j \mu_j )\\  
    &= \beta_i (x_i - \mu_i) + (x_i - \mu_i) \frac{\beta_j\sigma_{ij}}{\mathbb{V}[X_i]} \\
    f^{(c)}_{12} (x) &= F_{\text{lin}}(x_1, x_2) - f^{(c)}_1(x) - f^{(c)}_2(x) - f^{(c)}_\emptyset\\ &= \beta_1 x_1 + \beta_2 x_2 - \beta_1 (x_1 - \mu_1) - (x_1 - \mu_1) \frac{\beta_2\sigma_{12}}{\mathbb{V}[X_1]}\\  
    &- \beta_2 (x_2 - \mu_2) - (x_2 - \mu_2) \frac{\beta_1\sigma_{12}}{\mathbb{V}[X_2]} - (\beta_1 \mu_1 + \beta_2 \mu_2 ) \\ 
    &= -\sum_{\ell \in \{1,2\}} (x_\ell - \mu_\ell) \frac{\beta_{-\ell}\sigma_{-\ell, \ell}}{\mathbb{V}[X_\ell]}
\end{align*}
with $\mathbb{E}[X \vert X_i = x_i] = \mu_j + \frac{\sigma_{ij}}{\mathbb{V}[X_i]} (x_i - \mu_i)$.
It follows that $F_{\text{lin}}(x) = f^{(c)}_\emptyset + f^{(c)}_1 (x) + f^{(c)}_2 (x) + f^{(c)}_{12} (x)$, where the second-order term adjusts for the cross-correlation effects added in the first-order terms to account for the influence of the feature distribution.

\paragraph{Linear Model with a Two-Way Interaction.}

\begin{align*}
    f^{(c)}_\emptyset &= \mathbb{E}[F_{2\text{int}}(X)]\\ &= \beta_1 \mu_1 + \beta_2 \mu_2 + \beta_{12} (\sigma_{12} + \mu_1 \mu_2)\\
    f^{(c)}_i (x) &= \mathbb{E}[F_{2\text{int}}(X) \vert X_i = x_i] - f^{(c)}_\emptyset\\ 
    &= \beta_i x_i + \beta_j \left(\mu_j + \frac{\sigma_{ij}}{\mathbb{V}[X_i]} (x_i - \mu_i)\right) + \beta_{ij} x_i \left(\mu_j + \frac{\sigma_{ij}}{\mathbb{V}[X_i]} (x_i - \mu_i)\right)\\ 
    &- (\beta_i \mu_i + \beta_j \mu_j  +\beta_{ij} (\sigma_{ij} + \mu_i \mu_j)\\  
    &= \beta_i (x_i - \mu_i) + (x_i - \mu_i) \frac{\beta_j\sigma_{ij}}{\mathbb{V}[X_i]} + \beta_{ij} (x_i - \mu_i) \left(\mu_j + \frac{\sigma_{ij}x_i}{\mathbb{V}[X_i]}\right) - \beta_{ij} \sigma_{ij}\\
    f^{(c)}_{12} (x) &= F_{2\text{int}}(x_1, x_2) - f^{(c)}_1(x) - f^{(c)}_2(x) - f^{(c)}_\emptyset\\ 
    &= \beta_{ij}(x_1 - \mu_1) (x_2 - \mu_2) + \beta_{12}\sigma_{12}\\ 
    &-\sum_{\ell \in \{1,2\}} (x_\ell - \mu_\ell) \frac{\beta_{-\ell}\sigma_{-\ell, \ell}}{\mathbb{V}[X_\ell]} - \beta_{12} \sigma_{12} \left(\frac{(x_1-\mu_1)x_1}{\mathbb{V}[X_1]} + \frac{(x_2-\mu_2)x_2}{\mathbb{V}[X_2]}\right) \\ 
    &= f_{2\text{int}}^{(m)} + f_{\text{lin}}^{(c)}  - \beta_{12} \sigma_{12} \left(\frac{(x_1-\mu_1)x_1}{\mathbb{V}[X_1]} + \frac{(x_2-\mu_2)x_2}{\mathbb{V}[X_2]}\right)\\
\end{align*}
with $\mathbb{E}[X_1 X_2] = cov(X_1,X_2) + \mathbb{E}[X_1]\mathbb{E}[X_2] = \sigma_{12} + \mu_1 \mu_2$ and $\mathbb{E}[X \vert X_i = x_i] = \mu_j + \frac{\sigma_{ij}}{\mathbb{V}[X_i]} (x_i - \mu_i)$.
The functional relationship can be decomposed by $F_{\text{int}}(x) = f^{(c)}_\emptyset + f^{(c)}_1 (x) + f^{(c)}_2 (x) + f^{(c)}_{12} (x)$ using c-fANOVA. 
Compared to m-fANOVA, c-fANOVA integrates over the conditional distribution. Hence, features in $-S$ may influence feature effects of $S$ through cross-correlation effects if features are dependent on each other.

\subsection{Details on Illustration of Feature Influence Methods}\label{sec_framework_example_influence}

In this section, we provide more detailed definitions of the feature influence methods of the illustrative example of Section 3.4. 
As defined in Section 3.4, we assume that we are given three independent and standard normally distributed features $X_1, X_2, X_3 \sim \mathcal{N}(0,1)$ and the following functional relationship: $F(x) = x_1 + x_2 + x_3 + x_1x_2 + x_1x_2x_3$. In this case the m-fANOVA is equivalent to the c-fANOVA and the b-fANOVA with $b = \mathbb{E}[X]$.
We now define the feature influence measures based on the fANOVA decomposition. The respective weights of each effect are summarized in Table 3 in the main paper.

\paragraph{Individual Effects for $i = 1$.}

\begin{align*}
    \phi_x^\emptyset(i) &= f_1(x) = x_1\\
    \phi_x^{SV}(i) &= f_1(x) + \frac{1}{2} f_{12}(x) + \frac{1}{3} f_{123}(x) = x_1 + \frac{1}{2} x_1 x_2 + \frac{1}{3} x_1 x_2 x_3\\
    \phi_x^+(i) &= f_1(x) + f_{12}(x) + f_{123}(x) = x_1 + x_1 x_2 + x_1 x_2 x_3 \\
\end{align*}

\paragraph{Joint Effects for $ij = \{1,2\}$.}

\begin{align*}
    \phi^{\emptyset}(ij) &= f_1(x) + f_2(x) + f_{12}(x) = x_1 + x_2 + x_1 x_2\\ 
    \phi^{GV}(ij) &= f_1(x) + f_2(x) + f_{12}(x)  + w_1^2 f_{123}(x) =x_1 + x_2 + x_1 x_2 + w_1^2 x_1x_2x_3\\
    \phi^{+}(ij) &= f_1(x) + f_2(x) + f_{12}(x)  + f_{123}(x) = x_1 + x_2 + x_1 x_2 + x_1 x_2 x_3 \\
\end{align*}

\paragraph{Interaction Effects for $ij = \{1,2\}$.}

    \begin{align*}
        \phi^{I\emptyset}(ij) &= f_{12}(x) = x_1 x_2\\
        \phi^{SI}(ij) &= f_{12}(x) + w_1^2 f_{123}(x) = x_1 x_2 + w_1^2 x_1 x_2 x_3\\
        \phi^{I+}(ij) &= f_{12}(x) + f_{123}(x) = x_1 x_2 + x_1 x_2 x_3\\
    \end{align*}
where $\mathbb{E}[F(X)] = f_\emptyset = 0$ and $\mathbb{E}[F(x_S, X_{-S})] = \mathbb{E}[F(X)\vert X_S = x_S]$ and $\mathbb{E}[X_i X_j] = \mathbb{E}[X_i] \mathbb{E}[X_j]$ holds due to the above stated distributional assumptions.

Note that for a linear functional relationship without interactions and uncorrelated features, the individual influence methods result in the same contribution defined by the main effect:
$\phi^\emptyset(i) =
        \phi^{SV}(i) = 
        \phi^+(i) =  
        \beta_i \cdot (x_i - \mu_i)$, 
while the interaction influence methods become zero:
$\phi^{I\emptyset}(S) = \phi^{SI}(S) = \phi^{I+}(S)   = 0$.

Similarly to individual influence methods, the pure, partial, and full effects result in the same contribution defined by the sum of the main effects: $\phi^{\emptyset}(S)=\phi^{GV}(S)=\phi^{+}(S) = \sum_{i \in S}\phi^\emptyset(i) = \sum_{i \in S} \beta_i \cdot (x_i - \mu_i)$.

\clearpage
\section{DETAILS ON UNIFYING FEATURE-BASED EXPLANATIONS}
\label{appx_section_table}

In this section, we present a more comprehensive categorization of existing methods using our framework.
We categorize all methods in \cref{appx_tab_unification}.
We further describe and explain our categorization for local explanations (\cref{appx_sec_unif_local}), global sensitivity analysis (\cref{appx_sec_unif_sens}), and global risk-based explanations (\cref{appx_sec_unif_risk}).

\begin{table}[htb]
  \caption{Categorization of XAI methods for {\color{blue} local}, {\color{violet} sensitivity}, and {\color{teal} risk} explanation game, where $^*$ is based on gradients}
    \begin{tabular}{m{2cm}|m{2cm}|m{2cm}|m{8cm}}
    \toprule
      Distribution coverage & Higher-order Interaction Influence & Feature Influence& Explanations  \\
      \midrule
       \multirow{9}{*}{b-fANOVA}  & \multirow{3}{*}{pure (no)} & individual & {\color{blue} ArchAttribute}, {\color{blue} Gradient$\times$Input$^*$ (linear $F$)}, {\color{blue} LRP-$\epsilon^*$}\\
         &  & joint & {\color{blue} ArchAttribute} \\
         &  & interaction &  \\\cmidrule{2-4}
         & \multirow{3}{*}{partial} & individual & {\color{blue} BSHAP}, {\color{blue} Integrated Gradients}, {\color{blue} DeepSHAP$^*$}, {\color{blue} DeepLIFT$^*$} \\
         &  & joint  &  {\color{blue} Joint SVs} \\
         &  & interaction &  {\color{blue}Shapley Taylor Interaction},  {\color{blue}Faith-SHAP}, {\color{blue} Integrated Hessians$^*$}\\\cmidrule{2-4}
         & \multirow{3}{*}{full} & individual &  {\color{blue} ICE}, {\color{blue} Occlusion-1}, {\color{blue} Gradient$\times$Input$^*$ (more than full)}\\
         &  & joint  & {\color{blue} ICE}, {\color{blue}Occlusion-patch} \\
         &  & interaction &  \\
       \midrule
       \multirow{9}{*}{m-fANOVA}  & \multirow{3}{*}{pure (no)} & individual &  {\color{blue} PDP}, {\color{violet} Closed Sobol' Index ($\underline{\tau}$)}\\
         &  & joint &  {\color{blue} PDP},   {\color{violet} Closed Sobol' Index ($\underline{\tau}$)}\\
         &  & interaction &  {\color{blue} Int. Shapley-GAM}, {\color{violet} H-statistic}\\\cmidrule{2-4}
         & \multirow{3}{*}{partial} & individual & {\color{blue} Int. SHAP}, {\color{blue} Expected Gradient$^*$},  {\color{violet} Sobol'-SV}, {\color{teal} SAGE (in practice)}\\
         &  & joint  &  \\
         &  & interaction &  {\color{blue} Int. $k$-SVs}\\\cmidrule{2-4}
        & \multirow{3}{*}{full} & individual &  {\color{violet} Total Sobol' Index ($\bar\tau$)}, {\color{violet} Superset measure ($\Upsilon^2$}), {\color{teal} PFI}\\
         &  & joint &  {\color{violet} Total Sobol' Index ($\bar\tau$)}, {\color{teal} Grouped PFI}\\
         &  & interaction & {\color{violet} Superset measure ($\Upsilon^2$)}\\
       \midrule
       \multirow{9}{*}{c-fANOVA}  & \multirow{3}{*}{pure (no)} & individual &  {\color{blue} M Plot}, {\color{blue} ALE$^*$}\\
         &  & joint &  {\color{blue} M Plot} \\
         &  & interaction & {\color{blue} Obs. Shapley-GAM}, {\color{blue} ALE}$^*$ 
         \\\cmidrule{2-4}
         & \multirow{3}{*}{partial} & individual & {\color{blue} Obs. SHAP}, {\color{blue} TreeSHAP}, {\color{violet} Sobol'-SV}, {\color{teal} SAGE}\\
         &  & joint &   \\
         &  & interaction &  {\color{blue} Obs. $k$-SVs}\\\cmidrule{2-4}
         & \multirow{3}{*}{full} & individual & {\color{teal} CFI}, {\color{blue} PredDiff-1}\\
         &  & joint & {\color{blue} PredDiff-Patch} \\
         &  & interaction & {\color{blue}PredDiff-Interaction}\\
       \bottomrule
    \end{tabular}
    \label{appx_tab_unification}
\end{table}

\subsection{Unifying Local Feature-Based Explanations}\label{appx_sec_unif_local}
Here, we summarize local feature-based explanations, i.e. explanations based on the local explanation game.

\subsubsection{Perturbation-Based Explanations}
First, we summarize perturbation-based explanations.

\paragraph{Occlusion-1 and Occlusion-Patch.}
Occlusion-1 and Occlusion-Patch \citep{Zeiler_Fergus_2014} were introduced for image data.
By masking the relevant feature (or feature patch) the change in prediction is measured.
More formally, in the general case, Occlusion-Patch computes for a patch of features $S \subseteq D$
\begin{align*}
    \phi^{\text{Occlusion-Patch}}(S) := F(x) - F^{(b)}_{-S}(x) = \nu_x^{(\text{loc})}(D)-\nu_x^{(\text{loc})}(-S) = \phi^+(S),
\end{align*}
using a baseline $b$.
Hence, Occlusion-Patch corresponds to the full joint effect $\phi^+$, and to the full individual effect for Occlusion-1 using b-fANOVA.

\paragraph{PredDiff-1, PredDiff-Patch and PredDiff-Interaction}
PredDiff \citep{Robnik2008} is a popular method for image classification tasks \citep{DBLP:conf/iclr/ZintgrafCAW17}.
It measures similarly to Occlusion the change in prediction when masking features.
However, instead of using a baseline, PredDiff uses the conditional distribution, i.e., for a patch $S \subseteq D$
\begin{align*}
    \phi^{\text{PredDiff-Patch}}(S) := F(x) - \mathbb{E}[F(X) \mid X_{-S}=x_{-S}] = F(x) - F_{-S}^{(c)}(x) = \nu_x^{(\text{loc})}(D)-\nu_x^{(\text{loc})}(-S) = \phi^+(S).
\end{align*}
Hence, PredDiff-Patch measures the full joint effect, and PredDiff-1 the full individual effect of c-fANOVA.
PredDiff-Interaction \citep{Blucher.2022} measures the interaction by extracting individual effects, thereby measuring full interaction effects of c-fANOVA. 

\paragraph{Individual Conditional Expectations (ICE).}
ICE curves \citep{goldstein_peeking_2015} visualize the influence of a feature of interest on the prediction in the presence of all other features. The ICE function itself is defined for individual or groups of features at some baseline value $x_S = b_S$ which is typically chosen from $\mathcal{X}_S$, while the remaining features are set to the observed feature values of the instance to be explained:
\begin{equation*}
    \phi^{\text{ICE}}(S) := F(b_S, x_{-S})
\end{equation*}
Centering ICE curves to remove additive shifts is typically done to understand and quantify the influence of feature interactions \citep{herbinger2022repid}. One possibility is to center it according to the actual prediction of the instance, which leads to the full individual ($S = i$) and joint influence measure $\phi^+(S)$ using b-fANOVA. It is defined by
\begin{equation*}
    \phi^{\text{c-ICE}}(S) := - (F(b_S, x_{-S}) - F(x) ) = f_S^{(b)}(x) + \sum_{T \subseteq D: T \cap S \notin \{S,\emptyset\}} f_T^{(b)}(x).
\end{equation*}

\paragraph{Shapley Additive Explanations (SHAP).}
Baseline SHAP (BSHAP) \citep{DBLP:journals/jmlr/StrumbeljK10,DBLP:conf/icml/SundararajanN20} computes  the \gls{SV} on the set function $\nu^{(\text{loc})}(S) = F_S^{(b)}(x_0)$, i.e., the local explanation game.
Therefore, BSHAP corresponds to the partial individual effect of $b$-fANOVA.
On the contrary, interventional (int.) SHAP \citep{lundberg_2017_unified, DBLP:conf/aistats/JanzingMB20} computes the \gls{SV} on the marginal expectations $\nu^{(\text{loc})}(S) = F_S^{(m)}(x_0)$, i.e., the partial individual effects of m-fANOVA.
Lastly, observational (obs.) SHAP \citep{strumbelj_2014,lundberg_2017_unified, aas_explaining_2021} computes the \gls{SV} on the conditional expectations $\nu^{(\text{loc})}(S) = F_S^{(c)}(x_0)$, i.e., the partial individual effect of c-fANOVA.

\paragraph{TreeSHAP.}
TreeSHAP \citep{lundberg_consistent_2019} is a model-specific method for tree-based models.
It computes the exact \glspl{SV} on the tree-induced conditional distribution, i.e., the conditional distribution approximated by the learned tree structure.
Hence, TreeSHAP corresponds to the partial indivdual effect of c-fANOVA.
Notably, conditional expectations approximated by trees might have significant drawbacks in practice \citep{aas_explaining_2021}.

\paragraph{$k$-Shapley Values ($k$-SVs) and the Shapley-GAM.}
$k$-\glspl{SV} \citep{bordt2022shapley} are an extension of pairwise Shapley interactions \citep{lundberg_consistent_2019} introduced on trees.
The $k$-\glspl{SV} compute the Shapley interaction index \citep{Grabisch.1999} for the highest order, and adjust lower-order interactions recursively.
$k$-\glspl{SV} can be applied with baseline, marginal (interventional), and conditional (observational) imputations, thereby measuring partial interaction effects of b-, m- and c-fANOVA, respectively.
The Shapley-GAM \citep{bordt2022shapley} is the edge case of $k$-\glspl{SV} for $k=d$.
They provide the pure interaction effects for b-,m-, and c-fANOVA, depending on which imputation method is used.

\paragraph{Shapley Taylor Interaction Index.}
The Shapley Taylor interaction index \citep{sundararajan2020shapley} is an extension of the \gls{SV} that measures partial interaction effects up to explanation order $k$.
The highest-order interactions (order $=k$) thereby capture all higher-order interactions (order $>k$).
In general, the Shapley-Taylor interactions can be applied with any set function, but were introduced on b-fANOVA.

\paragraph{Faithful Shapley Interaction Index (Faith-SHAP).}
The faithful Shapley interaction index (Faith-SHAP) \citep{tsai_faith-shap_2022} is an extension of \glspl{SV}, and specifies partial interactions up to explanation order $k$.
In contrast to Shapley-Taylor interactions, it distributes higher-order interactions on all interactions that are computed. 
It thereby corresponds to a faithful approximation of the game given the Shapley-weighted regression problem \citep{tsai_faith-shap_2022}.
Faith-SHAP can be applied to any set function, but was presented using a baseline, and therefore corresponds to b-fANOVA.

\paragraph{Joint Shapley Values (JointSV).}
Joint\glspl{SV} \citep{Harris.2022} were introduced as an extension to \glspl{SV} for interactions up to explanation order $k$.
Joint\glspl{SV} are based on joint marginal contributions and thereby measure a partial joint effect.
While Joint\glspl{SV} can be applied to any set function, they were introduced with baseline imputation, thereby corresponding to b-fANOVA.

\paragraph{ArchAttribute.}
ArchAttribute \citep{tsang2020does} is a feature attribution method that quantifies the pure effect of an individual feature or a group of features compared to a baseline context and is defined by:
\begin{equation*}
    \phi^{\text{ArchAttr}}(S) := F(x_S, b_{-S}) - F(b).
\end{equation*}
This definition directly corresponds to the individual pure effect $\phi_x^\emptyset(i) = f_i^{(b)}(x)$ for $S = i$ and the joint pure effect $\phi_x^\emptyset(S) = \sum_{\emptyset \neq L \subseteq S} f_L^{(b)}(x)$ for $\vert S \vert > 1$ using b-fANOVA.

\paragraph{Partial Dependence Plot (PDP).}
The PDP is a global feature effect method that visualizes the expected marginal effect of an individual or a group of features on the predicted outcome \citep{friedman_greedy_2001}. The underlying partial dependence (PD) function is defined at a data point $x_S \in \mathcal{X}_S$ by 
\begin{equation*}
    \phi^{\text{PD}}(S) := \mathbb{E}[F(x_S, X_{-S})].
\end{equation*}
Similarly to ICE curves, PDPs are centered to remove additive shifts \citep{herbinger2022repid}. Centering the PD function with respect to mean prediction leads to the pure individual ($S = i$) and joint influence $\phi^\emptyset(S)$ using the m-fANOVA. The centered PD function is defined by
\begin{equation*}
    \phi^{\text{c-PD}}(S) := \mathbb{E}[F(x_S, X_{-S})] - \mathbb{E}[F(X)] = \sum_{\emptyset \neq L \subseteq S} f_L^{(m)} (x).
\end{equation*}

\paragraph{Marginal (M) Plot.}
The M Plot is the conditional variant of the PDP, i.e., a global feature effect method that visualizes the conditional expected marginal effect of an individual or a group of features on the predicted outcome \citep{friedman_greedy_2001, apley_visualizing_2020}. The underlying marginal (M) function is defined at a data point $x_S \in \mathcal{X}_S$ by 
\begin{equation*}
    \phi^{\text{M}}(S) := \mathbb{E}[F(X) \vert X_{S} = x_S].
\end{equation*}
Such as for PDPs, also M Plots can be centered with respect to mean prediction which results in the pure individual ($S = i$) and joint influence $\phi^\emptyset(S)$ using the c-fANOVA. The centered M function is defined by
\begin{equation*}
    \phi^{\text{c-M}}(S) := \mathbb{E}[F(X) \vert X_{S} = x_S] - \mathbb{E}[F(X)] = \sum_{\emptyset \neq L \subseteq S} f_L^{(c)} (x).
\end{equation*}

\subsubsection{Gradient-Based Attribution Methods}

In this section, we briefly introduce gradient-based attribution methods and categorize them according to our framework.
There exist two categories of gradient-based methods: Methods either rely directly on the output, or compute attributions layer by layer and back-propagate these.
It was shown \citep{DBLP:conf/iclr/AnconaCO018} that the layer-wise attribution computation is not necessary.
In the following, we analyze layer-wise attribution methods for each layer independently, which follows \cite{Deng2024Unify}.

\paragraph{Grad$\times$Input.}
Grad$\times$Input is a fundamental method to attribute importance to individual features \citep{Shrikumar_Greenside_Kundaje_2017}.
The method simply computes
\begin{equation*}
    \phi^{\text{Grad}\times\text{Input}}(i) := \frac{\partial{F(x)}}{\partial x_i} x_i.
\end{equation*}

This method typically assumes centered features, i.e. $\mathbb{E}[X] = 0$, and we generalize it with the baseline $b := \mathbb{E}[X]$ as
\begin{equation*}
    \phi^{\text{Grad}\times\text{Input}}(i) := \frac{\partial{F(x)}}{\partial x_i} (x_i-b_i).
\end{equation*}
As discussed in \cref{sec_unified_framework}, the gradient describes the sum of all fANOVA effects to
\begin{align*}
    \frac{\partial F}{\partial x_i} = \sum_{S\subseteq D: i \in S} \frac{\partial f^{(b)}_S(x)}{\partial x_i}.
\end{align*}
By using the infinite Taylor expansion around $b$ from Eq.~(\ref{appx_eq_taylor_exp}), we find for b-fANOVA effects
\begin{align*}
    (x_i-b_i) \cdot \frac{\partial f^{(b)}_S(x)}{\partial x_i}
    &= (x_i-b_i)\cdot \sum_{\kappa \in \Omega_S} C(\kappa) \cdot \nabla F(\kappa) \cdot \frac{\partial}{\partial x_i} \pi(\kappa)
    = (x_i-b_i)\cdot \sum_{\kappa \in \Omega_S} C(\kappa) \cdot \nabla F(\kappa) \cdot \frac{\partial}{\partial x_i} \prod_{j\in D}(x_j-b_j)^{\kappa_j} 
    \\
    &=\sum_{\kappa \in \Omega_S} C(\kappa) \cdot \nabla F(\kappa) \cdot \prod_{j\in D}(x_j-b_j)^{\kappa_j} \cdot \kappa_i =\sum_{\kappa \in \Omega_S} \underbrace{C(\kappa) \cdot \nabla F(\kappa) \cdot \pi(\kappa)}_{\phi(\kappa) \text{ for } S=i, \text{ else } I(\kappa)} \cdot \kappa_i.
\end{align*}

Hence, Grad$\times$Input computes a variant of the full effect, where every interaction effect $f^{(b)}_S$ involving $i \in S$ is more than fully accounted in the explanation as
\begin{align*}
     \phi^{\text{Grad}\times\text{Input}}(i) = \frac{\partial{F(x)}}{\partial x_i} (x_i-b_i) = \sum_{\kappa \in \Omega_i} \phi(\kappa) \cdot \kappa_i + \sum_{S\subseteq D: i \in S, s>1} \sum_{\kappa \in \Omega_S} I(\kappa) \cdot \kappa_i  
\end{align*}

In fact, the Taylor interaction effect $I(\kappa$) contained in the b-fANOVA effect $f^{(b)}_S$ is weighted by the order $\kappa_i \geq 1$ of $x_i$, which yields an even stronger influence of higher-order functions. 
If $x_i$ appears only linearly in each interaction, then it corresponds directly to the full effect.

Moreover, when assuming a linear function, as shown by \citet{Deng2024Unify}, this attribution becomes equal to $\phi(\kappa)$ with $\kappa=[0,\dots,\kappa_i=1,\dots,0]$ and the Taylor independent effect $\phi$.
In this case, by Eq.~(\ref{appx_eq_taylor_interactions}) $\phi(i)$ is part of the generic independent effect $\psi(i)$, since $\phi(\kappa) \in \Omega_i$.
Hence, by Corollary~\ref{cor_fanova_moebius}, $\phi(i)$ is part of the pure b-fANOVA effect $f^{(b)}_i$.
If we assume linearity of $F$ in $x_i$ around $x_0$, i.e., vanishing higher-order derivatives, then $\phi(i)=\psi(i)=f^{(b)}_i$.

\paragraph{LRP-$\epsilon$.}
Layer-wise relevance propagation (LRP) is a method proposed by \cite{bach2015pixel}, and similar  to Grad$\times$Input.
In fact, for ReLU activation functions, both methods coincide \citep{DBLP:conf/iclr/AnconaCO018}.
In general, LRP is applied layer-wise, by back-propagating attributions layer by layer, and in each layer, the method behaves similarly to Grad$\times$Input \citep{Deng2024Unify}.
Thus, it measures in each layer the pure effect with the additional linearity assumption, similar to Grad$\times$Input.

\paragraph{DeepSHAP.}
DeepSHAP \citep{lundberg_2017_unified} computes attributions layer-wise, where the \gls{SV} is used.
It was shown by \cite{Deng2024Unify}, that for each layer, the layer-wise attribution is given by
\begin{align*}
    \phi(i) = \sum_{\kappa \in \Omega_i}\phi(\kappa) + \sum_{S \supset i: s>1} \frac{1}{s} \sum_{\kappa \in \Omega_S} I(\kappa) = \psi(i) + \sum_{S \supset i: s>1} \frac{1}{s} J(S) =  f_i^{(b)} + \sum_{S \supset i: s>1} \frac{1}{s} f^{(b)}_S,
\end{align*}
which corresponds to the partial individual effect (\gls{SV}) of b-fANOVA in \cref{tab_influence_summary}.

\paragraph{Integrated Gradients (IG).}
Integrated Gradient (IG) \citep{DBLP:conf/icml/SundararajanTY17} is a method motivated by the Aumann-\gls{SV} \citep{aumann_shapley1974} from the additive cost sharing literature \citep{Friedman_2004}.
IG computes the average gradient over a straight path from $x_0$ to a baseline $b$ as
\begin{align*}
    \phi^{\text{IG}}(i) := (x_i - b_i) \cdot \int_{\alpha=0}^1 \frac{\partial F(b+\alpha \cdot (x-b))}{\partial x_i} d\alpha.
\end{align*}
It was shown \citep{Deng2024Unify} that IG summarizes Taylor interactions as
\begin{align}\label{eq_intgrad_effects}
    \phi^{\text{IG}}(i) = \psi(i) + \sum_{i \in S: s > 1}\sum_{\kappa \in \Omega_S}\frac{\kappa_i}{\sum_{i \in D}{\kappa_i}}I(\kappa),
\end{align}
which corresponds to the b-fANOVA main effect and a partial b-fANOVA interaction effect.

\paragraph{Expected Gradients (EG).}
Expected Gradients (EG) \citep{DBLP:journals/natmi/ErionJSLL21} is an extension of IG to random baselines, which computes IG as an expectation over baselines as
\begin{align*}
    \phi^{\text{EG}}(i) := \mathbb{E}_{b \sim P(b)}\left[(x_i - b_i) \cdot \int_{\alpha=0}^1 \frac{\partial F(b+\alpha \cdot (x-b))}{\partial x_i} d\alpha \right].
\end{align*}

Using the non-recursive definition of b-fANOVA effects (\cref{cor_fanova_moebius}) taking the expectation of $b \sim B$, yields
\begin{align*}
    \mathbb{E}[f^{(B)}(x)] = \mathbb{E}[\sum_{L \subseteq S}(-1)^{s-\ell} F^{(B)}_L(x)] 
    = \mathbb{E}[\sum_{L \subseteq S}(-1)^{s-\ell} F(x_L,B_{-L})] =\sum_{L \subseteq S}(-1)^{s-\ell} \mathbb{E}[F(x_L,B_{-L})].
\end{align*}
The distribution of $B$ is usually determined by the collected data, i.e., $B \sim X$, hence $\mathbb{E}[F(x_L,B_{-L})] = F_L^{(m)}(x)$ corresponds to m-fANOVA.
Reversing the previous equations, we obtain
\begin{align*}
    \mathbb{E}[f_S^{(B)}(x)] = f_S^{(m)}(x).
\end{align*}
Similarly, the Taylor independent and interaction effects with b-fANOVA correspond to the Taylor independent and interaction effects with m-fANOVA, when the expectation is taken over $b$, where $b$ follows the distribution of the features $x$.
In conclusion, EG measures the partial individual effect with respect to m-fANOVA.

\paragraph{DeepLIFT.}
DeepLIFT \citep{Shrikumar_Greenside_Kundaje_2017} is a method that computes attributions layer by layer, similar to LRP but by including a baseline $b$.
It was shown \citep{Deng2024Unify} that DeepLIFT similar to IG is applied in each layer, and thus DeepLIFT accounts in each layer for b-fANOVA effects partially.

\paragraph{Integrated Hessians (IH).}
Integrated Hessians (IH) \citep{DBLP:journals/jmlr/JanizekSL21}
is an extension of IG to interactions.
For two features $i,j$, IH computes
\begin{align*}
    \phi^{\text{IH}}(ij) := (x_i - b_i)(x_j - b_j) \cdot \int_{\alpha_i=0}^1 \int_{\alpha_j=0}^1 \alpha_i \alpha_j \frac{\partial F(b+\alpha_i \alpha_j \cdot (x-b))}{\partial x_i \partial x_j} d\alpha_i d\alpha_j,
\end{align*}
i.e., applying IG for each feature.
By taking the gradient, all effects vanish that do not involve feature $x_i$, as seen in Eq.~(\ref{eq_intgrad_effects}).
In detail, for IG, we have
\begin{align*}
    \int_{\alpha=0}^1 \frac{\partial F(b + \alpha\cdot (x-b))}{\partial x_i} d\alpha 
    &=  \sum_{S \subseteq D}\sum_{\kappa \in \Omega_S}\int_{\alpha=0}^1 \frac{\partial I_{b+\alpha \cdot (x-b)}(\kappa)}{\partial x_i} d \alpha,
\end{align*}
and moreover with the definition of Taylor interactions, we receive by Eq.~(\ref{appx_eq_taylor_exp})
\begin{align*}
    \int_{\alpha=0}^1 \frac{\partial I_{b+\alpha \cdot (x-b)}(\kappa)}{\partial x_i} d \alpha = C(\kappa) \nabla F(\kappa) \pi(\kappa)  \frac{\kappa_i}{x_i-b_i} \int_{\alpha=0}^1 \alpha^{\sum_{i \in D} \kappa_i }d\alpha = I(\kappa) \frac{1}{x_i-b_i} \frac{\kappa_i}{\sum_{i\in D}\kappa_i},
\end{align*}
which is the result provided by \cite{Deng2024Unify}, and stated in \cref{eq_intgrad_effects}.
Note that all terms vanish, where $\kappa_i =0$.
Again, by repeating the same argument for $x_j$, we obtain
\begin{align*}
    \phi^{\text{IH}}(ij) = \sum_{ij \subseteq S} \sum_{\kappa \in \Omega_S}\frac{\kappa_i \kappa_j}{(\sum_{i \in D} \kappa_i)^2} I(\kappa) = J(ij) + \sum_{ij \subseteq S: s>2} \sum_{\kappa \in \Omega_S} \frac{\kappa_i \kappa_j}{(\sum_{i \in D} \kappa_i)^2} I(\kappa),
\end{align*}
i.e., IH corresponds to a partial (pairwise) interaction effect of b-fANOVA.

\paragraph{Accumulated Local Effects (ALE)}
Accumulated Local Effects (ALE) \citep{apley_visualizing_2020} computes first the gradients using the conditional distribution and then averages over the feature's values in a certain range
\begin{align*}
    \phi_x^{\text{ALE}}(i) := \int_{x_{\min}}^{x_i}\mathbb{E}[\frac{\partial F(X)}{\partial x_i} \mid X_i = z] dz.
\end{align*}

ALE analyzes the conditional expectation of the partial derivative of $F$, which corresponds to the pure effect with c-fANOVA on the derivative.
By averaging over the values $X_i=z$ from $z=x_{\min}$ to $x_i$, ALE captures effects related to $F$, which however are difficult to state precisely due to the unknown conditional distribution. 
ALE further computes second-order effects by applying ALE on both features consecutively, and subtracting main effects, thereby measuring interaction effects.
This further ensures a full decomposition \citep{apley_visualizing_2020}, if applied on all possible effects, which is a result of the \gls{MT}-like computation of effects.

\subsection{Unifying Global Feature-Based Explanations}
In this section, we unify global feature-based explanations using our framework.

\subsubsection{Sensitivity-Based Explanations}\label{appx_sec_unif_sens}
We first discuss measures from sensitivity analysis, which are based on the global sensitivity game.

\paragraph{Sobol' Indices.}
The Sobol' indices \citep{Sobol_2001, Owen_2013} were introduced as a global variance importance measure.
The closed Sobol' index $\underline\tau$ measures the variance of marginalized predictions $\phi^{\text{ClosedSobol'}}(S) := \mathbb{V}[F^{(m)}_S(X)] = \sum_{T \subseteq S} \sigma^2_T$, where feature independence is assumed, thereby corresponding to m-fANOVA.
Clearly, the closed Sobol' index is the pure joint effect, or pure individual effect.
In contrast, the total Sobol' index $\bar\tau$ quantifies the full individual and joint effect of m-fANOVA, which is clear from its definition and \cref{tab_influence_summary}.

\paragraph{H-Statistic.}
The H-statistic \citep{friedman_predictive_2008} assumes feature independence, and was introduced as a measure of interaction.
For two features $i,j$ it approximates the variance of the joint minus the individual marginalized predictions as 
\begin{align*}
    \phi^{\text{H-Statistic}}(ij) := \frac{\mathbb{V}[F^{(m)}_{ij}(X)-F^{(m)}_{i}(X)-F^{(m)}_j(X)]}{\mathbb{V}[F^{(m)}_{ij}(X)]} = \frac{\mathbb{V}[f^{(m)}_{ij}(X)]}{\mathbb{V}[F^{(m)}_{ij}(X)]},
\end{align*}
and thus corresponds to a scaled variant of the pure interaction effect of the global sensitivity game using m-fANOVA for independent features.

\paragraph{Sobol'-SV.}
The Sobol'-\gls{SV} was introduced for independent \citep{Owen_2014} and dependent \citep{Owen.2017} features, thereby corresponding to m- and c-fANOVA, respectively.
It measures the partial individual effect due to the \gls{SV} applied on the marginalized prediction $F_S$.

\paragraph{Superset Measure.}
The superset measure $\Upsilon^2$ measures the full interaction effect, which is clear by its definition and \cref{tab_influence_summary}, i.e., $\Upsilon^2$ corresponds to the Co-\gls{MT} applied on m-fANOVA.

\subsubsection{Performance-Based (Risk-Based) Explanations}\label{appx_sec_unif_risk}

In this section, we summarize performance-based measures that rely on the global risk game.

\paragraph{Permutation Feature Importance (PFI) and Grouped PFI.}
The PFI \citep{Breiman2001, fisher2019all} is a global feature importance method that measures the performance drop after removing an individual feature by marginally perturbing the feature of interest. Therefore, the PFI is defined by the full individual influence measure $\phi^+(i)$ based on the global risk game and m-fANOVA:
\begin{equation*}
    \phi^{\text{PFI}}(i) =  \mathbb{E}[\ell(F^{(m)}_{-i}(X), Y)] - \mathbb{E}[\ell(F(X), Y)] = \phi^+(i) = m^{(risk)}(i) + \sum_{T\supset i} m^{(risk)}(T),
\end{equation*}
where $F_{-i}^{(m)}$ corresponds to the outcome when feature $i$ is removed by marginalization in presence of the remaining features $-i$ using m-fANOVA.

The Grouped PFI \citep{au2022grouped} is a global feature importance measure for groups of features and corresponds to PFI by marginally perturbing a feature group instead of an individual feature. Hence, it is defined by the full joint influence measure $\phi^+(S)$ based on the global risk game and m-fANOVA:
\begin{equation*}
    \phi^{\text{GPFI}}(S) =  \mathbb{E}[\ell(F^{(m)}_{-S}(X), Y)] - \mathbb{E}[\ell(F(X), Y)] = \phi^+(S) =  m^{(risk)}(S) + \sum_{T \cap S \notin \{S,\emptyset\}} m^{(risk)}(T).
\end{equation*}

\paragraph{Conditional Feature Importance (CFI).}
Conditional Feature Importance (CFI) \citep{strobl2008conditional,molnar2023model} is similar to PFI but relies on conditional expectations, thereby measuring the full individual and full joint effect of c-fANOVA.

\paragraph{Shapley Additive Global Importance (SAGE).}
Shapley additive global importance (SAGE) \citep{covert2020sage} was proposed as the \gls{SV} on the \emph{normalized} risk game using c-fANOVA, i.e.,
\begin{align*}
    \nu^{\text{SAGE}}(S) := \nu^{(\text{risk})}(S) - \nu^{(\text{risk})}(\emptyset) = \mathbb{E}[\ell(F^{(c)}_\emptyset(X),Y)] - \mathbb{E}[\ell(F^{(c)}_S(X),Y)].
\end{align*}
Notably, the \gls{SV} is not affected by such transformations.
Hence, SAGE measures the partial individual effect of c-fANOVA. 
In practice, marginal expectations are used to compute SAGE values thereby corresponding to m-fANOVA.

\clearpage
\section{DETAILS ON EXPERIMENTAL SETUP AND REPRODUCIBILITY}\label{appx_sec_exp_setup}

This section contains details on the experimental setup, models and datasets. All code and supplementary material is available at \url{https://github.com/FFmgll/unifying_feature_based_explanations}.

\subsection{Synthetic Examples}\label{appx_sec_synth_data_model}

In this section, we describe in more detail the data settings and modeling approach of the synthetic example in Section 5.1 and of further examples analyzed in more detail in Appendix G.1.

\paragraph{Datasets.}
For all our synthetic examples, we consider $4$ features that are multivariate normally distributed, i.e., $X \sim \mathcal{N}(\mu, \Sigma)$ with $\mu = 0$, $\sigma_{ii}^2 = 1$ and $\sigma_{ij} = \rho$ for $i \neq j$. We vary the correlation coefficient between no, medium, and high correlation by setting it to $\rho \in \{0,0.5,0.9\}$. For these three data settings, we define the following functional relationships
\begin{align*}
    F_{\text{lin}}(x) &= 2x_1 + 2 x_2 + 2x_3,\\
    F_{\text{int}}(x) &= F_{\text{lin}} + x_1 x_2 + x_1 x_2 x_3,\\
    F_{\text{add}}(x) &= 2x_1 + x_2^2 + x_3^3,
\end{align*}
where all of them are corrupted by noise $\mathcal{N}(0,0.01)$ to generate the target $Y$.

\paragraph{Models.}

To model the data-generating processes $F_{\text{lin}}$, $F_{\text{int}}$, and $F_{\text{add}}$, we randomly draw 10,000 samples and fit correctly specified linear models with \texttt{sklearn} \citep{Sklearn}, ensuring correctly learned effects for each functional form.
Input features are transformed to match the underlying functional relationships, with the relevant main effects and interaction terms selected accordingly.
This method guarantees that the models retain linearity in the parameter space, consistent with the specified functional relationships.
As a result, we obtain an almost perfect fit across all three synthetic datasets, achieving $R^2$ values greater than 0.99 on a separately sampled test dataset of size 10,000 following the same distribution as defined above.

\subsection{Models and Datasets for Real-World Examples}\label{appx_sec_real_world_models_datasets}

In addition to the synthetic datasets, we evaluate the performance of our models on four real-world datasets: \emph{California housing}, \emph{Bike sharing},  \emph{titanic}, and \emph{IMDB}, including regression and classification tasks. In the following, we describe the datasets and the corresponding models.
With the exception of \emph{IMDB}, we fit models on each dataset and evaluate the performance on regression tasks with the mean squared error (MSE) and the the coefficient of determination ($R^2$). 
For classification, we evaluate models using accuracy (ACC) and the $F_1$ score. All results are averaged over five Monte Carlo cross-validation runs to ensure robustness.
We use \texttt{XGBoost} for training gradient-boosted decision trees \citep{XGBoost} and \texttt{PyTorch} for training neural networks \citep{pytorch}. All other trained models and preprocessing steps are implemented with \texttt{sklearn} \citep{Sklearn}.
In the case of the \emph{IMDB} language data we do not fit a model but retrieve a pre-trained and fine-tuned model from the \texttt{transformers} framework \citep{Wolf_Transformers_State-of-the-Art_Natural_2020}.

\paragraph{California Housing.} 

This dataset is used for a regression task, where the goal is to predict median house prices based on features such as population, median income, and location-related attributes. The following models were evaluated:

\begin{itemize}
\item \textbf{XGBoost}: Applied directly to the raw dataset without additional preprocessing, this model achieved an $R^2 = 0.8330 \pm 0.0028$ and an MSE of $0.2207 \pm 0.0014$ using default hyperparameters.
\item \textbf{Neural Network}: A fully connected neural network with four hidden layers (sizes 100, 50, 10, 1) was trained on normalized data, where all features were treated as numeric and standardized. The neural network achieved an $R^2 = 0.7461 \pm 0.0117$ and an MSE of $0.5246 \pm 0.7460$. 
\end{itemize}

\paragraph{Bike Sharing.} 
The \emph{Bike sharing} dataset is used for a regression task, predicting the total number of bike rentals based on features such as weather conditions, season, and time of day. 
The dataset was preprocessed by applying a \texttt{RobustScaler} to the numerical features to handle outliers, and ordinal encoding was applied to the categorical features \texttt{season} and \texttt{weather}.
For neural network models, the target variable was log-normalized using the transformation $y=\text{log}(y+1)$ to stabilize variance and improve model performance.
The following models were evaluated:

\begin{itemize}
\item \textbf{Random Forest}: A random forest model with 50 trees was trained on the preprocessed data, achieving an $R^2$ of $0.9471 \pm 0.0024$ and an MSE of $1\,863 \pm 103$. 
\item \textbf{Neural Network}: A fully connected neural network with four hidden layers (sizes 100, 50, 10, 1) was trained on the preprocessed data, where the categorical features were treated as numeric and standardized.
The regression target variable was additionally log-normalized.
This model achieved an $R^2$ of $0.9436 \pm 0.0016$ and an MSE of $0.3299 \pm 0.0055$.
\end{itemize}

\paragraph{Titanic.} 
The \emph{Titanic} dataset is used for a binary classification task that predicts the survival probability of passengers based on characteristics such as age, sex, class of travel, and fare.
We preprocess the data by applying an \texttt{OrdinalEncoder} to the categorical features.

\begin{itemize}
    \item \textbf{XGBoost}: The model was applied to the preprocessed data using default hyperparameters, achieving an $F_1$ score of $0.7408 \pm 0.0236$ and an accuracy (ACC) of $0.8119 \pm 0.0215$.
\end{itemize}

\paragraph{IMBD.} 
The sentiment analysis language data is based on the \emph{IMDB} dataset \citep{Maas.2011,Lhoest_Datasets_A_Community_2021} where example sentences to the model are self-created. The model is a pre-trained and fine-tuned version of \texttt{DistilBERT} \citep{Sanh.2019} on \emph{IMDB} as provided with the \texttt{transformers} \citep{Wolf_Transformers_State-of-the-Art_Natural_2020} framework. 
The model predicts the sentiment of movie reviews with a sentiment score $\left[-1,1\right]$ where values close to $-1$ denote a \textit{negative} predicted sentiment and values close to $1$ correspond to a \textit{positive} predicted sentiment.
Sentiment scores close to $0$ are considered neutral.
In our experiments, we manually remove punctuation and special symbols.
In case of b-fANOVA, we set out-of-coalition tokens to the \texttt{MASK} token of the transformer model.
Notably, the selected \texttt{DistilBERT} variant assigns sentences provided only \texttt{MASK} tokens moderately high scores of around $0.50$ (i.e., on the positive side).
We do not evaluate and report performance scores of the pre-trained model on a separate test set.

\subsection{Experimental Setup}\label{appx_sec_experiment_setup}

We compute explanations for the models and datasets described in \cref{appx_sec_real_world_models_datasets,appx_sec_synth_data_model} based on the \texttt{shapiq} \citep{Shapiq.2024} python library. 
To do so, we create three kinds of cooperative explanation games, a \emph{local explanation game} (see \cref{def_local_explanation_game}), a \emph{global risk game} (see \cref{def_global_risk_game}), and a \emph{global sensitivity game} (see \cref{def_sensitivity_game}).
Both the global risk and global sensitivity games are based on a collection of independent local explanation games.
The local explanation game consists of a model, a data point, and an imputation strategy.
The imputation strategies follow the definitions of $F_S$ in the b-fANOVA, m-fANOVA, and c-fANOVA.
To model b-fANOVA, we compute mean feature values from a background dataset for each feature independently and impute missing features with these baseline values.
For m-fANOVA, we sample replacement data points directly from a background dataset according to the joint marginal distribution.
For the synthetic datasets and known functional dependencies (see \cref{appx_sec_synth_data_model}), we compute the c-fANOVA with the ground-truth conditional distributions.
For real-world datasets and c-fANOVA, we rely on the conditional imputer implemented in the \texttt{shapiq} package \citep{DBLP:conf/aaai/MuschalikFHH24}.
For marginal and conditional imputation we average over $512$ sampled replacement data points for each game evaluation.
All imputation strategies operate on the test datasets.

\subsection{Hardware Details}\label{appx_hardware_detials}

All computations have been conducted on consumer hardware, i.e., Dell XPS 15 (Intel i7 13700H, 16GB RAM) and required only minor computation times. Overall computations took around 60 CPUh.

\clearpage
\section{FURTHER EXPERIMENTAL RESULTS}\label{appx_sec_further_exp}

This section contains additional empirical results.
\cref{appx_results_synth} contains a deep dive into the synthetic setting mentioned throughout the main body of this work (see \cref{sec_unified_framework,sec_experiments}). 
\cref{appx_results_real_world} extends on \cref{sec_experiments_real_world} and contains additional results on real-world data.

\subsection{Dissecting the Linear Models on Synthetic Data}\label{appx_results_synth}

Here, we provide more detailed explanations for the results of the synthetic experiments summarized in Section 5.1 as well as of the other two functional relationships defined in Appendix F.1. We start with explaining the individual feature effects for $F_{\text{lin}}$ which is the basis for $F_{\text{int}}$ being the model we analyzed in Section 5.1. We then show how the additionally multiplicative interaction term in $F_{\text{int}}$ influences each feature's effect. Finally, we demonstrate how individual effects for an additive (non-linear) main effect model $F_{\text{add}}$ differ from linear effects.
All explanations are compared based on the instance $x = (1,1,1,1)$.

\paragraph{Linear Main Effect Model.}
The left plot in \cref{fig_appx_synt_linear} illustrates the influence of feature distribution and of higher-order interactions for $F_{\text{lin}}$. It shows that the pure, partial, and full effects have the same value (main effect) for all features for b- and m-fANOVA. This indicates that there are no feature interactions present between any of the features. Since both decompositions break the dependencies between the feature of interest and the remaining features in the dataset and due to a purely linear additive relationship of the features, feature dependencies do not influence the individual effects. Feature $x_4$ does not have an influence on the model and thus has an individual effect of zero when b- or m-fANOVA are used. This changes for c-fANOVA since it accounts for cross-correlation effects with other influential features. Therefore, the pure individual effects of c-fANOVA are the same as the ones of b- and m-fANOVA for the uncorrelated setting, but adds to each individual effect the effect of the remaining features by the factor $\rho$ in the correlated settings. Hence, the interpretation of the pure effect with c-fANOVA quantifies how much of the overall feature effects we account for if we add the feature of interest to an empty model while considering the underlying feature distribution. It follows that the full effect for $\rho = 0.9$ is almost zero for all features since the other three features have already been added to the model and account for the effects of the remaining feature through cross-correlation effects. 

\paragraph{Linear Model with Main and Interaction Effects.}
The right plot in \cref{fig_appx_synt_linear} is a reprint of the plot shown in Section 5.1 for the model $F_{\text{int}}$. It shows that pure and partial m-fANOVA effects of $x_3$ are affected by $\rho$. This effect is caused by the interaction effect between $x_1$ and $x_2$ which influences the joint marginal distribution of $x_{-3}$ by their covariance and hence influences the pure and partial individual effects of the feature $x_3$. The partial and full plots of b- and m-fANOVA clearly show that interactions are present between the first three features with a higher interaction between $x_1$ and $x_2$. C-fANOVA again accounts for cross-correlation effects and thus shows the same pattern for pure effects as $F_{\text{lin}}$, while indicating learned feature interactions for partial and full effects.

\begin{figure}[htb]
    \centering
    \includegraphics[width=0.49\linewidth]{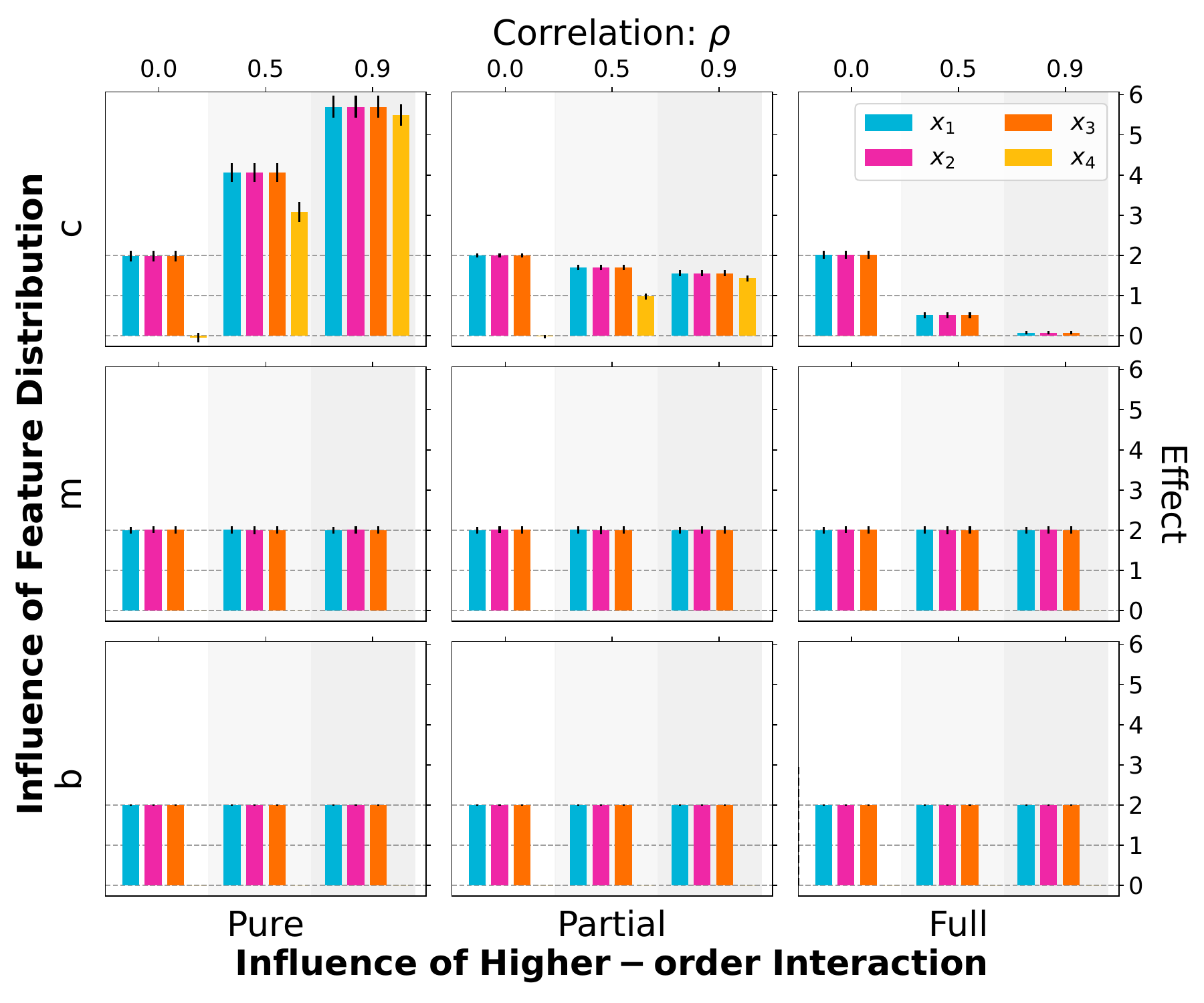}
    \includegraphics[width=0.49\linewidth]{explanations_all_linear_interaction_12x10.pdf}
    \caption{Individual local explanations for the instance $x = (1,1,1,1)$ and the linear model (left) with the ground-truth functional relationship $F_{\text{lin}}(x) = 2x_1 + 2 x_2 + 2 x_3$ and the linear model including interactions (right) with the ground-truth functional relationship $F_{\text{int}}(x) = 2x_1 + 2 x_2 + 2 x_3 + x_1x_2 + x_1x_2x_3$ averaged over $30$ repetitions of varying random seeds (fluctuation is shown by error bars).}
    \label{fig_appx_synt_linear}
\end{figure}

\paragraph{Additive Main Effect Model.}
\cref{fig_appx_synt_non_linear_interactions} shows the local effects for $F_{\text{add}}$. Comparable to $F_{\text{lin}}$, b-fANOVA is neither influenced by the feature distribution ($\rho$) nor by the higher-order interactions. While m-fANOVA is equivalent to b-fANOVA for $F_{\text{lin}}$, it differs for the still additive relationship, when effects are non-linear: The effect of $x_2$ for the instance $x = (1,1,1,1)$ becomes $0$ since $f_1(x) = \mathbb{E}[F_{\text{add}}(x_2, X_{-2})] - \mathbb{E}[F(X)] = x_2 + \mathbb{E}[X_1 + X_3^3] - \mathbb{E}[X_1 + X_2^2 + X_3^3] = 0$. However, as for b-fANOVA, pure, partial, and full effects are the same and thus indicating that no feature interactions have been learned by the model. Due to accounting for cross-correlation effects, $x_4$ shows again an influence for pure and partial effects, when c-fANOVA is used. Moreover, full effects with c-fANOVA show a similar vanishing trend as for $F_{\text{lin}}$ and $F_{\text{int}}$.

\begin{figure}[htb]
    \centering
    \includegraphics[width=0.49\linewidth]{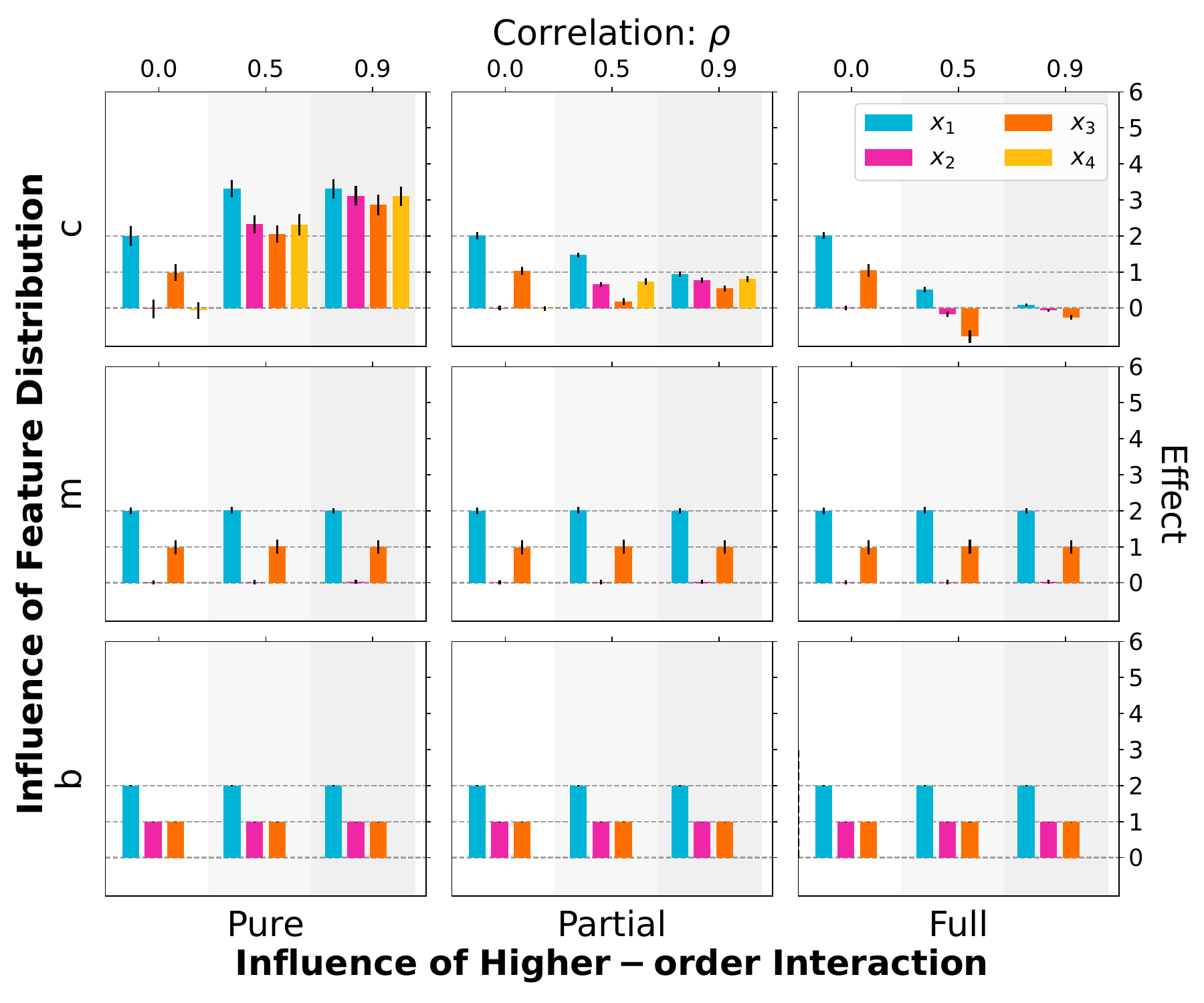}
    \caption{Individual local explanations for the instance $x = (1,1,1,1)$ and the additive main effect model (left) with the ground-truth functional relationship $F_{\text{add}}(x) = 2x_1 + x_2^2 + x_3^3$ averaged over $30$ repetitions of varying random seeds (fluctuation is shown by error bars).}
    \label{fig_appx_synt_non_linear_interactions}
\end{figure}

\subsection{Additional Explanations on Real-World Data}\label{appx_results_real_world}

This section describes additional explanation results on real-world benchmark data as shown in \cref{fig_appendix_titanic_xgb_local,fig_appendix_calif_nn,fig_appendix_bike_xgb_global,fig_appendix_titanic_xgb_sensitivity}.

\paragraph{Titanic Data - Local Explanation Game.}
\cref{fig_appendix_titanic_xgb_local} shows a series of local explanations with pure (a), partial (b), and full (c) higher-order interaction influence for the \emph{Titanic} dataset. 
The local instance under analysis was a 49 year old man who embarked from Cherbourg, holding an upper-class ticket with a relatively high fare price.
The XGBoost model correctly predicts that he survived the disaster.
The explanations indicate that his \emph{Sex} (male) and \emph{Age} (49 years) decrease his chances of survival.
In contrast, his high fare price (\emph{Fa}), the port of embarkation (\emph{Em}), and his upper-class ticket substantially enhance his survival probability.
This finding aligns with descriptive analyses of the \emph{Titanic} dataset, which reveal that passengers who embarked from Cherbourg had the highest survival rate of $0.56$ in comparison to Queenstown with $0.38$ and $0.33$ for Southampton.
Notably, the model also highlights increased survival odds due to interactions between \emph{Em} and \emph{Fa}, as well as between \emph{Em} and \emph{Age}.
Both of these second-order interactions are observed across all three higher-order interaction influences (a-c) depicted in \cref{fig_appendix_titanic_xgb_local}.

\paragraph{Titanic Data - Global Sensitivity Game.}
\cref{fig_appendix_titanic_xgb_sensitivity} displays pure (a), partial (b), and full (c) effects, reported as an individual (top row), interaction (middle row), and joint (bottom row) effect for the global sensitivity game for an XGBoost model fitted on the \emph{Titanic} dataset. From the individual explanations it can be observed that the variables \emph{Sex} and \emph{Pc} (ticket class) are most important for the prediction. 
Moreover, the \emph{Pc}, \emph{Fa} (passenger fare), and \emph{Em} (Port of Embarkation) features have relatively high full effects, indicating the presence of interactions.
In fact, from the interactions up to order $2$ (middle row) it can be observed that there are synergies (positive two-way interactions) between \emph{Em} and \emph{Sex}, as well as between \emph{Pc} and \emph{Sex}.
Moreover, partial and full effects indicate that \emph{Fa} is involved in higher-order interactions (order $>2$).
Lastly, the joint effects are very homogeneous, due to the dominating individual effect of the feature \emph{Sex}, which is accounted for in all pairwise joint effects that involve this feature.

\paragraph{California Housing Data - Global Risk Game.}
\cref{fig_appendix_calif_nn} displays global feature importance in terms of the global risk game for a neural network trained on the \emph{California housing} dataset.
We report pure, partial (SAGE) and full (PFI) importances for individuals and second-order interactions.
Looking at individual importances (top row), reveals that the features \emph{la} (latitude of property) and \emph{lo} (longitude of property) are strongly present in higher-order interactions.
In fact, the pure importances appear negative, which is counter-intuitive. Due to the marginal imputation the model predicts worse, if only one of these features is known.
As a consequence, pure effects of the risk game for complex models should be interpreted with care, if m-fANOVA or b-fANOVA is used (thereby breaking dependencies).
From the interactions (bottom row), we observe a strong positive pure effect between \emph{la} and \emph{lo} indicating that the exact location of the property is important to the model's performance.

\paragraph{Bike Sharing Data - Global Risk Game.}
\cref{fig_appendix_bike_xgb_global} shows pure, partial, and full effects for the global risk game and the \emph{Bike sharing} dataset using an XGBoost regression model.
The partial effect (middle) for individuals (top row) corresponds thereby to SAGE, whereas the full effect (right) reported as individual importance (top row) to PFI.
It can be observed that the feature \emph{ho} (hour of the day) is most influential for the performance.
Observing differences in pure and full effects, it is apparent that the \emph{wo} (working day) feature increases its importance indicating that it is involved in higher-order interactions.
In fact, investigating second-order interactions, a positive interaction (synergy) of \emph{ho} and \emph{wo} is observed.
In general, the differences between pure, partial and full effects are mostly minor indicating the absence of strong interaction effects.

\paragraph{IMDB - Local Explanation Game.}
The local explanation game for the \emph{IMDB} movie review dataset is based on \cite{sundararajan2020shapley}, \cite{tsai_faith-shap_2022}, and \cite{DBLP:conf/aaai/MuschalikFHH24} and we rely on b-fANOVA with the \texttt{MASK} token as baseline. The sentiment scores are within $\left[-1,1\right]$ where values close to $-1$ denote a \textit{negative} predicted sentiment and values close to $1$ correspond to a \textit{positive} predicted sentiment.
Sentiment scores close to $0$ are considered neutral.
The target sentence ``\textit{The acting was bad, but the movie enjoyable}'' receives a positive predicted sentiment of approximately $0.8117$, and the baseline prediction is also rather positive around $0.5361$.
\cref{fig_appendix_language_local} depicts pure (left), partial (middle), and full (right) individual effects (upper row) and interactions up to the second order (lower row) with 2-SVs (2-SII) for the partial effect.
All three effects show strong differences, which indicates the presence of higher-order interactions.
The pure individual effects highlight the positive influence of ``enjoyable'', whereas ``bad'' and ``but'' obtain negative influences.
Interestingly, ``acting" also receives negative influence, which possibly indicates a bias in the prediction.
The partial individual effect (SV) shows a stronger positive influence for "enjoyable'', and a weaker negative influence for ``acting'' and ``bad''.
In contrast, ``but'' now receives a positive partial influence, which indicates its involvement in higher-order interactions.
In fact, ``but'' and ``enjoyable'' receives a strong positive interaction, as well as ``bad'' and ``but''.
Lastly, full effects show predominantly positive influences, which indicates the presence of many positive higher-order interactions.

\begin{figure}
    \centering
    \hfill
    \begin{minipage}[c]{0.32\textwidth}
    \centering
    \phantom{\textbf{Pure, Order 1}}
    \\
    \includegraphics[width=\textwidth]{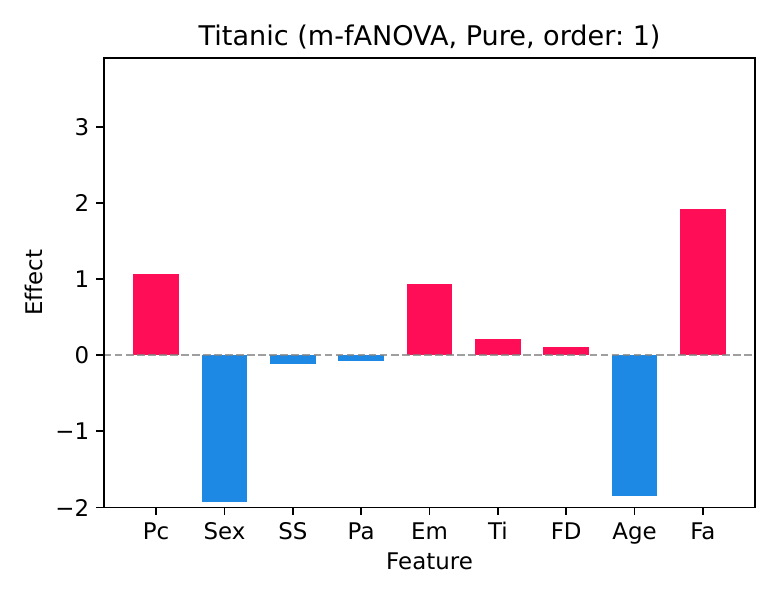}
    \\
    \includegraphics[width=\textwidth]{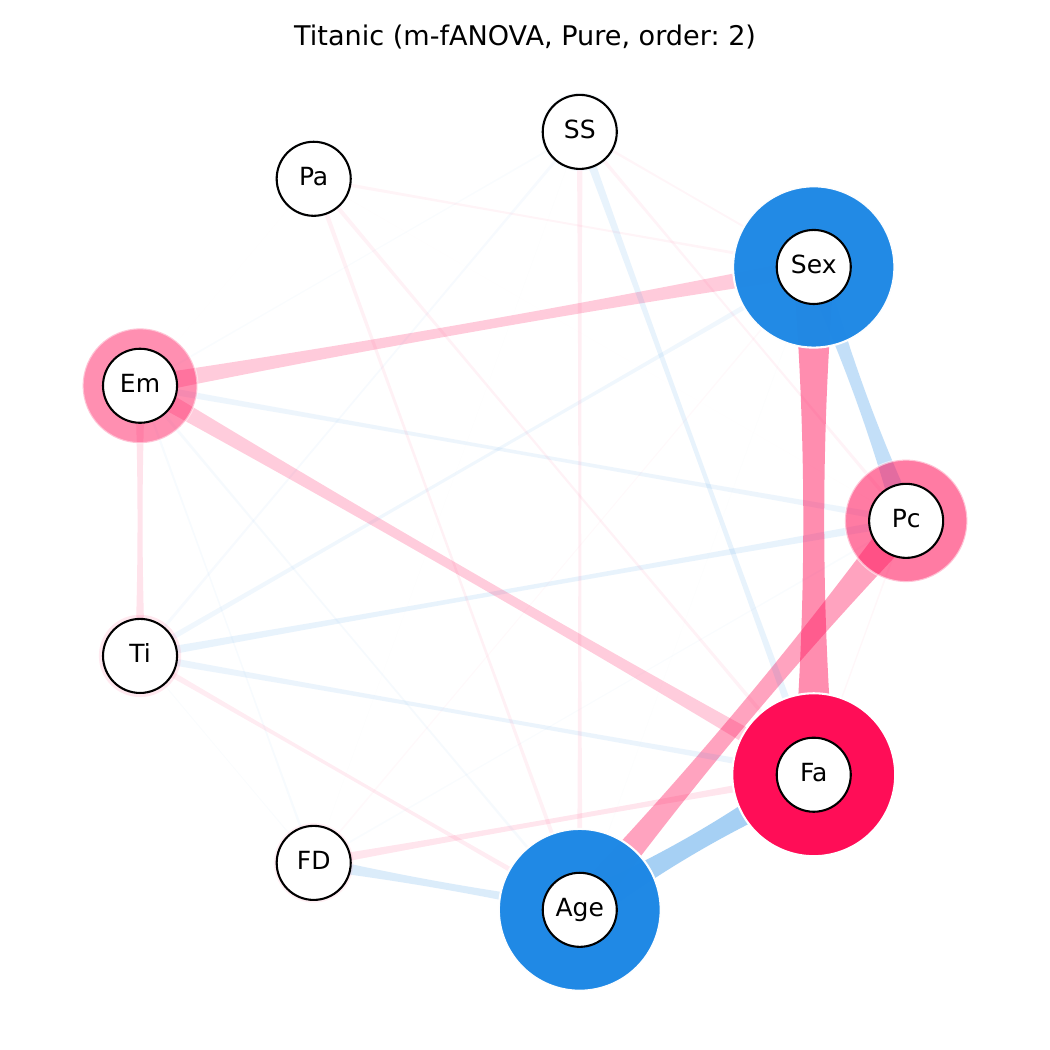}
    \\[1em]
    \textbf{(a) Pure}
    \end{minipage}
    \hfill
    \begin{minipage}[c]{0.32\textwidth}
    \centering
    \textbf{SHAP}
    \\
     \includegraphics[width=\textwidth]{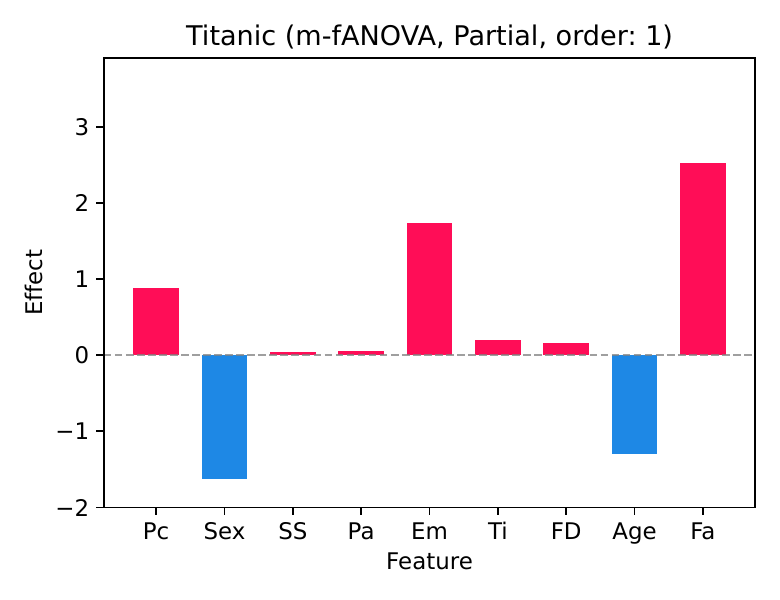}
    \\
    \includegraphics[width=\textwidth]{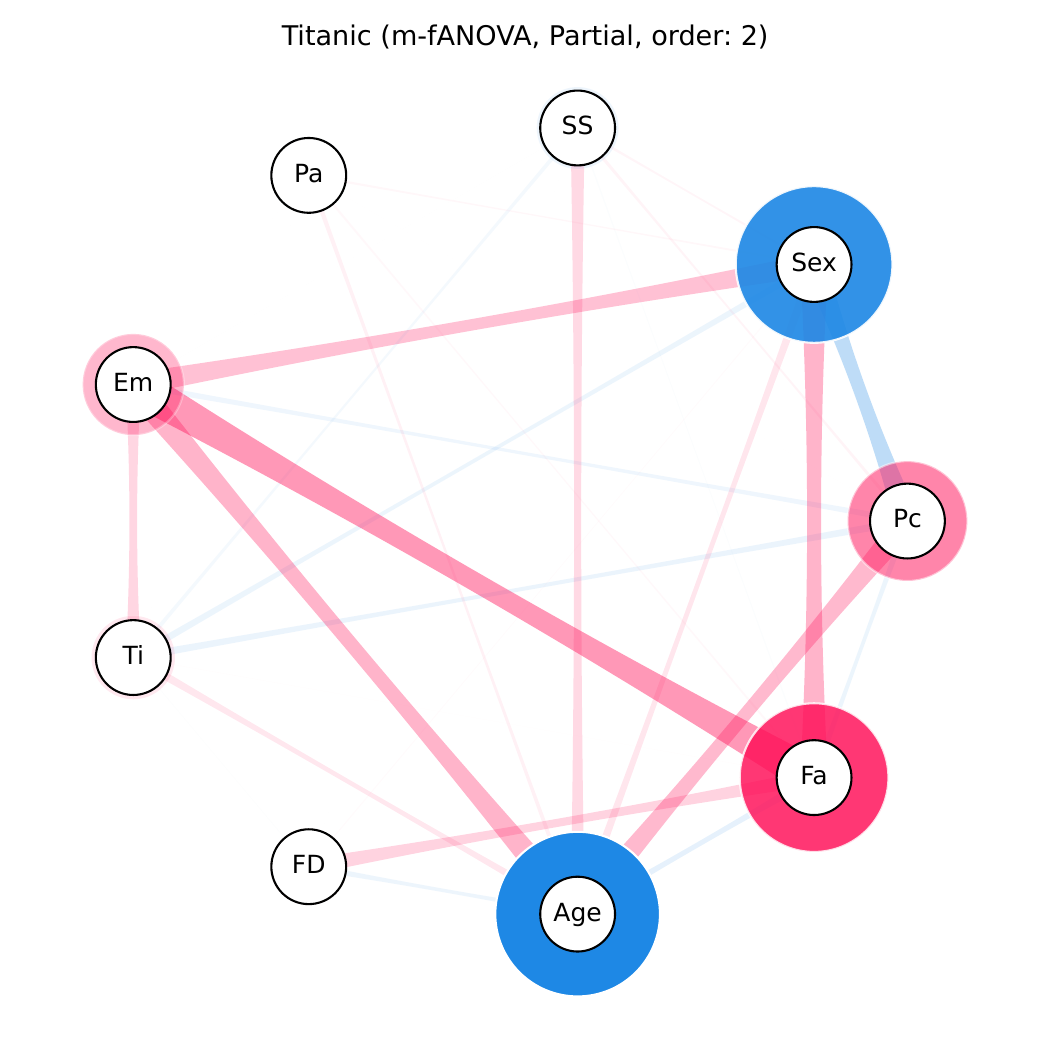}
    \\[1em]
    \textbf{(b) Partial}
    \end{minipage}
    \hfill
    \begin{minipage}[c]{0.32\textwidth}
    \centering
    \phantom{\textbf{Pure, Order 1}}
    \\
    \includegraphics[width=\textwidth]{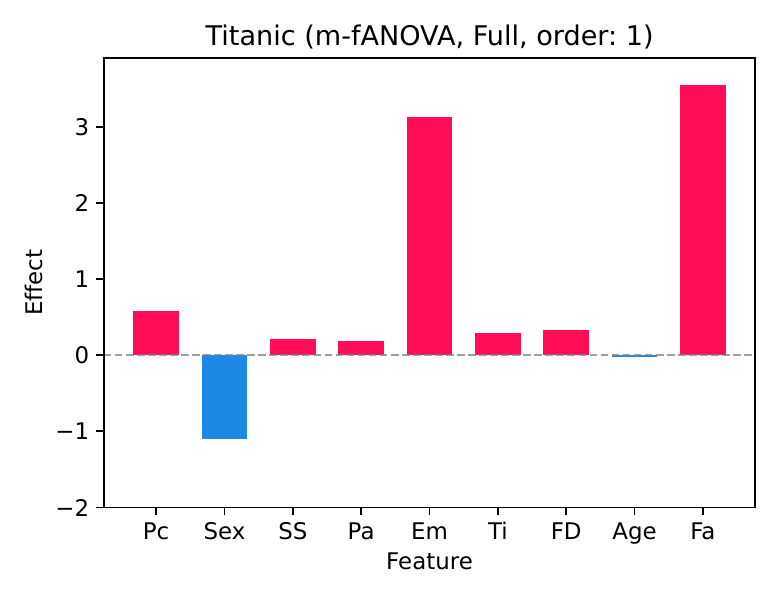}
    \\
    \includegraphics[width=\textwidth]{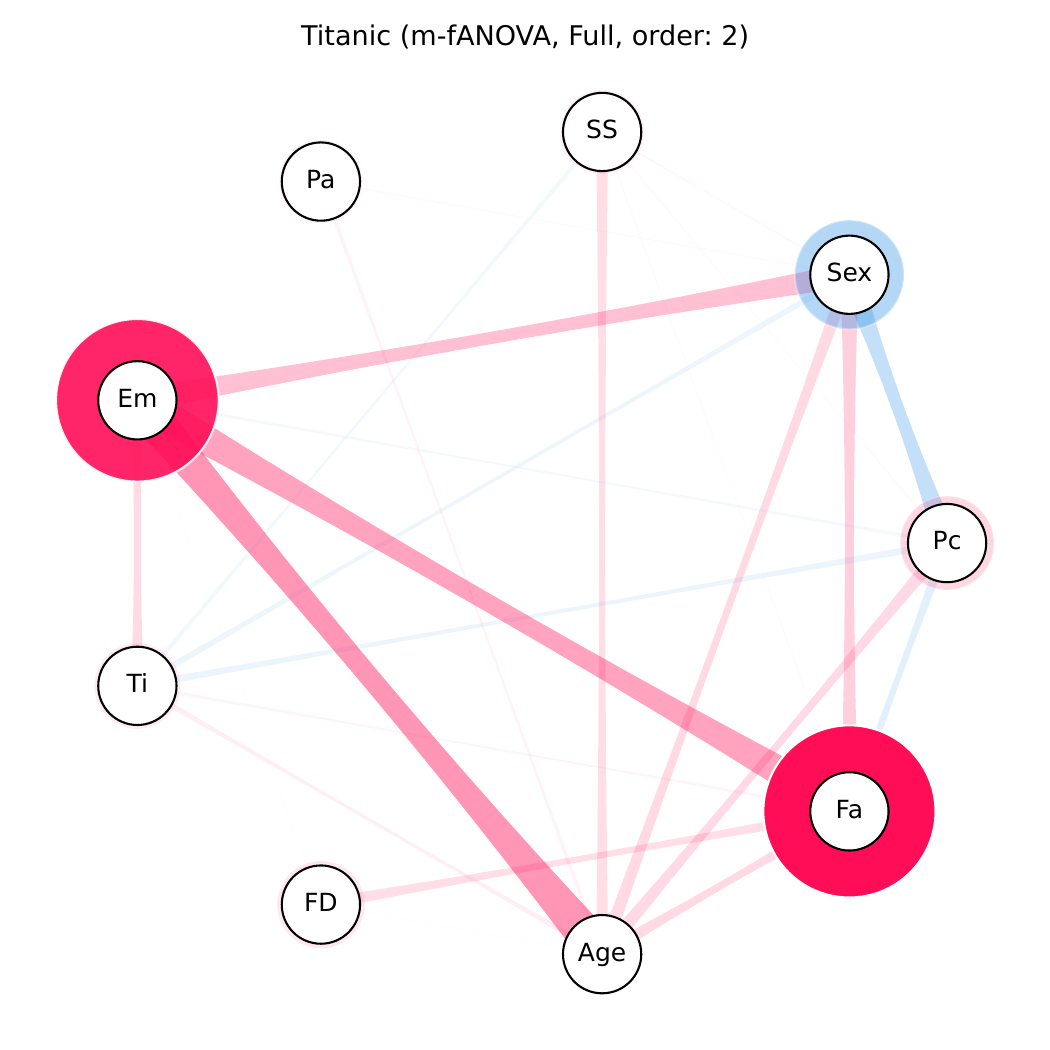}
    \\[1em]
    \textbf{(c) Full}
    \end{minipage}
    \vspace{1.5em}
    \caption{\textbf{Local Explanation Game:} Pure (a), partial (b), and full (c) feature influences for an XGBoost model fitted on \emph{Titanic}. The first row depicts the individual explanations where (b) corresponds to interventional SHAP. The second row shows the pure, partial (2-\glspl{SV}, 2-SII) and full interactions up to the second order. The local instance, which is explained has the following feature values: : $x_{\text{Sex}} = \text{male}$, $x_{\text{Age}} = 49$, $x_{\text{SibSp}} = 1$, $x_{\text{Parch}} = 0$, $x_{\text{Fare}} = 56.9292$, $x_{\text{Embarked}} = \text{C}$, $x_{\text{Title}} = \text{Rare Title}$, and $x_{\text{FsizeD}} = \text{small}$.}
    \label{fig_appendix_titanic_xgb_local}
\end{figure}

\begin{figure}
    \centering
    \hfill
    \begin{minipage}[c]{0.32\textwidth}
    \centering
    \phantom{\textbf{Pure, Order 1}}
    \\
    \includegraphics[width=\textwidth]{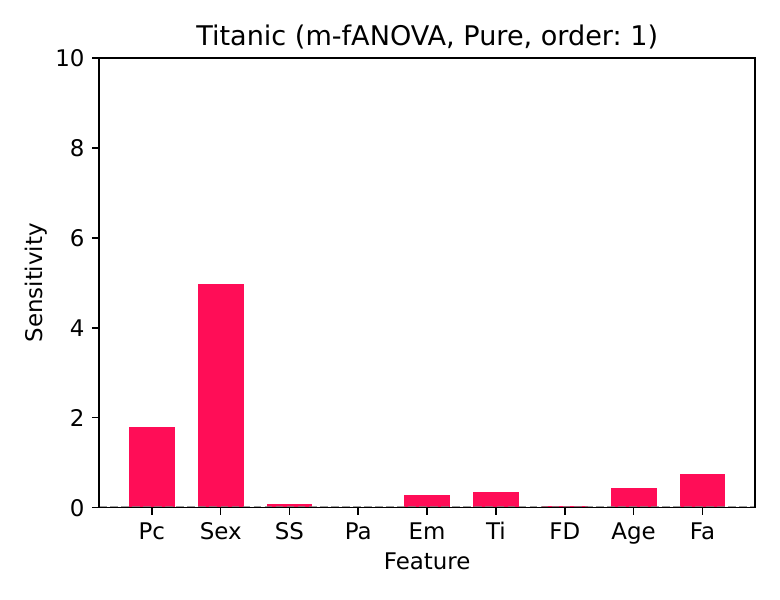}
    \\
    \includegraphics[width=\textwidth]{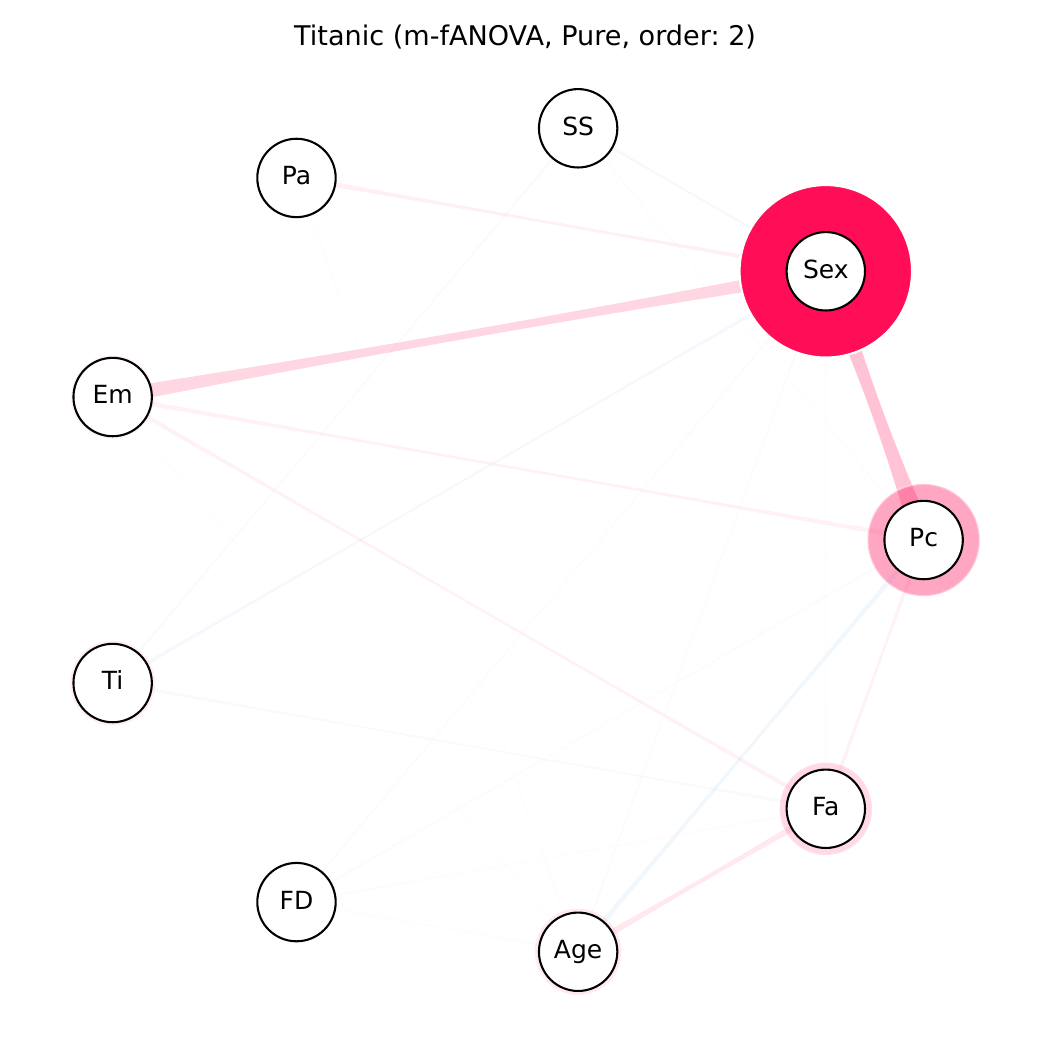}
    \\
    \includegraphics[width=\textwidth]{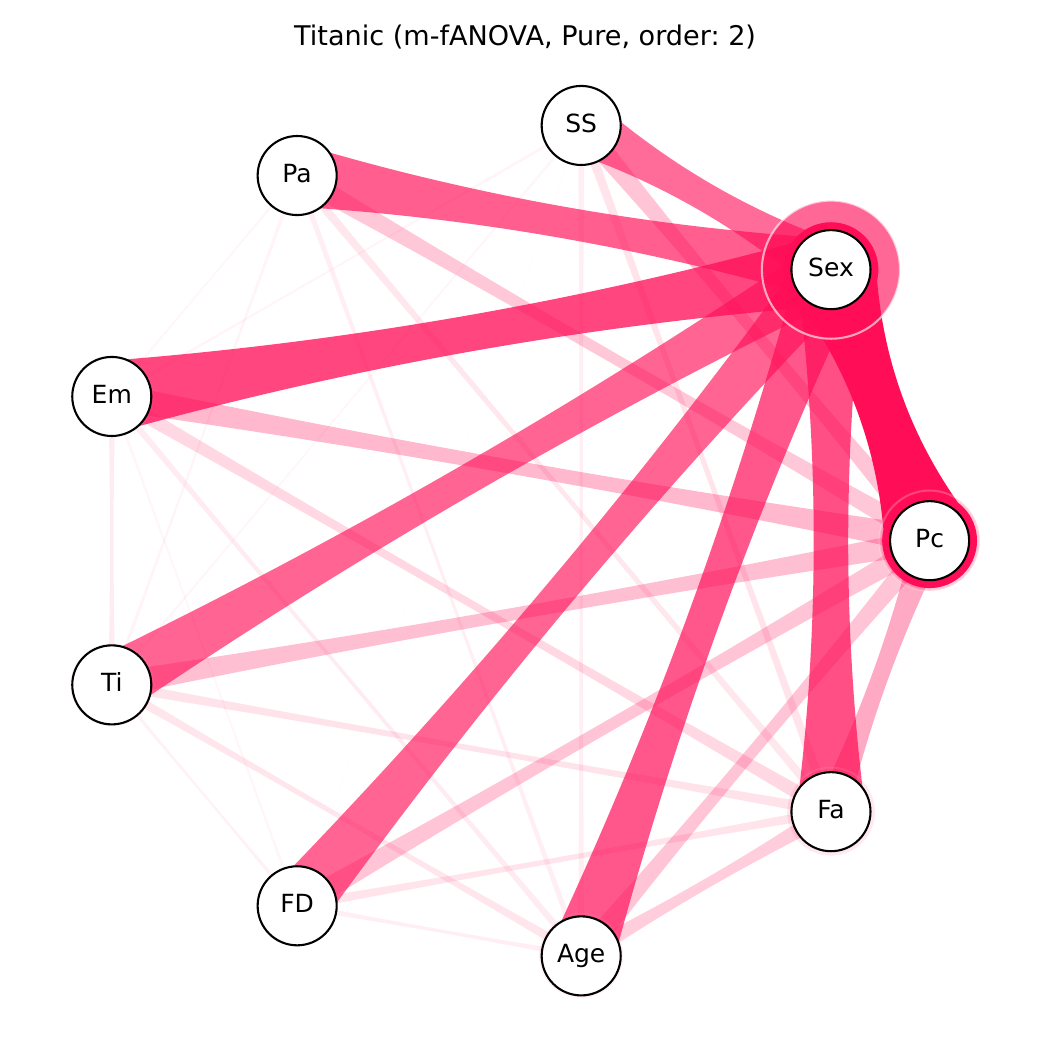}
    \\[1em]
    \textbf{(a) Pure}
    \end{minipage}
    \hfill
    \begin{minipage}[c]{0.32\textwidth}
    \centering
    \phantom{\textbf{S}}
    \\
     \includegraphics[width=\textwidth]{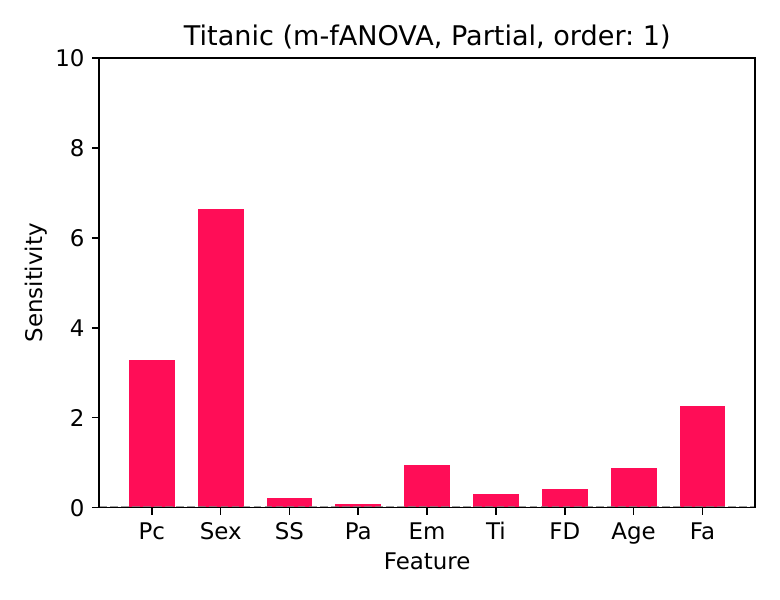}
    \\
    \includegraphics[width=\textwidth]{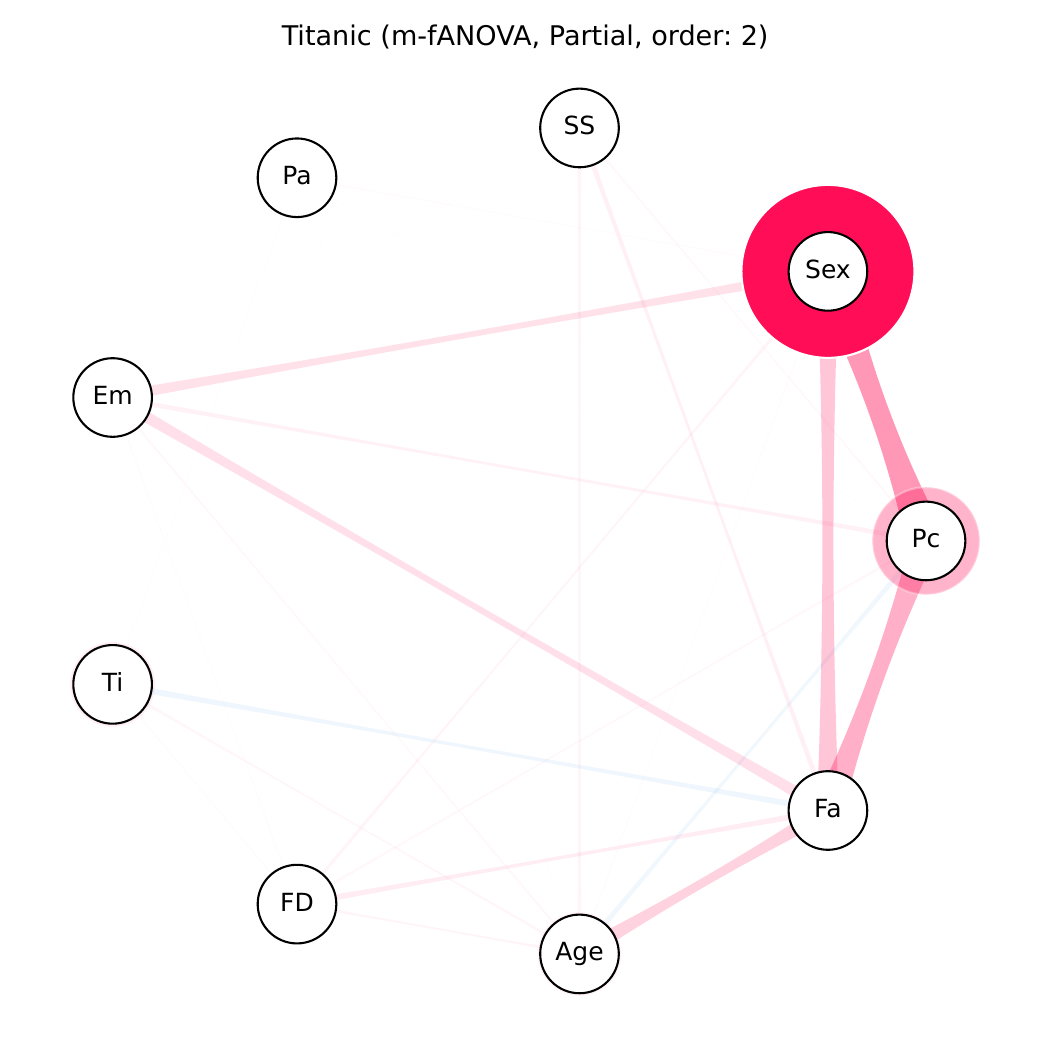}
    \\   
    \includegraphics[width=\textwidth]{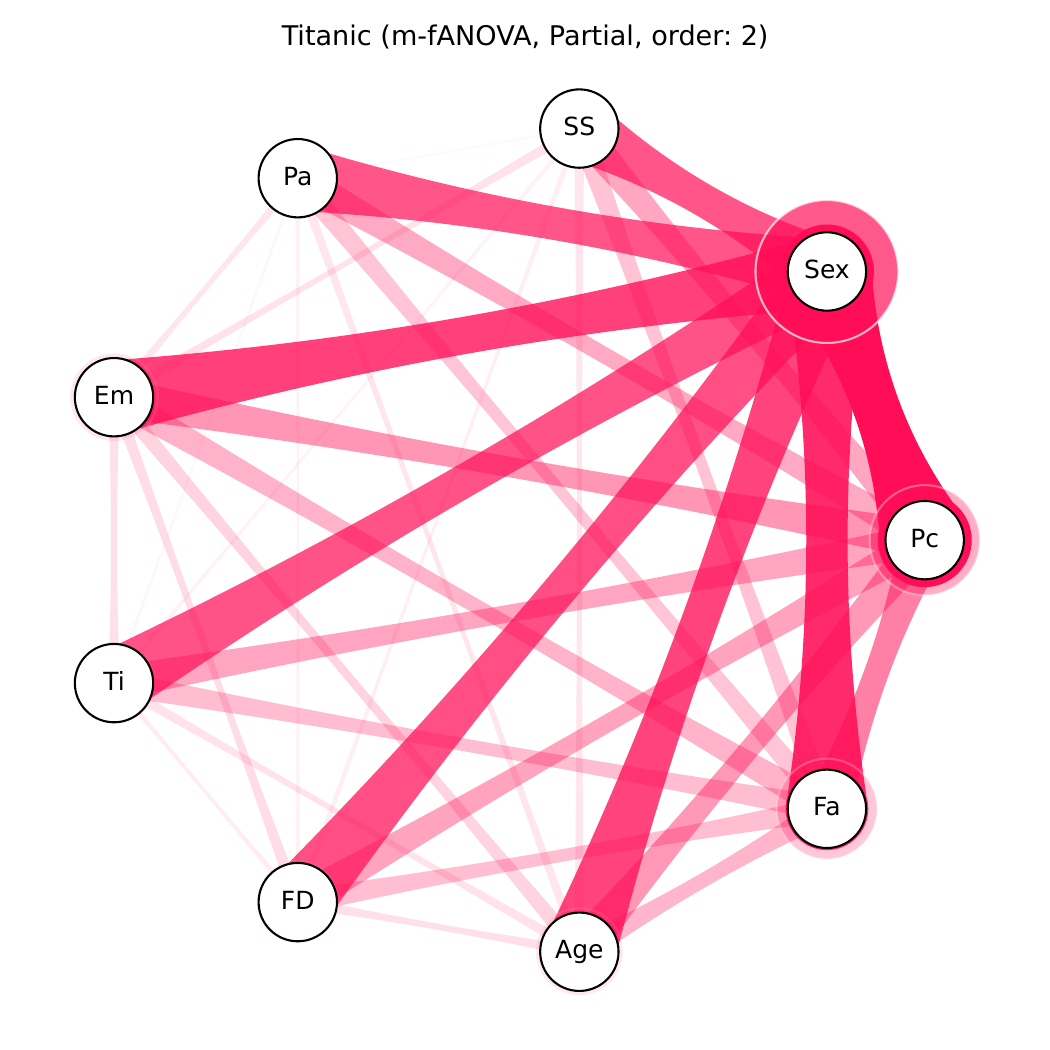}
    \\[1em]
    \textbf{(b) Partial}
    \end{minipage}
    \hfill
    \begin{minipage}[c]{0.32\textwidth}
    \centering
    \phantom{\textbf{Superset Measure $\Upsilon^2$}}
    \\
    \includegraphics[width=\textwidth]{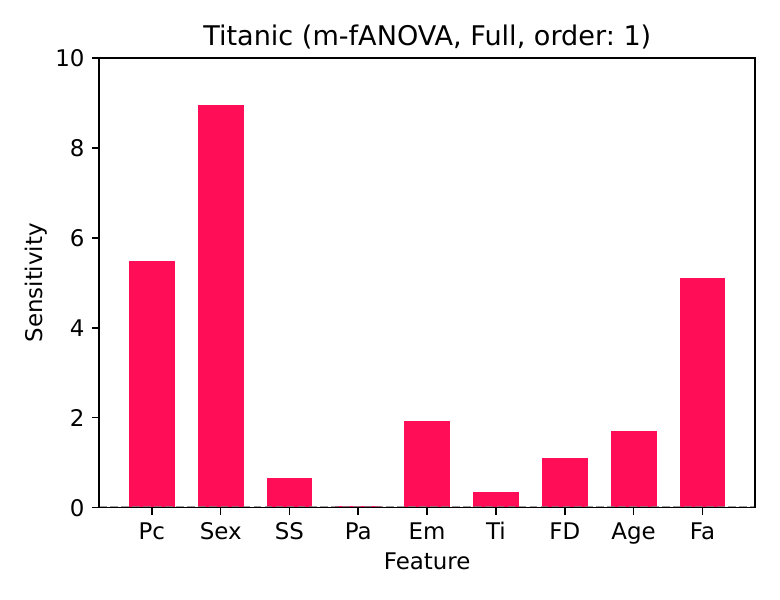}
    \\
    \includegraphics[width=\textwidth]{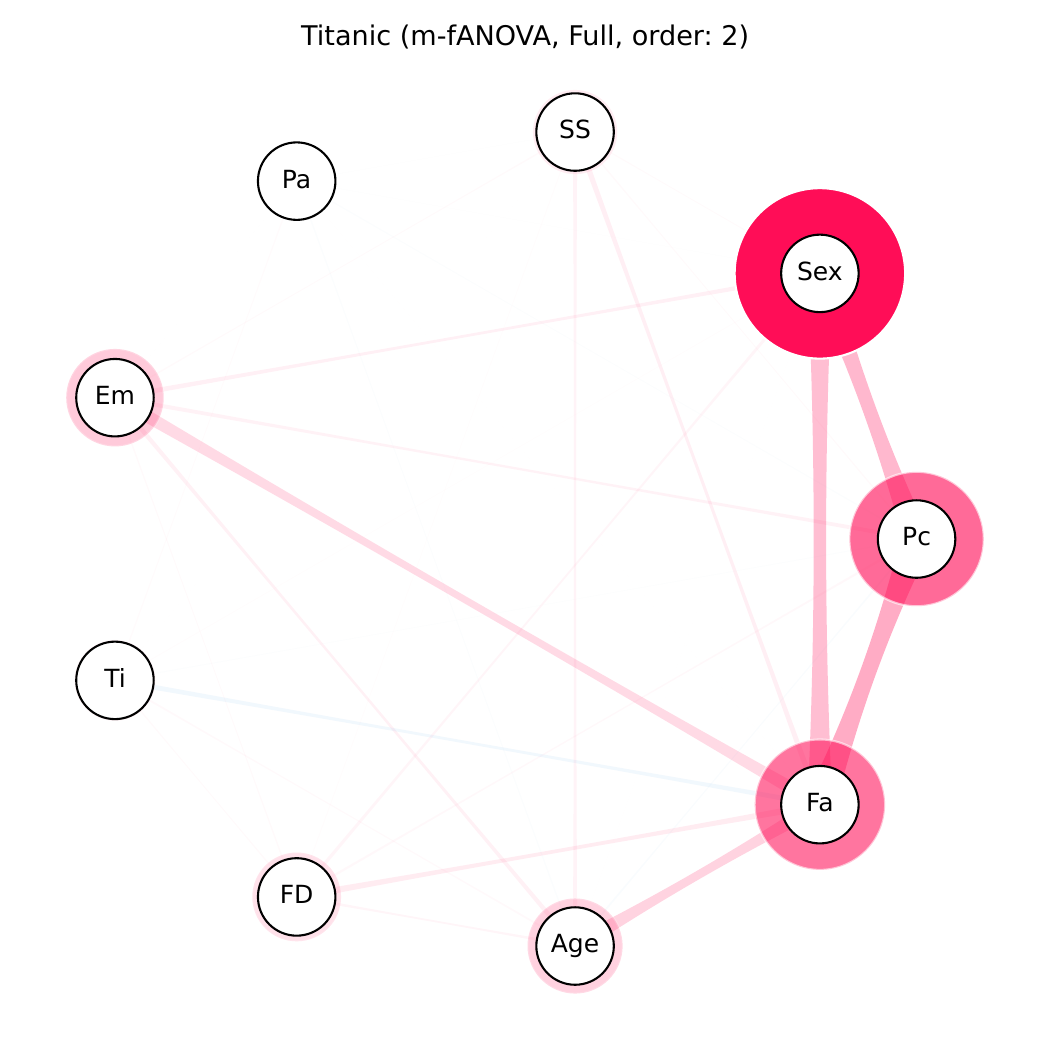}
    \\
    \includegraphics[width=\textwidth]{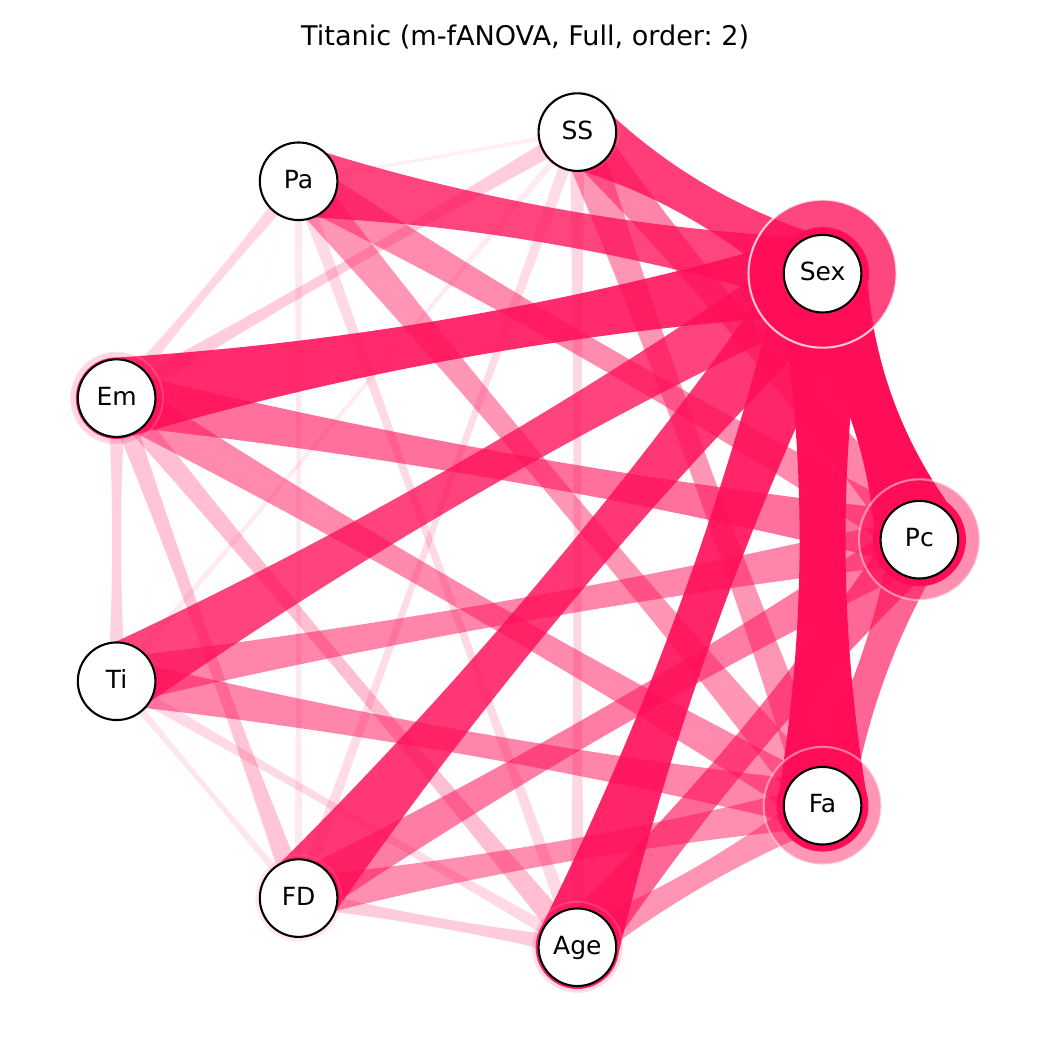}
    \\[1em]
    \textbf{(c) Full}
    \end{minipage}
    \vspace{1.5em}
    \caption{\textbf{Global Sensitivity Game:} Pure (a), partial (b), and full (c) feature influences for an XGBoost model fitted on \emph{Titanic}. The first row depicts the individual explanations, the second row interactions, and the third row the joint effects up to order $2$. The column (c) corresponds to the superset measure $\Upsilon^2$. This global risk game is based on 500 randomly drawn local explanation games from the test dataset.}
    \label{fig_appendix_titanic_xgb_sensitivity}
\end{figure}

\begin{figure}
    \centering
    \hfill
    \begin{minipage}[c]{0.32\textwidth}
    \centering
    \phantom{\textbf{Pure, Order 1}}
    \\
    \includegraphics[width=\textwidth]{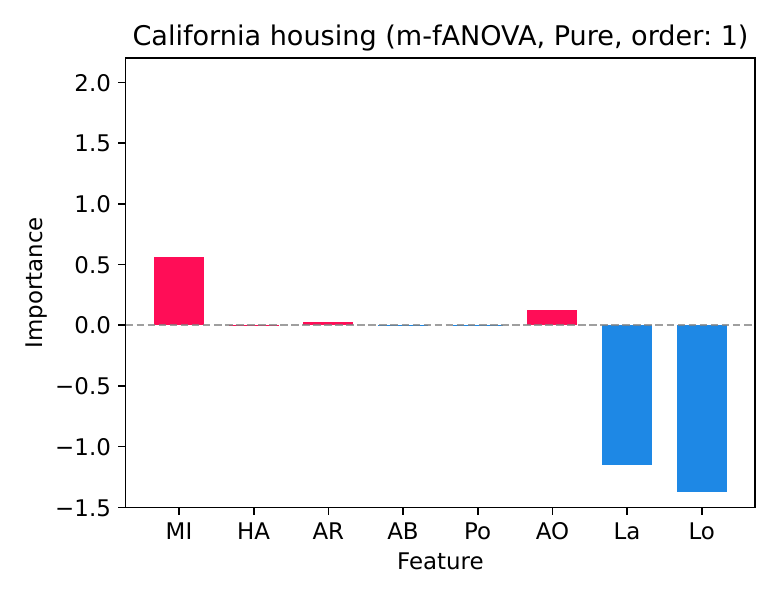}
    \\
    \includegraphics[width=\textwidth]{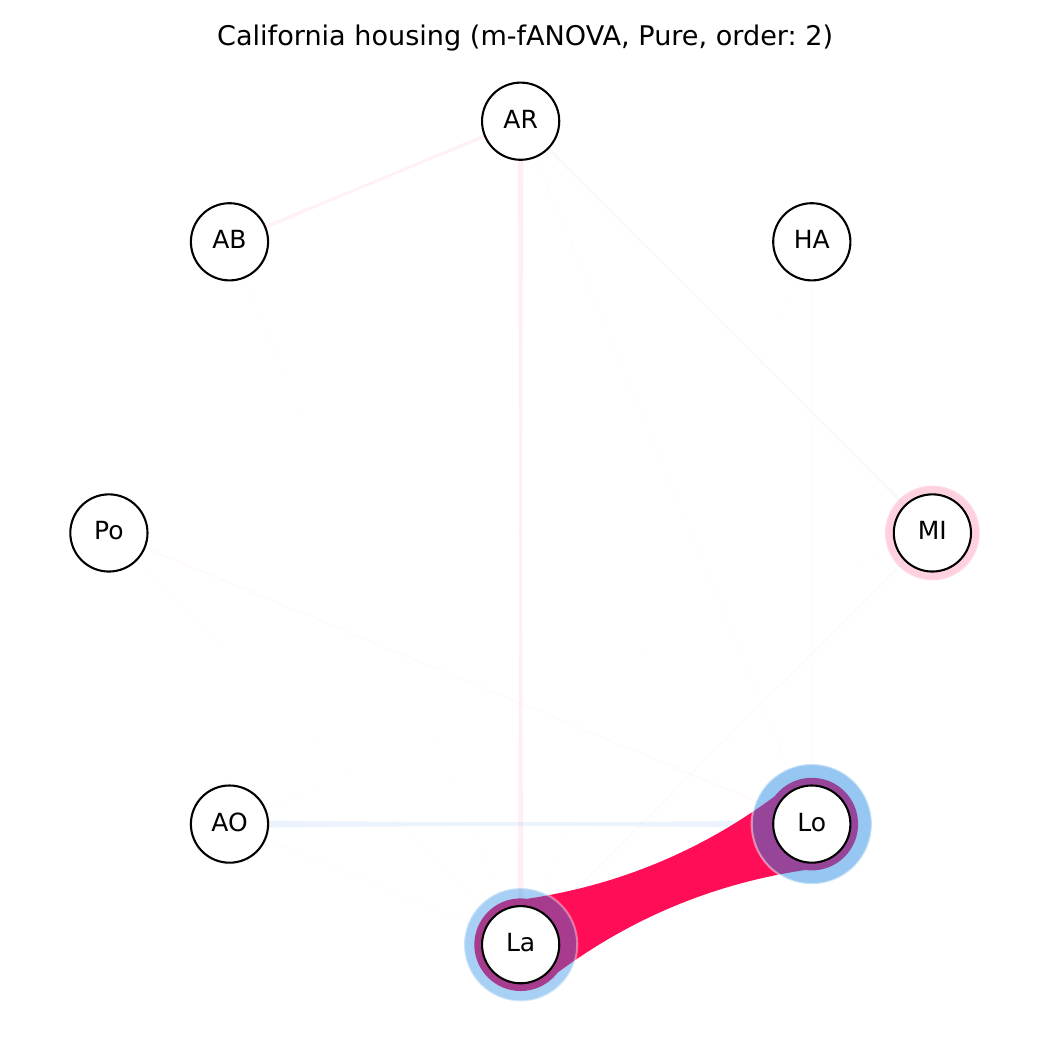}
    \\[1em]
    \textbf{(a) Pure}
    \end{minipage}
    \hfill
    \begin{minipage}[c]{0.32\textwidth}
    \centering
    \textbf{SAGE}
    \\
     \includegraphics[width=\textwidth]{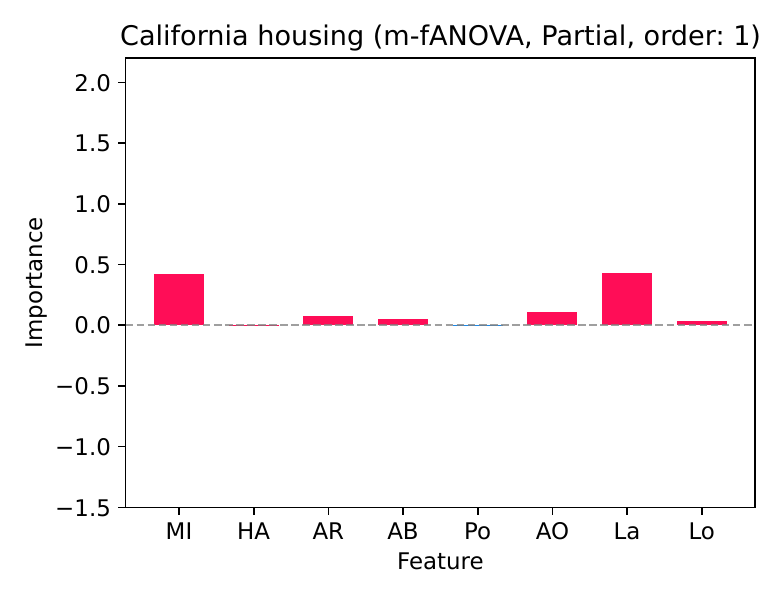}
    \\
    \includegraphics[width=\textwidth]{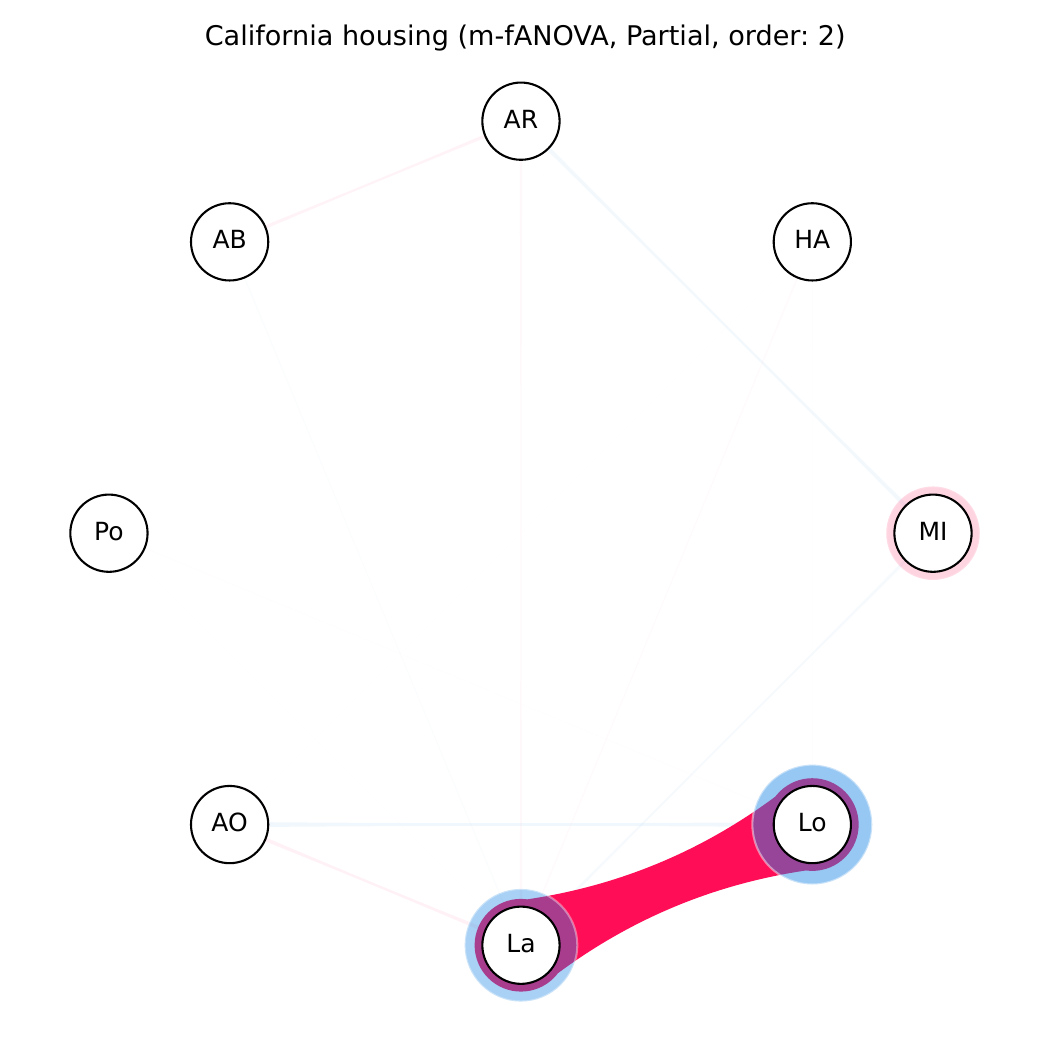}
    \\[1em]
    \textbf{(b) Partial}
    \end{minipage}
    \hfill
    \begin{minipage}[c]{0.32\textwidth}
    \centering
    \textbf{PFI}
    \\
    \includegraphics[width=\textwidth]{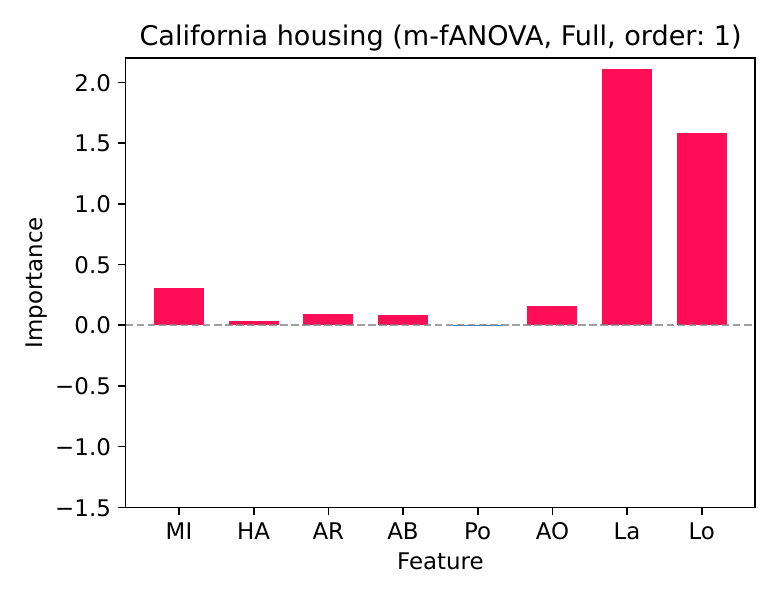}
    \\
    \includegraphics[width=\textwidth]{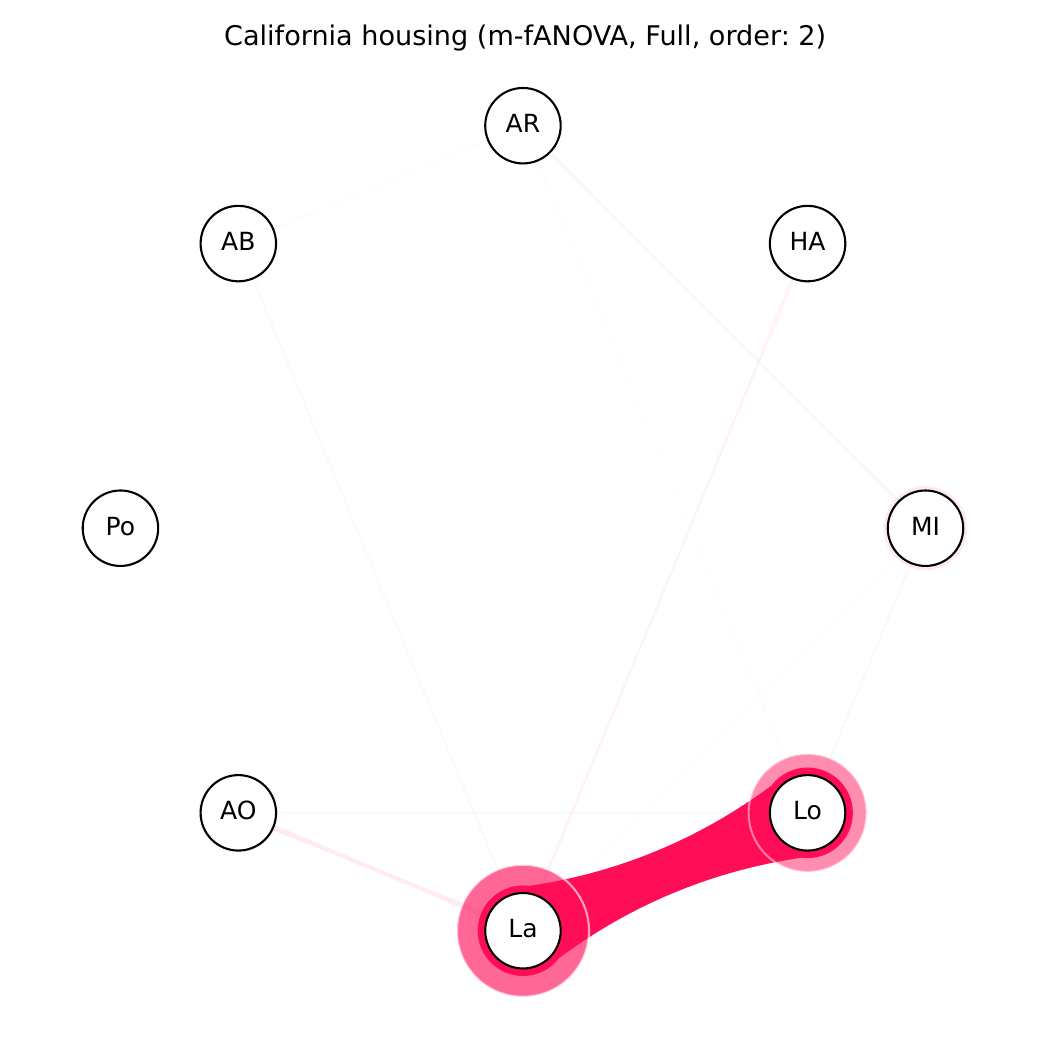}
    \\[1em]
    \textbf{(c) Full}
    \end{minipage}
    \vspace{1.5em}
    \caption{\textbf{Global Risk Game:} Pure (a), partial (b), and full (c) feature influences for a neural network fitted on \emph{California housing}. The first row depicts the individual explanations where (b) corresponds to SAGE and (c) corresponds to PFI. The second row shows the pure, partial (2-\glspl{SV},2-SII) and full interactions up to the second-order. Importance is defined according to the global risk game by the change in performance measured by the mean squared error. This global risk game is based on 500 randomly drawn local explanation games from the test dataset.}
    \label{fig_appendix_calif_nn}
\end{figure}

\begin{figure}
    \centering
    \hfill
    \begin{minipage}[c]{0.32\textwidth}
    \centering
    \phantom{\textbf{Pure, Order 1}}
    \\
    \includegraphics[width=\textwidth]{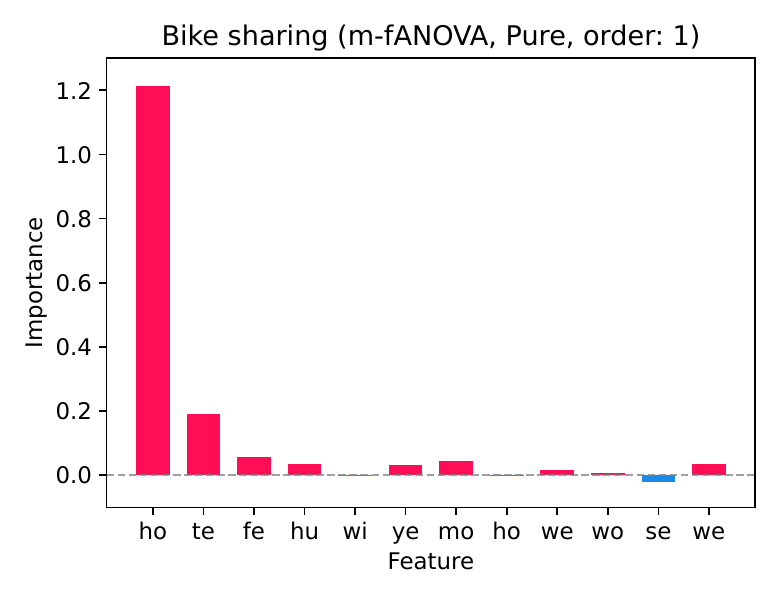}
    \\
    \includegraphics[width=\textwidth]{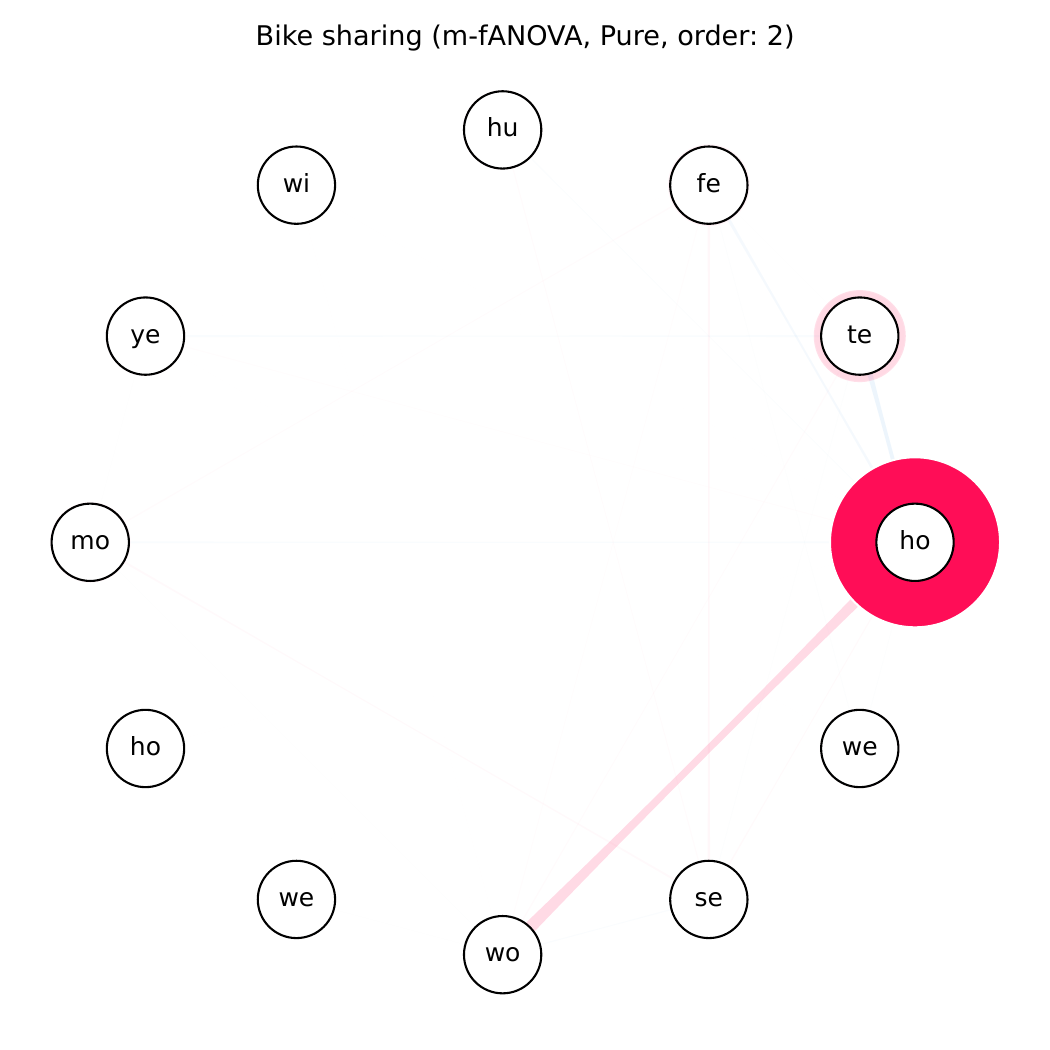}
    \\[1em]
    \textbf{(a) Pure}
    \end{minipage}
    \hfill
    \begin{minipage}[c]{0.32\textwidth}
    \centering
    \textbf{SAGE}
    \\
     \includegraphics[width=\textwidth]{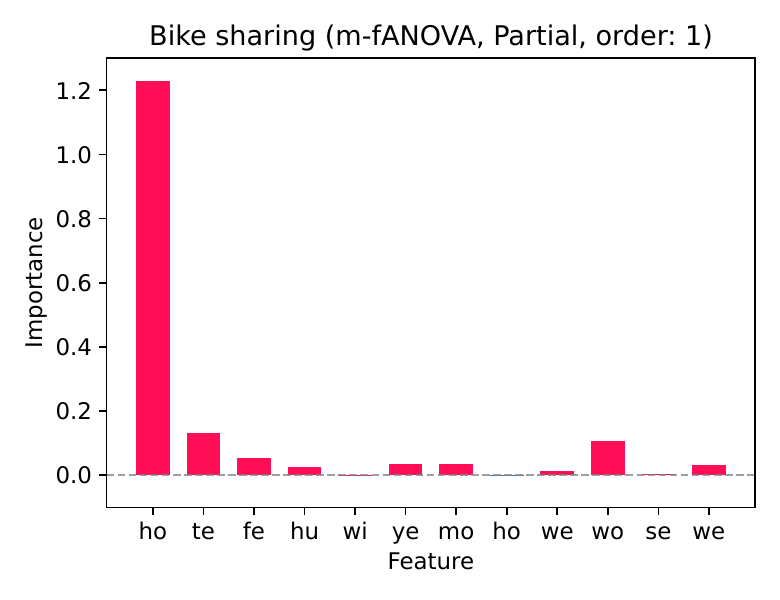}
    \\
    \includegraphics[width=\textwidth]{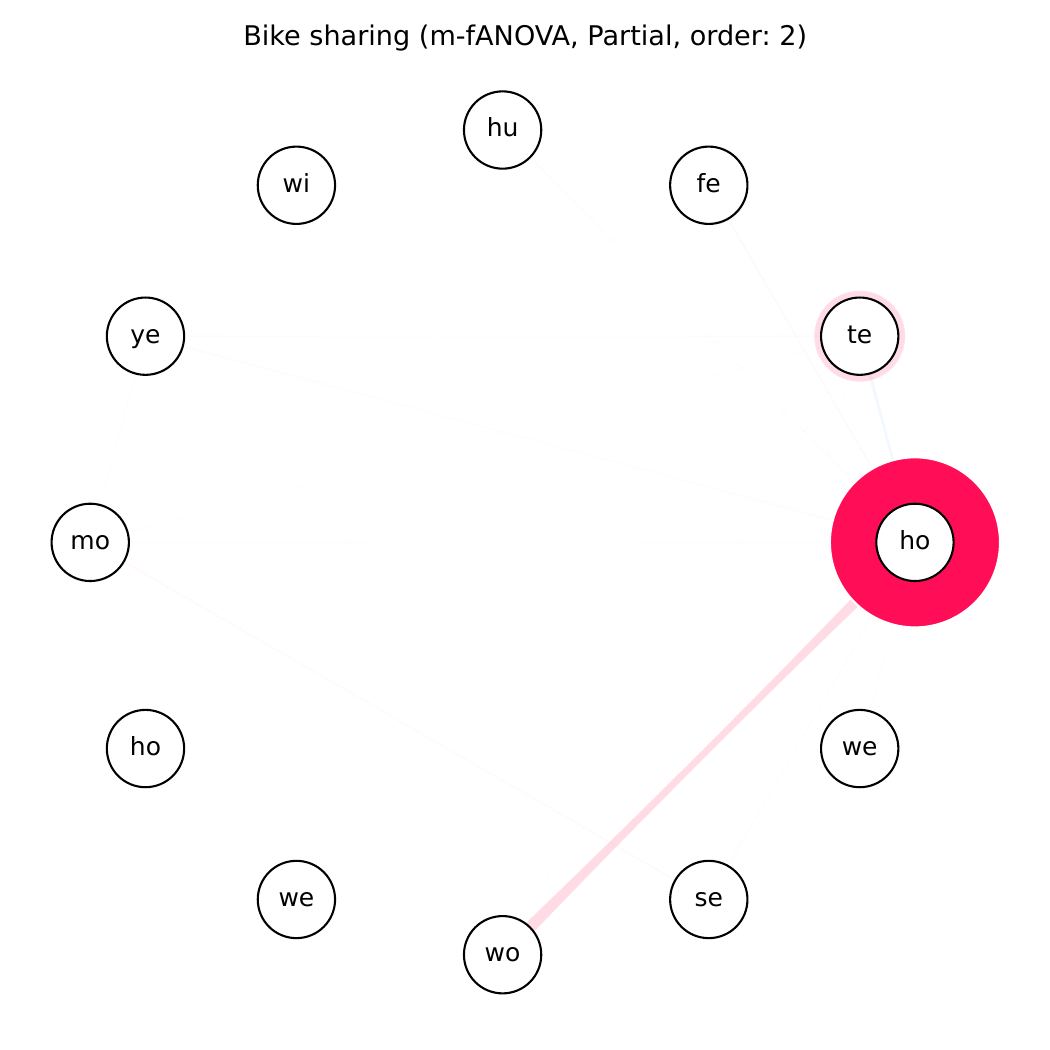}
    \\[1em]
    \textbf{(b) Partial}
    \end{minipage}
    \hfill
    \begin{minipage}[c]{0.32\textwidth}
    \centering
    \textbf{PFI}
    \\
    \includegraphics[width=\textwidth]{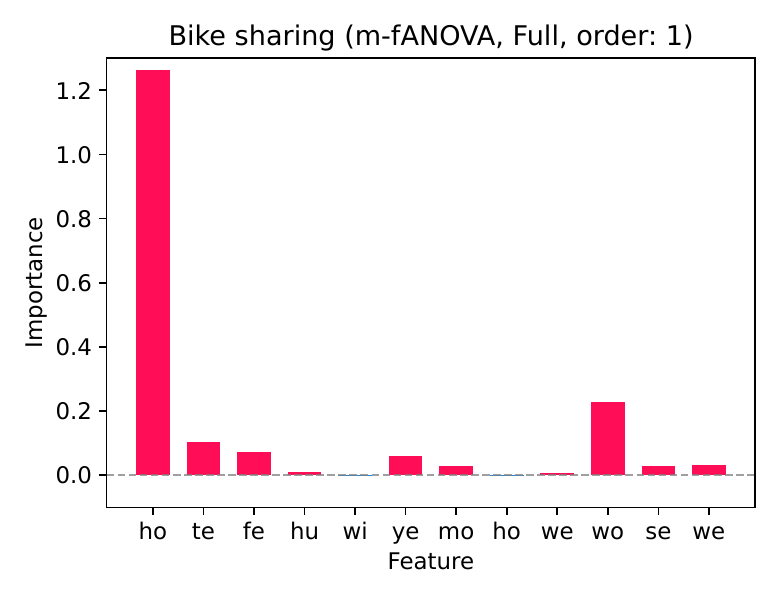}
    \\
    \includegraphics[width=\textwidth]{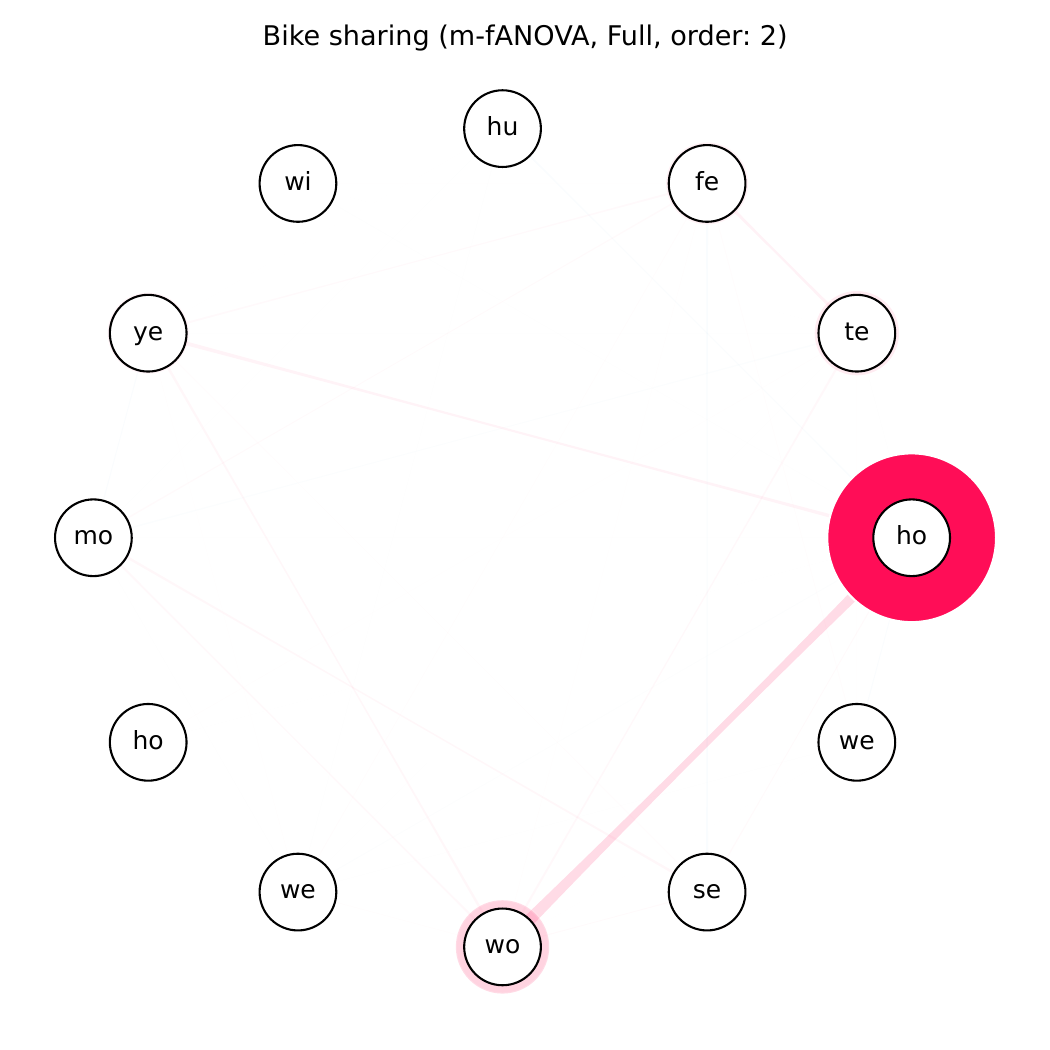}
    \\[1em]
    \textbf{(c) Full}
    \end{minipage}
    \vspace{1.5em}
    \caption{\textbf{Global Risk Game:} Pure (a), partial (b), and full (c) feature influences for a neural network fitted on \emph{Bike sharing}. The first row depicts the individual explanations where (b) corresponds to SAGE and (c) corresponds to PFI. The second row shows the pure, partial (2-\glspl{SV}, 2-SII) and full interactions up to the second-order. Importance is defined according to the global risk game by the change in performance measured by the mean squared error. This global risk game is based on 100 randomly drawn local explanation games from the test dataset.}
    \label{fig_appendix_bike_xgb_global}
\end{figure}

\begin{figure}
    \centering
    \hfill
    \begin{minipage}[c]{0.32\textwidth}
    \centering
    \phantom{\textbf{Pure, Order 1}}
    \\[1em]
    \includegraphics[width=0.95\textwidth]{language_game_The_acting_was_bad_but_the_movie_enjoyable_pure_interaction_1.pdf}
    \\
    \includegraphics[width=\textwidth]{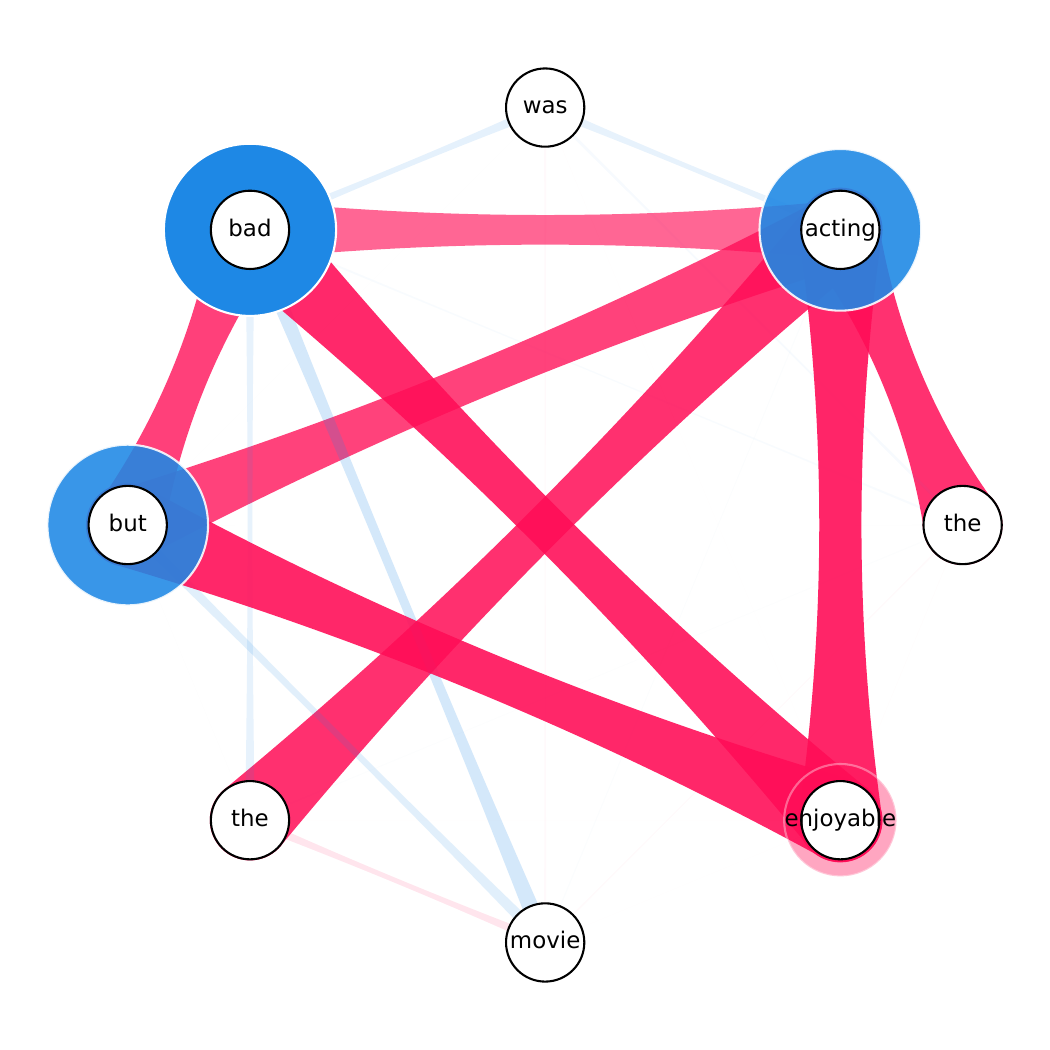}
    \\[1em]
    \textbf{(a) Pure}
    \end{minipage}
    \hfill
    \begin{minipage}[c]{0.32\textwidth}
    \centering
    \textbf{SHAP}
    \\[1em]
     \includegraphics[width=0.95\textwidth]{language_game_The_acting_was_bad_but_the_movie_enjoyable_partial_interaction_1.pdf}
    \\
    \includegraphics[width=\textwidth]{language_game_The_acting_was_bad_but_the_movie_enjoyable_partial_interaction_2.pdf}
    \\[1em]
    \textbf{(b) Partial}
    \end{minipage}
    \hfill
    \begin{minipage}[c]{0.32\textwidth}
    \centering
    \phantom{\textbf{Pure, Order 1}}
    \\[1em]
    \includegraphics[width=0.95\textwidth]{language_game_The_acting_was_bad_but_the_movie_enjoyable_full_interaction_1.pdf}
    \\
    \includegraphics[width=\textwidth]{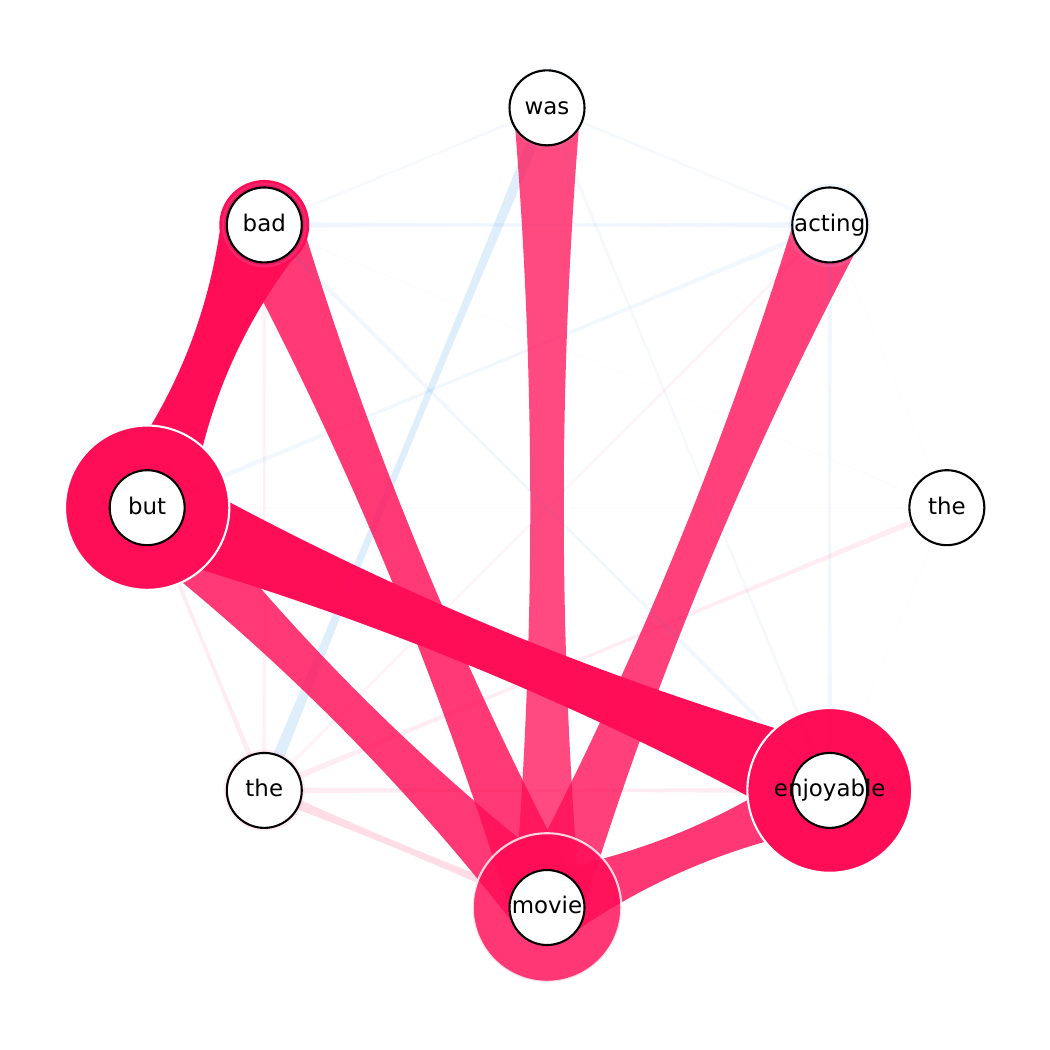}
    \\[1em]
    \textbf{(c) Full}
    \end{minipage}
    \vspace{1.5em}
    \caption{\textbf{Local Explanation Game:} Pure (a), partial (b), and full (c) feature influences for a sentiment analysis \textit{language model} fine-tuned on the \emph{IMDB} dataset \citep{tsai_faith-shap_2022}. The first row depicts the individual explanations where (b) corresponds to baseline SHAP using the \texttt{MASK} token, i.e., b-fANOVA. The second row shows pure, partial (2-\glspl{SV}, 2-SII) and full interactions up to the second-order. The sentence \textit{The acting was bad, but the movie enjoyable} receives a rather positive sentiment score of $0.8117$. Note that the model predicts a score of $0.5361$ for the empty set.}
    \label{fig_appendix_language_local}
\end{figure}
\end{document}